\documentclass{article}

\usepackage[final]{neurips_2021}

\usepackage[utf8]{inputenc} 
\usepackage[T1]{fontenc}    
\usepackage{hyperref}       
\hypersetup{citecolor=MidnightBlue}
\usepackage{url}            
\usepackage{booktabs}       
\usepackage{amsfonts}       
\usepackage{nicefrac}       
\usepackage{microtype}      
\usepackage{xcolor}         

\usepackage{amsmath}
\usepackage{amssymb}
\usepackage{amsthm}
\usepackage{enumitem}
\usepackage{graphicx}
\usepackage{subfig}
\usepackage{floatrow}
\usepackage{wrapfig}

\newcommand{\R}{\mathbb{R}}
\newcommand{\E}{\mathbb{E}}
\newcommand{\X}{\mathbb{X}}
\newcommand{\D}{\mathcal{D}}
\newcommand{\BMSEI}{\textsc{B-MS-EI}}
\newcommand{\indicate}[1]{\mathbf{1}_{\{#1\}}}
\allowdisplaybreaks

\DeclareMathOperator*{\argmax}{argmax}

\newtheorem{theorem}{Theorem}

\newtheorem{proposition}{Proposition}

\title{Multi-Step Budgeted Bayesian Optimization\\ with Unknown Evaluation Costs}

\author{%
  Raul Astudillo\\
  Cornell University\\
  \texttt{ra598@cornell.edu}\\
  \And
  Daniel R. Jiang\\
  Facebook\\
  \texttt{drjiang@fb.com}\\
  
  \And
  Maximilian Balandat\\
  Facebook\\
  \texttt{balandat@fb.com}\\
  
  \AND
  Eytan Bakshy\\
  Facebook\\
  \texttt{ebakshy@fb.com}

  \And
  Peter I. Frazier\\
  Cornell University\\
  \texttt{pf98@cornell.edu}\\

}

\begin{document}

\maketitle

\begin{abstract}
Bayesian optimization (BO) is a sample-efficient approach to optimizing costly-to-evaluate black-box functions. Most BO methods ignore how evaluation costs may vary over the optimization domain. However, these costs can be highly heterogeneous and are often unknown in advance. This occurs in many practical settings, such as hyperparameter tuning of machine learning algorithms or physics-based simulation optimization. Moreover, those few existing methods that acknowledge cost heterogeneity do not naturally accommodate a budget constraint on the total evaluation cost. This combination of unknown costs and a budget constraint introduces a new dimension to the exploration-exploitation trade-off, where learning about the cost incurs the cost itself. Existing methods do not reason about the various trade-offs of this problem in a principled way, leading often to poor performance. We formalize this claim by proving that the expected improvement and the expected improvement per unit of cost, arguably the two most widely used acquisition functions in practice, can be arbitrarily inferior with respect to the optimal non-myopic policy. To overcome the shortcomings of existing approaches,  we propose the \emph{budgeted multi-step expected improvement}, a non-myopic acquisition function that  generalizes classical expected improvement to the setting of heterogeneous and unknown evaluation costs. Finally, we show that our acquisition function outperforms existing methods in a variety of synthetic and real problems. 
\end{abstract}

\section{Introduction}

\looseness-1 Bayesian optimization (BO) \citep{shahriari2015survey,frazier2018tutorial} is a family of algorithms for optimizing black-box functions that performs well when the number of evaluations is limited \citep{Snoek2012ML,calandra2016bayesian,griffiths2020constrained}. However, most BO algorithms ignore the fact that the cost of evaluating the black-box objective function may vary substantially across the optimization domain and is often unknown. Problems with this feature arise commonly in practice. For instance, in the context of hyperparameter optimization of machine learning algorithms \citep{swersky2013multi, wu2020practical}, certain values of hyperparameters such as the learning rate may yield longer training times. Similarly, in materials design and robotics, simulation experiments can take longer for certain parameter configurations \citep{field1999practical}. 
 Figure~\ref{fig:lda_costs_hist} illustrates heterogeneity in evaluation costs from benchmark problems used in this paper, which can vary by an order of magnitude. Failing to account for these heterogeneous evaluation costs can lead to evaluating an expensive point when another less expensive one would provide equal benefit towards finding the optimum.

\begin{figure}
  \centering
 \includegraphics[width=0.32\textwidth]{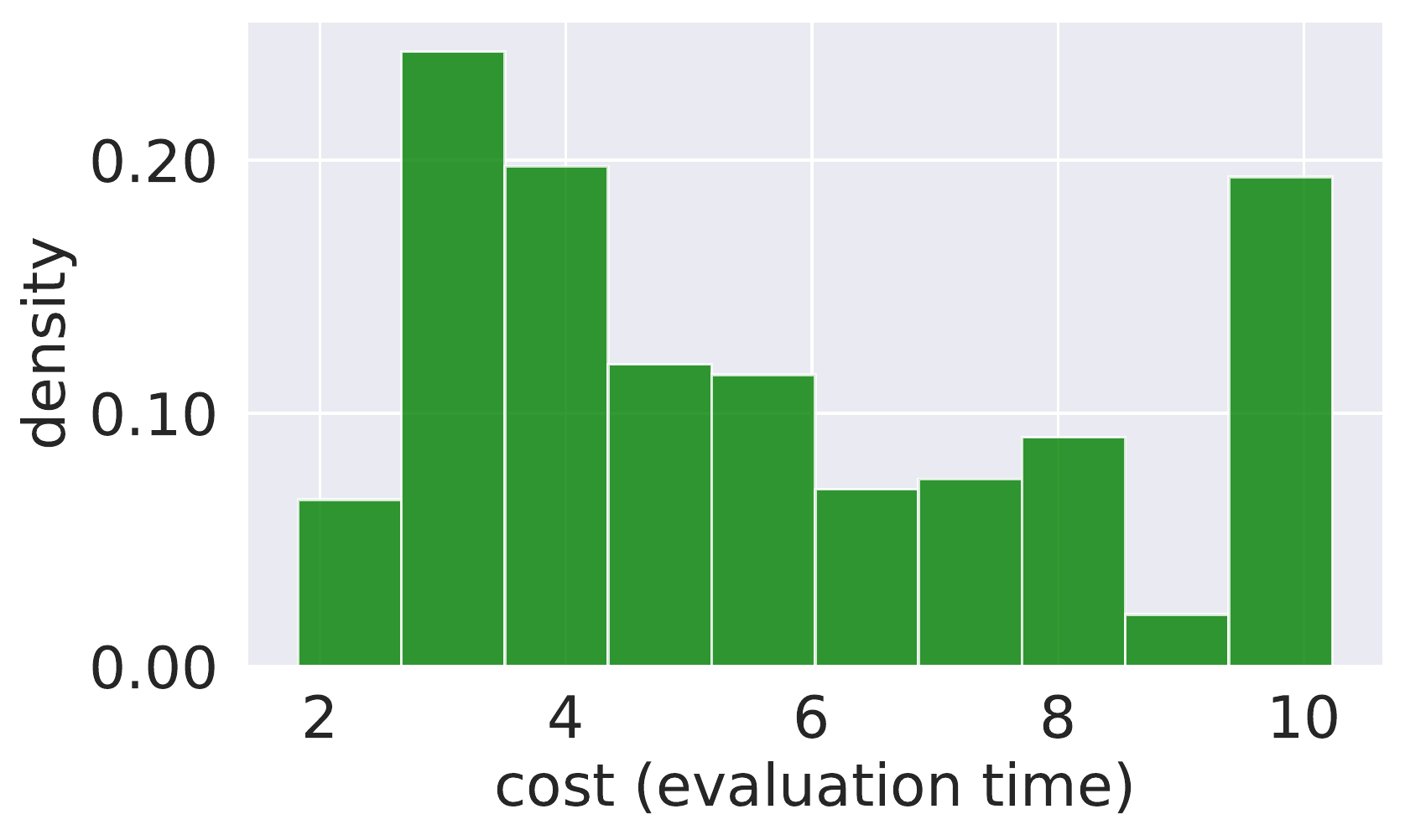}
 \includegraphics[width=0.32\textwidth]{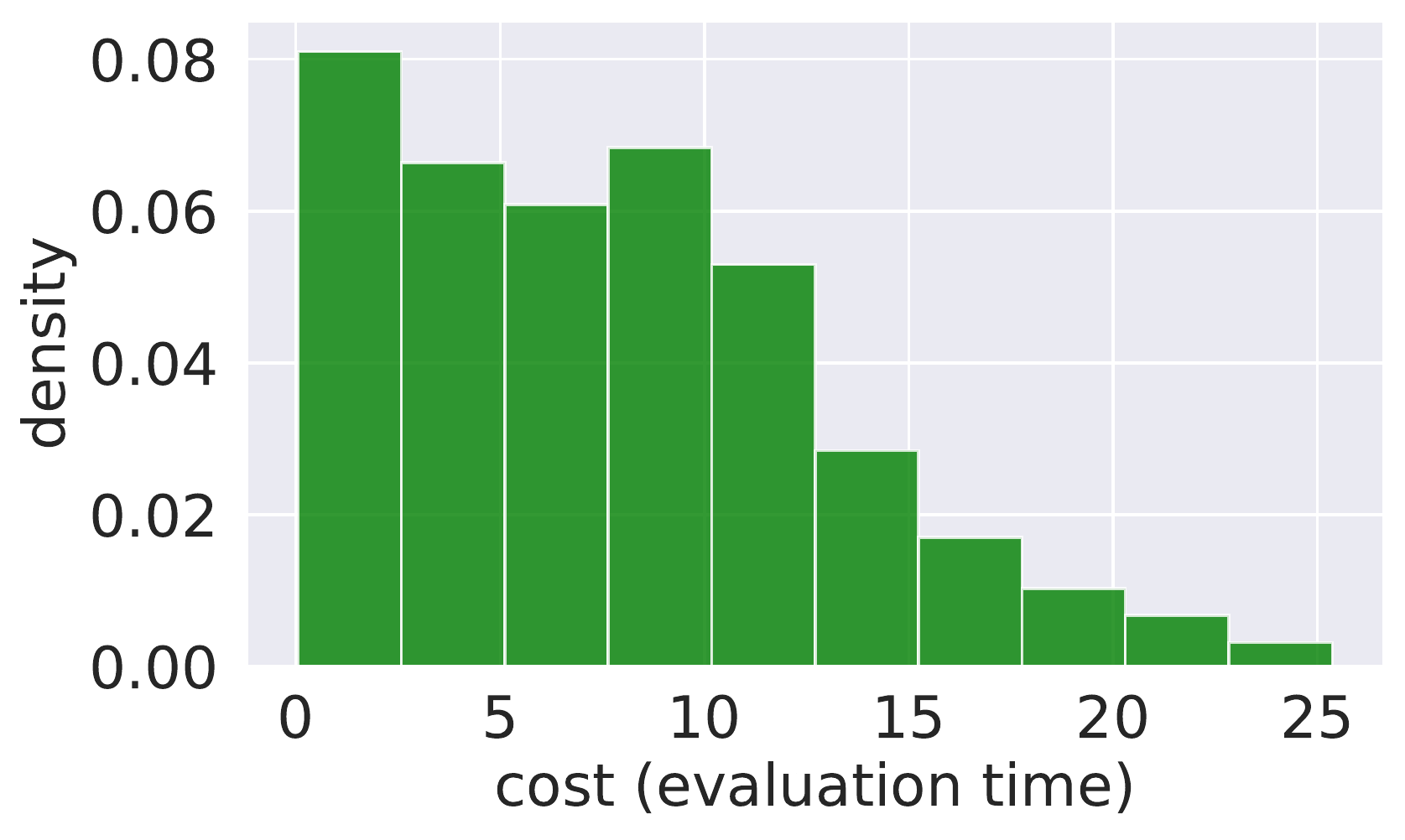}
 \includegraphics[width=0.32\textwidth]{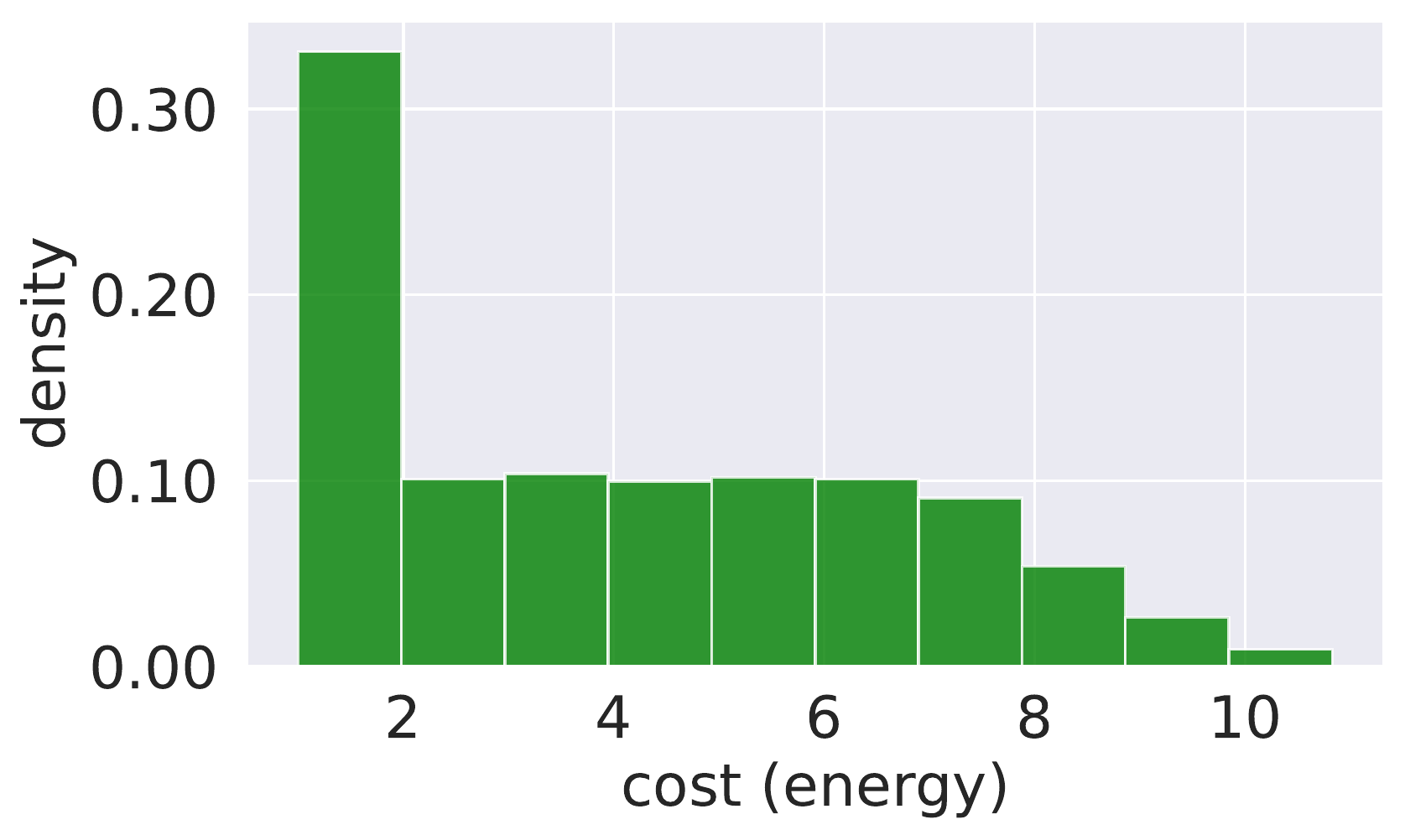}
   \vspace{-5pt}
  \caption{Evaluation times of the latent Dirichlet allocation, random forest, and energy-aware robot pushing benchmark problems (described later in Section~\ref{sec:description_benchmark}).} \label{fig:lda_costs_hist}
\end{figure}

 We consider \emph{budgeted} BO of a black-box objective function, whose evaluation costs are unknown and possibly heterogeneous across the domain. The goal is to find a point with the largest possible objective value by querying the objective function at a sequence of adaptively chosen points, where the total evaluation cost is subject to a budget constraint (this cost only affects evaluation, and not a point's quality upon implementation).
\looseness-1 
While some existing approaches do address heterogeneous evaluation costs, all do so heuristically, e.g. by maximizing a traditional cost-agnostic acquisition function divided by the cost of evaluation \citep{Snoek2012ML,poloczek2017multi,wu2020practical,lee2020cost}, or by rolling out a heuristic base policy \citep{lee2021}. Importantly, most of these approaches accommodate neither budget constraints nor uncertainty about the cost function as part of the exploration-exploitation trade-off. As we argue theoretically and demonstrate through experiments, this can lead to poor performance. A notable exception is the concurrent work of \cite{lee2021}, which introduces a budget-aware non-myopic acquisition function based on rollout of a heuristic base policy. This work appeared while the present paper was under review. 


\textbf{Main Contributions.} Motivated by the above shortcomings in existing work, we provide a principled approach to budgeted BO with unknown and potentially heterogeneous evaluation costs. Our main contributions are:
\begin{itemize}[itemsep=2pt,topsep=1pt,leftmargin=20pt]
\item We propose a Markov decision process (MDP) formulation of the budgeted BO problem with unknown and heterogeneous evaluation costs. Our formulation allows for a random time horizon (i.e., the last time before the budget is depleted),  
going beyond the fixed-horizon MDPs formulated in existing work on non-myopic BO.
\item 
\emph{Budgeted multi-step expected improvement} (\BMSEI{}), a novel look-ahead acquisition function that generalizes classical expected improvement to the budgeted cost-heterogeneous setting. \BMSEI{} can be seen as a principled approximation of the optimal policy of our MDP. 
\item We prove that expected improvement (EI) and its cost-normalized variant, two popular existing approaches, can be arbitrarily inferior with respect to the optimal non-myopic policy.
\item An empirical evaluation on a number of synthetic and real-world experiments demonstrates that \BMSEI{} performs favorably with respect to other acquisition functions that are widely-used in settings with heterogeneous costs.
\end{itemize}

The remainder of this work is organized as follows: In Section \ref{sec:background}, we review related work. Our problem setup is formalized in Section \ref{sec:problem}. In Section \ref{sec:acqf}, we introduce 
\BMSEI{} and discuss its efficient maximization via one-shot multi-step Monte Carlo trees. Numerical experiments are presented in Section \ref{sec:experiments}. Finally, we discuss directions of future work and conclude in Section \ref{sec:conclusion}.

\section{Related Work}
\label{sec:background}
Our work falls within the BO framework \citep{frazier2018tutorial}, whose origins date back to the seminal work of \cite{kushner1964new}. BO has been successful in a wide range of applications, such as hyperparameter tuning of machine learning algorithms \citep{Snoek2012ML, wu2020practical}, materials design \citep{zhang2020bayesian}, drug discovery \citep{griffiths2020constrained}, and robot locomotion \citep{calandra2016bayesian, wang2017max}. 

Within the BO literature, the works most closely related to ours are those that acknowledge the existence of costs for evaluating the objective function that are heterogeneous across the search space and aim to devise algorithms that are cost-aware. Much of this work has occurred in the multi-fidelity setting \citep{swersky2013multi,kandasamy2016gaussian,kandasamy2017multi, poloczek2017multi, song2019general, wu2020practical}, i.e., where cheaper approximations of the objective function are available. The only exceptions known to us are \cite{Snoek2012ML}, \cite{lee2020cost}, and \cite{lee2021}, which consider the single-fidelity setting.


\cite{Snoek2012ML} proposes the \textit{expected improvement per unit of cost (EI-PUC)}, i.e, $\mathrm{EI}(x) / c(x)$\footnote{This expression is for the case when $c(x)$ is known. When $c(x)$ is unknown and learned, then either the expectation is taken over the distribution of improvement to cost ratios (as we do in our experiments) or the denominator is replaced by the mean cost.} where $c(x)$ is the cost of evaluating the objective at $x$ and $\mathrm{EI}(x)$ is the expected improvement. \cite{lee2020cost} proposes a simple variation called the \textit{expected improvement per unit of cost with cost cooling (EI-PUC-CC)}. EI-PUC-CC is defined by $\textnormal{EI}(x)/c(x)^\nu$, where $\nu$ is the ratio between the current remaining and initial budgets. The intuition behind the cost exponent is that evaluating points with high cost should be discouraged early in the BO loop (when $\nu \approx 1$) and accommodated as the budget is consumed and $\nu$ decreases to $0$. However, neither EI-PUC nor EI-PUC-CC 
consider uncertainty in the cost or measure the budget-dependent value of information in a principled way. 

Work concurrent to ours, \cite{lee2021}, also tackles budgeted BO with heterogeneous costs using a non-myopic strategy. However, our work differs in both the model and solution method. \cite{lee2021} uses a finite-horizon constrained MDP, while our model is an MDP with a random horizon.
We argue that the
random horizon formulation is more natural: the formulation of \cite{lee2021} requires the addition of a zero reward, zero cost state to accommodate trajectories with a small number of evaluations. Within this formulation, \cite{lee2021} proposes a rollout acquisition function with a particular heuristic base policy: $h-1$ steps of EI-PUC followed by a last step of EI, where $h$ is the number of look-ahead steps performed. The acquisition function is essentially a single step of policy improvement over the ``EI-PUC followed by EI'' heuristic \citep{sutton2018reinforcement}. In contrast, our acquisition function aims to directly approximate the optimal policy.



\looseness-1 The approach of dividing a cost-agnostic acquisition function by some cost term is widely used in practice for addressing heterogeneity in costs, and  
is closely related to the use of ``value divided by cost'' in knapsack problems \citep{badanidiyuru2013bandits}.
In the knapsack problem, one selects items to include into a knapsack to maximize the sum of the selected items' values subject to a budget constraint on the sum of the items' costs.
In this setting, myopically adding items to the knapsack that maximize value divided by cost has strong theoretical guarantees: 
this algorithm provides at least $1/2$ the optimal value \citep{williamson2011design}.
However, this theoretical guarantee relies on the additive nature of value in the knapsack problem. In contrast, in BO, the value obtained from multiple evaluations is the maximum of the values of the evaluations, not their sum. Indeed, we show that in this setting the ``value divided by cost'' approach can perform arbitrarily worse than the optimal policy.

\looseness-1 Heterogeneous evaluation costs have also been considered in the multi-armed bandits literature \citep{badanidiyuru2013bandits,xia2015thompson,xia2016budgeted}. These works develop algorithms based on optimistic policies that maximize some form of reward-to-cost ratio. As mentioned above, this type of policy is sensible when the measure of performance is the cumulative regret but is not appropriate in an optimization or ``best-arm identification'' setting.

Our work is also closely related to non-myopic BO \citep{gonzalez2016glasses, lam2016bayesian, yue2020non, jiang2020binoculars, lee2020efficient, jiang2020efficient,lee2021}, a class of acquisition functions that account for future evaluations when quantifying a point's acquisition value. To the best of our knowledge, the work of \cite{lee2021} discussed above is the only one among these that is able to handle heterogeneous evaluation costs.


\section{Budgeted Bayesian Optimization with Unknown Evaluation Costs}
\label{sec:problem}
We now formally state the problem of budgeted BO with unknown evaluation costs. Given a compact optimization domain $\X\subset\R^d$, our goal is to find a point $x\in\X$ with the largest possible objective value by querying the black-box objective function, $f:\X \rightarrow \R$, at a sequence of points $\{x_i\}_{i=1}^n$, subject to the constraint $\sum_{i=1}^n c(x_n)\leq B$, where $c(x)$ is the cost of evaluating $f$ at $x$ and $B$ is the evaluation budget. The cost observation $c(x)$ is revealed immediately after the evaluation of $f(x)$ is performed. However, the actual cost function $c$ is unknown. 
As is typical in BO, we endow $f$ and $c$ with a joint prior distribution, $p$. An \emph{observation} in our setting is a triple $(x_i, y_i, z_i) \in \X\times\R\times \R_{>0}$, where $y_i$ is an observation of the objective $f$ at $x_i$, and $z_i$ is an observation of the cost $c$ for evaluating $f$ at $x_i$.

\subsection{MDP Formulation}
\looseness-1 The state of our MDP at step $n$ is $\mathcal D_n$, defined as the set of observations so far. These sets are defined recursively by $\D_n = \D_{n-1}\cup {(x_n, y_n, z_n)}$ for $n\ge 1$, where $\D_0$ is a set of initial observations. The joint posterior distribution over $f$ and $c$ given $\D_n$ is denoted by $p(\cdot\mid \D_n)$. The \emph{utility} generated by a particular state $\D_n$ is defined as the maximum observed objective value $u(\D_n) = \textstyle \max_{(x,y,z) \in \D_n} y$. Note that this utility function encodes the fact that \emph{after evaluation}, a point with maximum objective value is desired regardless of its cost. We also let $s(\D_n) = \textstyle \sum_{(x,y,z) \in \D_n} z$ denote the \emph{total cost} of observed points in $\D_n$.

The sets of observations $\D_1, \D_2, \ldots$ are random due to the yet unobserved values of the objective and cost functions. A \emph{policy} $\pi = \{\pi_k\}_{k=1}^\infty$ is a sequence of functions, each mapping sets of observations to points in $\X$, so that $x_k = \pi_k(\D_{k-1})$. Given a set of observations $\D$ such that $s(\D) \le B$ (i.e., there is nonnegative remaining budget), the \emph{value function} of a policy $\pi$ is defined as $V^\pi(\D) = \E^{\pi}\bigl[u(\D_{N_B}) - u(\D_0) \, | \, \D_0 = \D\bigr]$, where the random stopping time $N_B = \sup\{k : s(\D_k) \leq B\}$ is the largest time step $k$ for which the budget constraint is still satisfied. The notation $\E^\pi[ \, \cdot \, ]$ indicates an expectation taken over sequences of observation sets $\D_1, \D_2, \ldots, \D_{N_B}$ selected by a policy $\pi$. For a set $\D$ where $s(\D) > B$ (i.e., budget is exhausted), we define $V^\pi(\D) = 0$. Our goal is to find a policy $\pi$ that maximizes the increase in expected utility:
\begin{equation}
\label{eq:opt_policy}
   V^*(\D) = \sup_{\pi \in \Pi} V^\pi(\D),
\end{equation}
where $\Pi$ is the set of all possible policies. The above problem is well-defined provided that $N_B < \infty$ for all policies $\pi$. This is the case, for example, when $\ln c$ follows a Gaussian process (GP) prior, a modeling choice that we make in our numerical experiments. Since the horizon $N_B$ is random, the formulation (\ref{eq:opt_policy}) can be viewed as a stochastic shortest path problem, rather than a standard finite or discounted infinite horizon dynamic program \citep{bertsekas1991analysis}. Note that at time $k$, the set $\D_k$ contains all past observed costs and fully captures the remaining budget. 
This formulation is capable of the following:
\begin{enumerate}[itemsep=2pt,topsep=1pt,leftmargin=20pt]
    \item Through multi-step planning, it can navigate the trade-off of how to sequence high-cost and low-cost evaluations in order to make the best use of the given budget.
    \item It can reason about uncertainty when planning optimal cost-learning. For example, an exploratory evaluation may be worthwhile in a region with moderate estimated cost and high model uncertainty: the evaluation may reveal the cost is lower than estimated, allowing us to explore the region more fully for low cost.
\end{enumerate}

\subsection{Contrast with Value-to-Cost Ratio Methods}

\looseness-1 It is not surprising that cost-agnostic methods, such as expected improvement (EI), can perform poorly when the evaluation cost is heterogeneous. 
In particular, ignoring cost can lead to evaluating excessively high-cost points, depleting budget and limiting future evaluations. In an attempt to avoid this, past work has focused on using a value-to-cost ratio \citep{Snoek2012ML, swersky2013multi,kandasamy2016gaussian,kandasamy2017multi, poloczek2017multi, song2019general, wu2020practical, lee2020cost}. 

We show here, however, that value-to-cost acquisition functions exhibit a complementary kind of undesirable behavior: they may repeatedly measure excessively low-cost points that are also low-value, leading to poor overall performance. In fact, the performance can be arbitrarily bad compared with an optimal policy. Theorem~\ref{thm:counterexample} demonstrates this formally for the most widely-used of these policies: measuring at the point that maximizes the expected improvement per unit of cost (EI-PUC), and also for cost-agnostic expected improvement (EI).

\begin{theorem}
\label{thm:counterexample}
The approximation ratios provided by the $\mathrm{EI}$ and $\mathrm{EI\text{-}PUC}$ policies are unbounded. That is, for any arbitrarily large $\rho > 0$ and each policy $\pi\in\{\mathrm{EI}, \mathrm{EI\text{-}PUC}\}$, there exists a Bayesian optimization problem instance (a  prior probability distribution over objective and cost functions, a budget, and a set of initial observations $\D_0$) where 
$V^*(\D_0) > \rho \, V^{\pi}(\D_0)$.
\end{theorem}

\vspace{-10pt}
\begin{proof}[Sketch of Proof.] \looseness-1 To show that EI-PUC has an unbounded approximation ratio, the detailed proof of Theorem~\ref{thm:counterexample} (provided in Section~\ref{supp:thm1} of the supplementary material) constructs a problem instance with a prior that is independent across a discrete domain with a known cost function. There are two kinds of points: high-cost points with large prior variance; and low-cost points with low variance. To support analysis, all points have the same mean.

The variance of the low-cost points is low enough that spending the entire budget on evaluating them earns only a fraction of the value, in expectation, earned by evaluating a single high cost point. The acquisition function EI-PUC, however, does not understand this. Its greedy nature leads it to overvalue these points. 
Indeed, EI (the numerator of EI-PUC's acquisition function) greedily values improvement relative to the status quo, even though the status quo will likely be surpassed by other later evaluations. This overvalues small improvements like those that result from low-variance points.

EI-PUC spends its entire budget on these low-cost low-variance points, earning almost no value. In contrast, the optimal policy spends its entire budget on a single evaluation of the high-cost point, obtaining substantially more value in expectation. 


For the case of EI, we construct a similar example, with the change that the low-cost points now have variance \emph{only slightly smaller} than that of the high-cost points, while still being significantly cheaper. Here, EI performs a single evaluation of the high-cost point and runs out of budget. 
On the other hand, repeatedly measuring low-cost points is far superior to EI in expectation.
\end{proof}
\vspace{-5pt}

\begin{figure}
\centering
\subfloat[EI-PUC]{%
\begin{tabular}[b]{c}%
\includegraphics[width=0.33\textwidth]{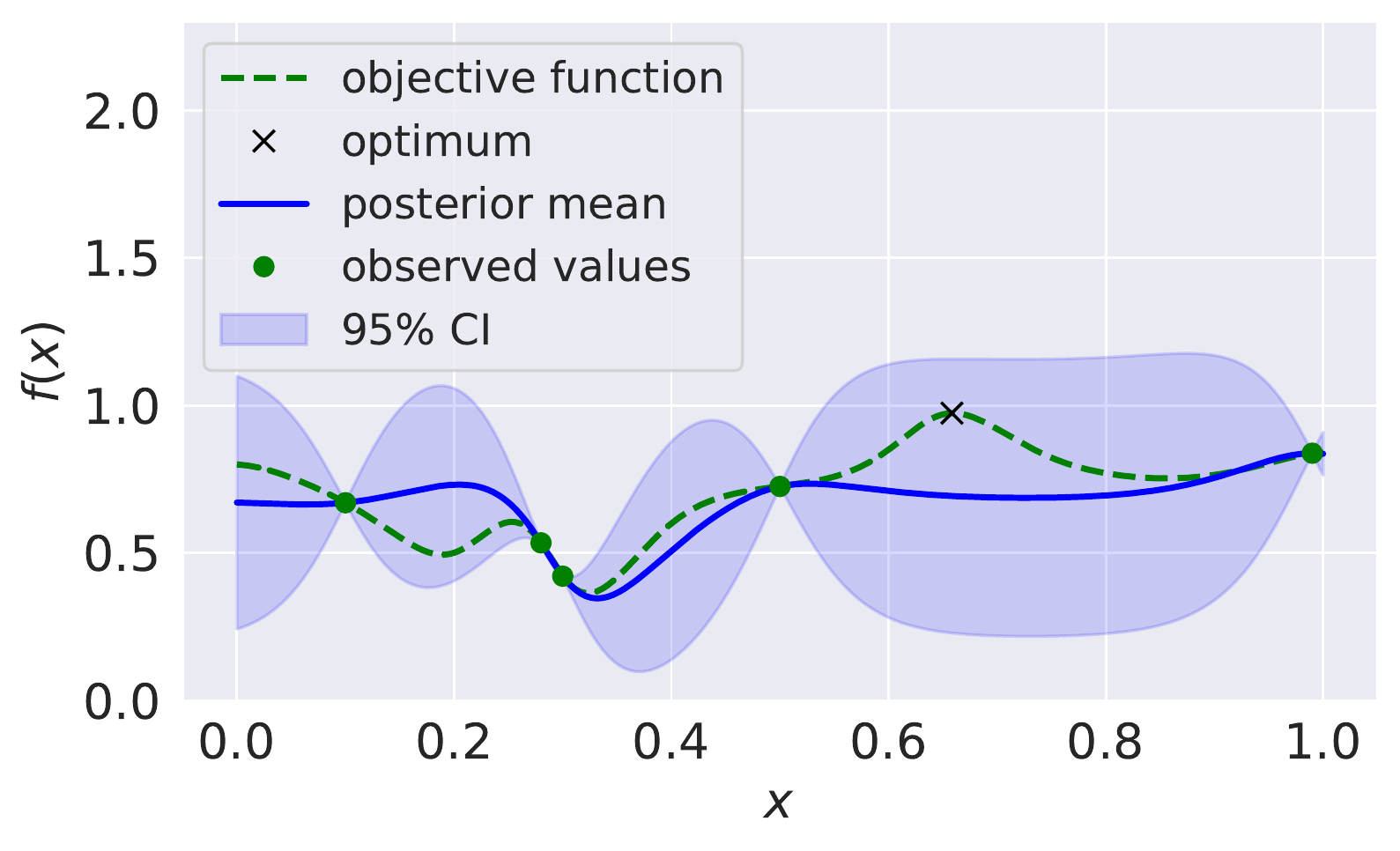}
  \includegraphics[width=0.33\textwidth]{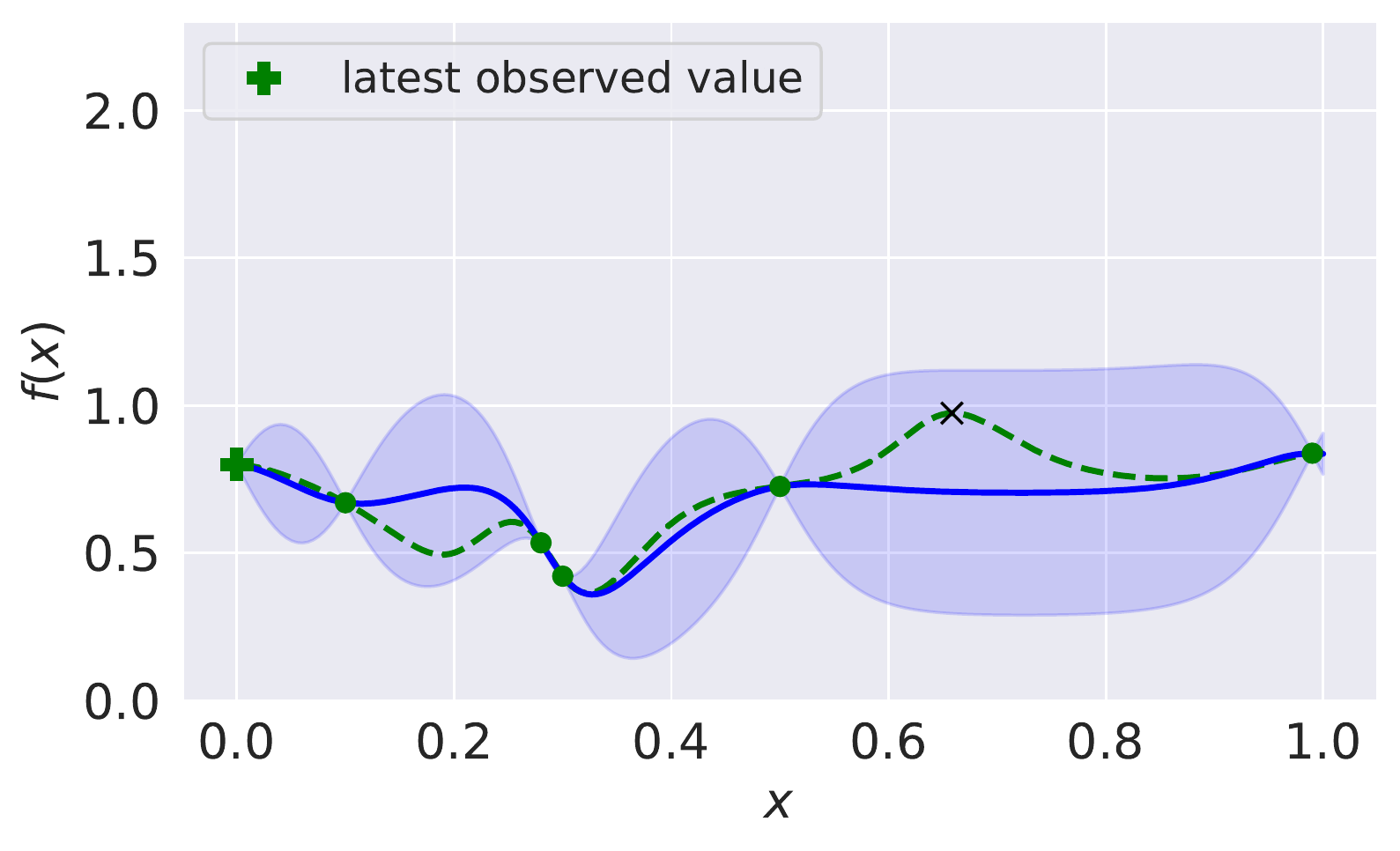}
  \includegraphics[width=0.33\textwidth]{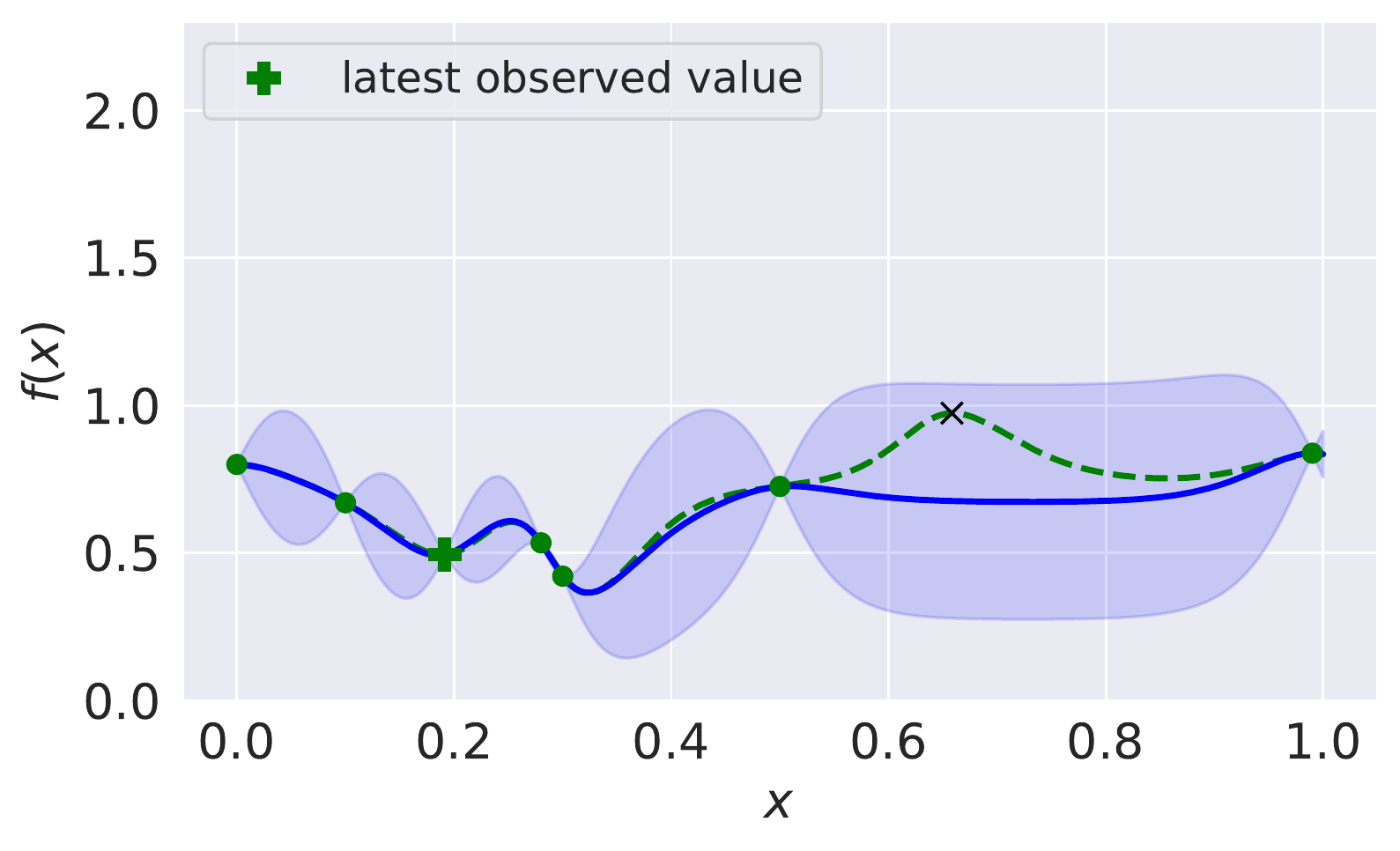}\\
    \includegraphics[width=0.33\textwidth]{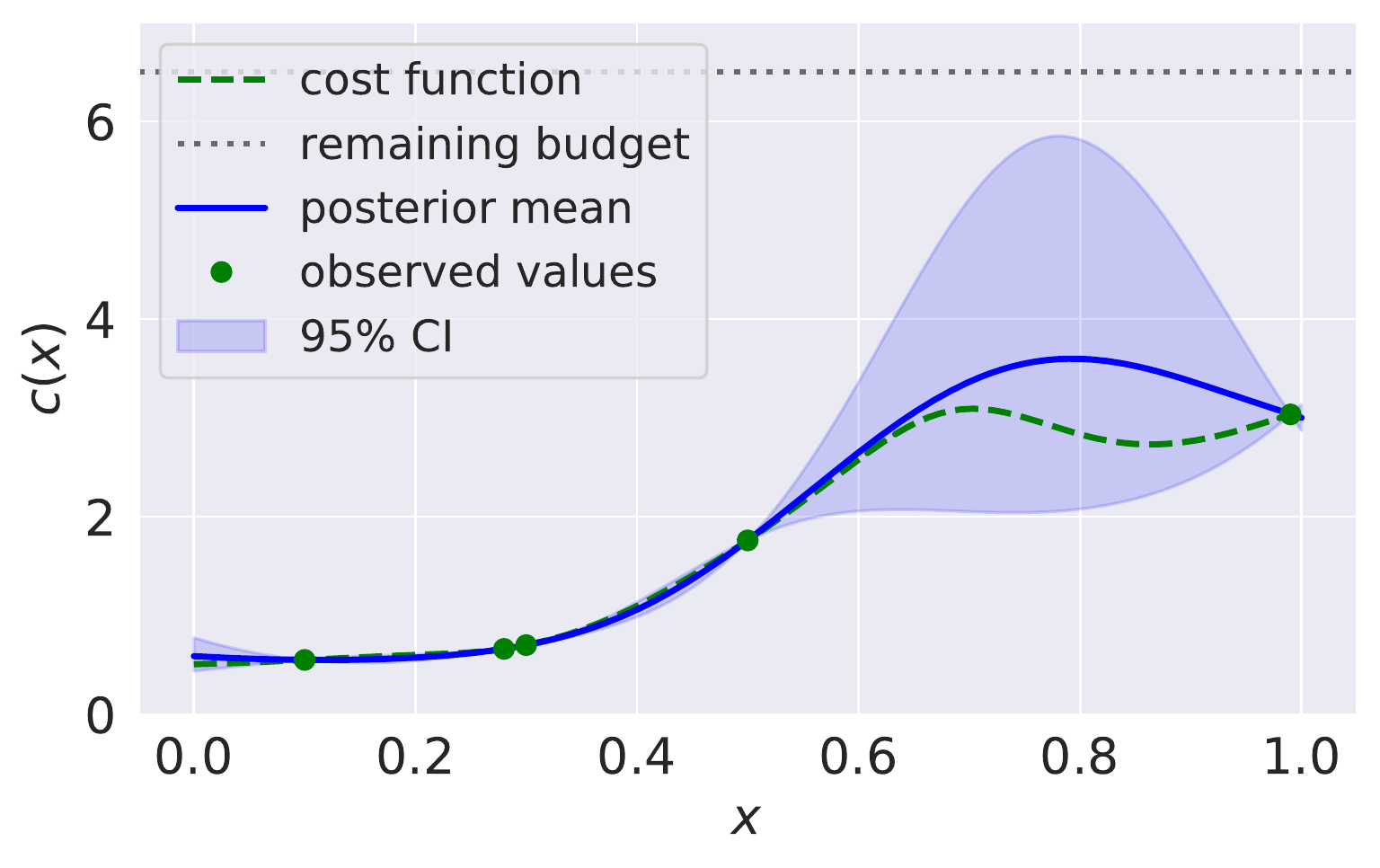}
  \includegraphics[width=0.33\textwidth]{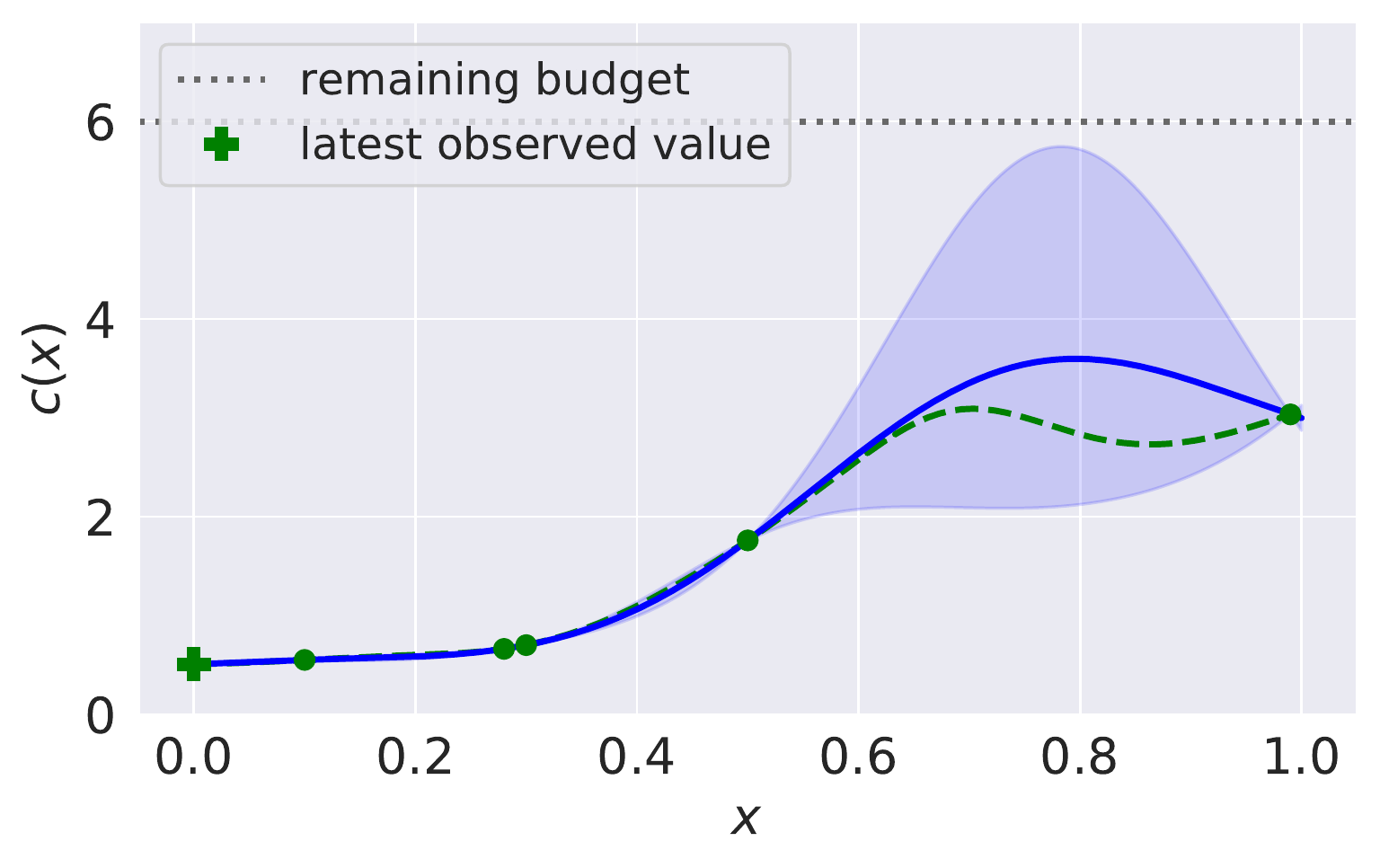}
  \includegraphics[width=0.33\textwidth]{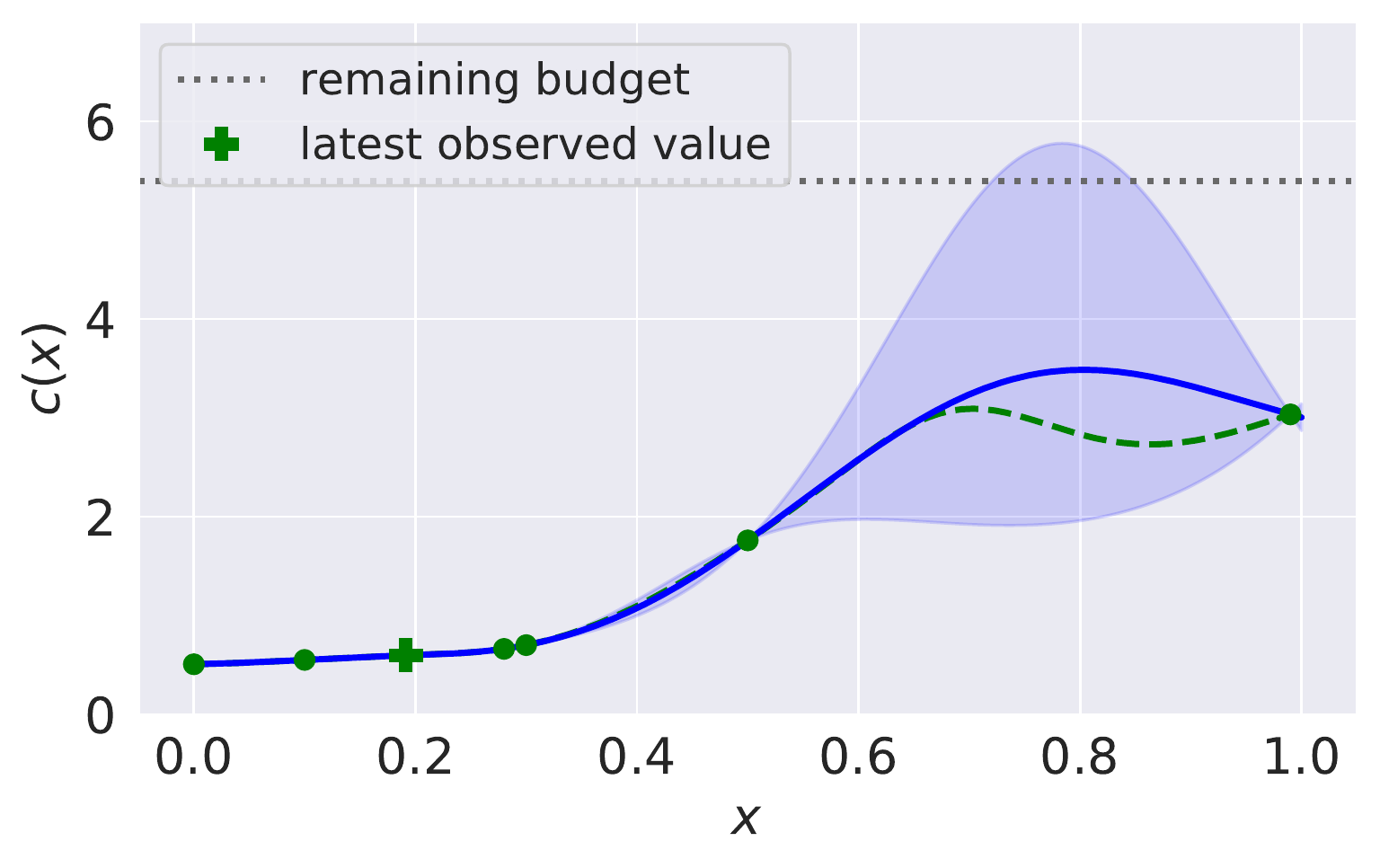}\\
   \includegraphics[width=0.33\textwidth]{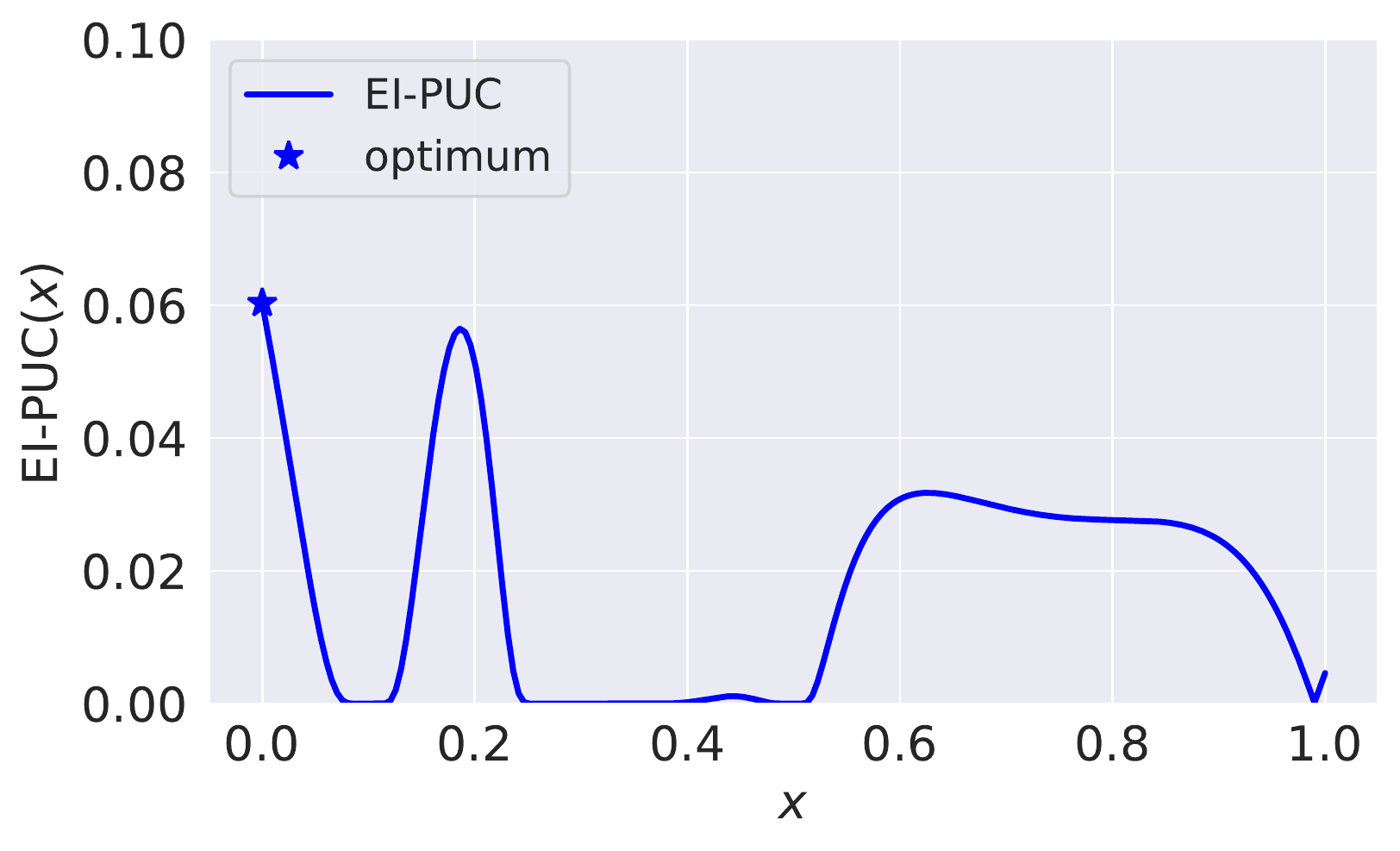}
  \includegraphics[width=0.33\textwidth]{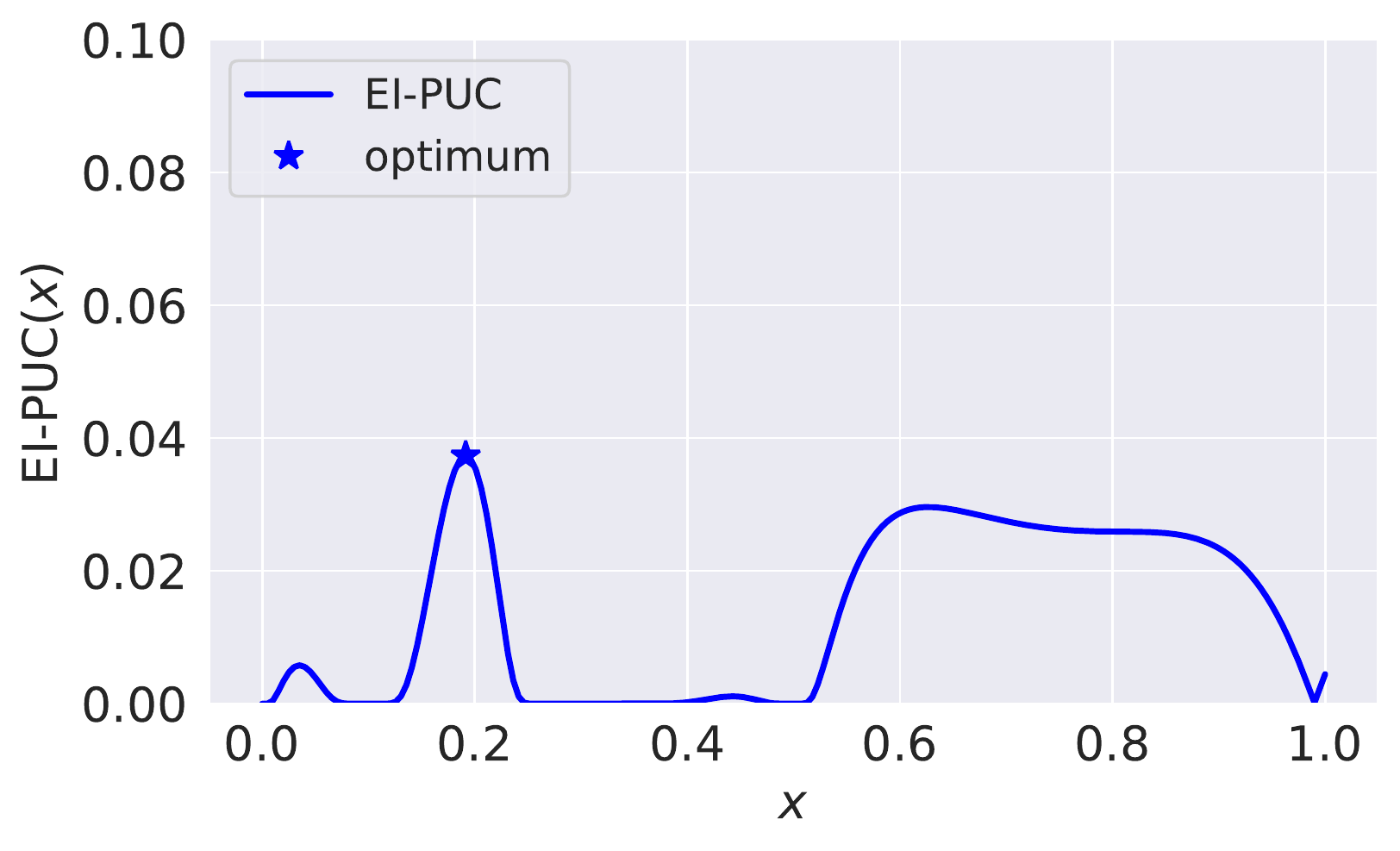}
  \includegraphics[width=0.33\textwidth]{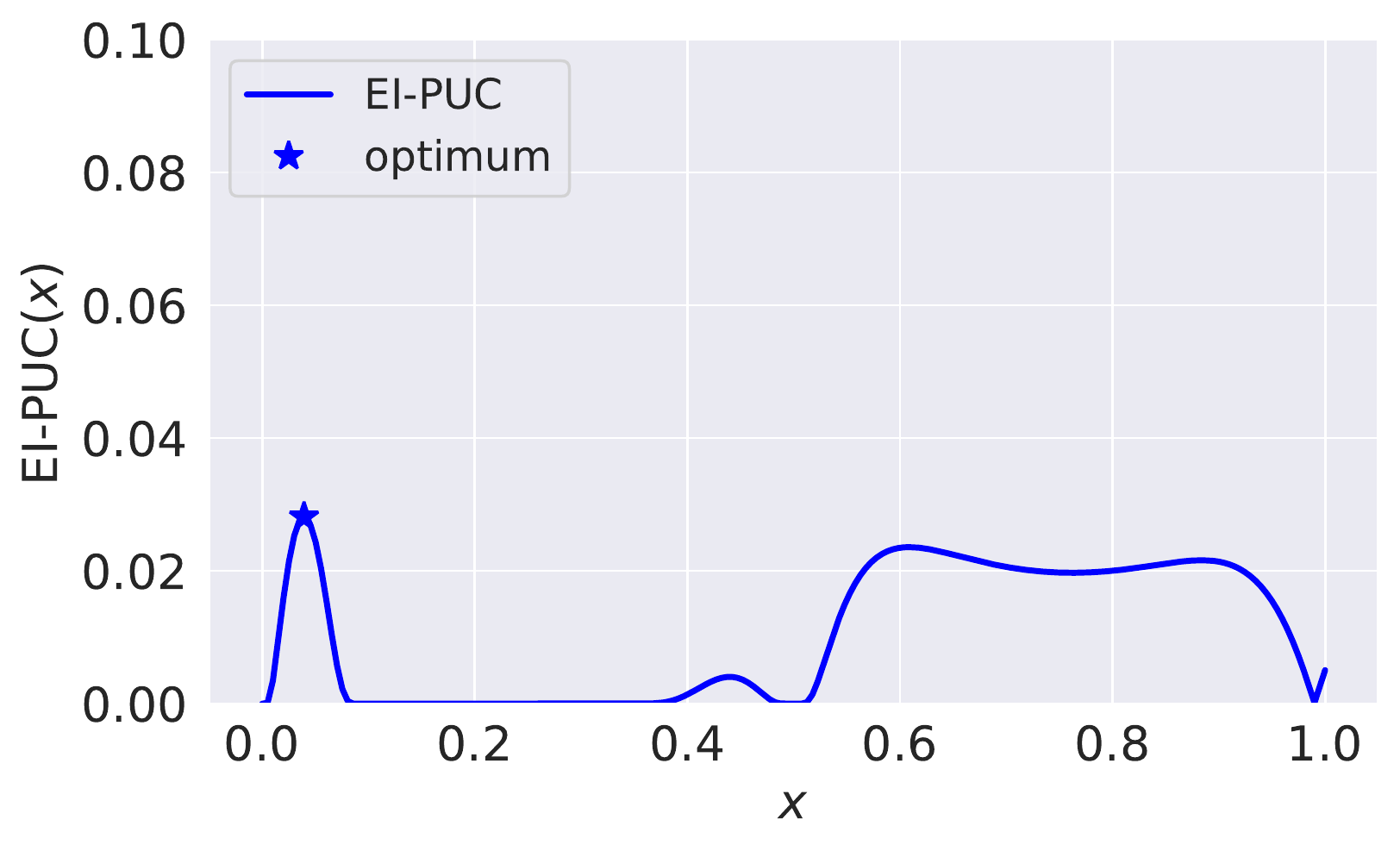}
 \end{tabular}
}\\
\subfloat[B-MS-EI]{%
\begin{tabular}[b]{c}
 \includegraphics[width=0.33\textwidth]{figures/animation/obj0.pdf}
  \includegraphics[width=0.33\textwidth]{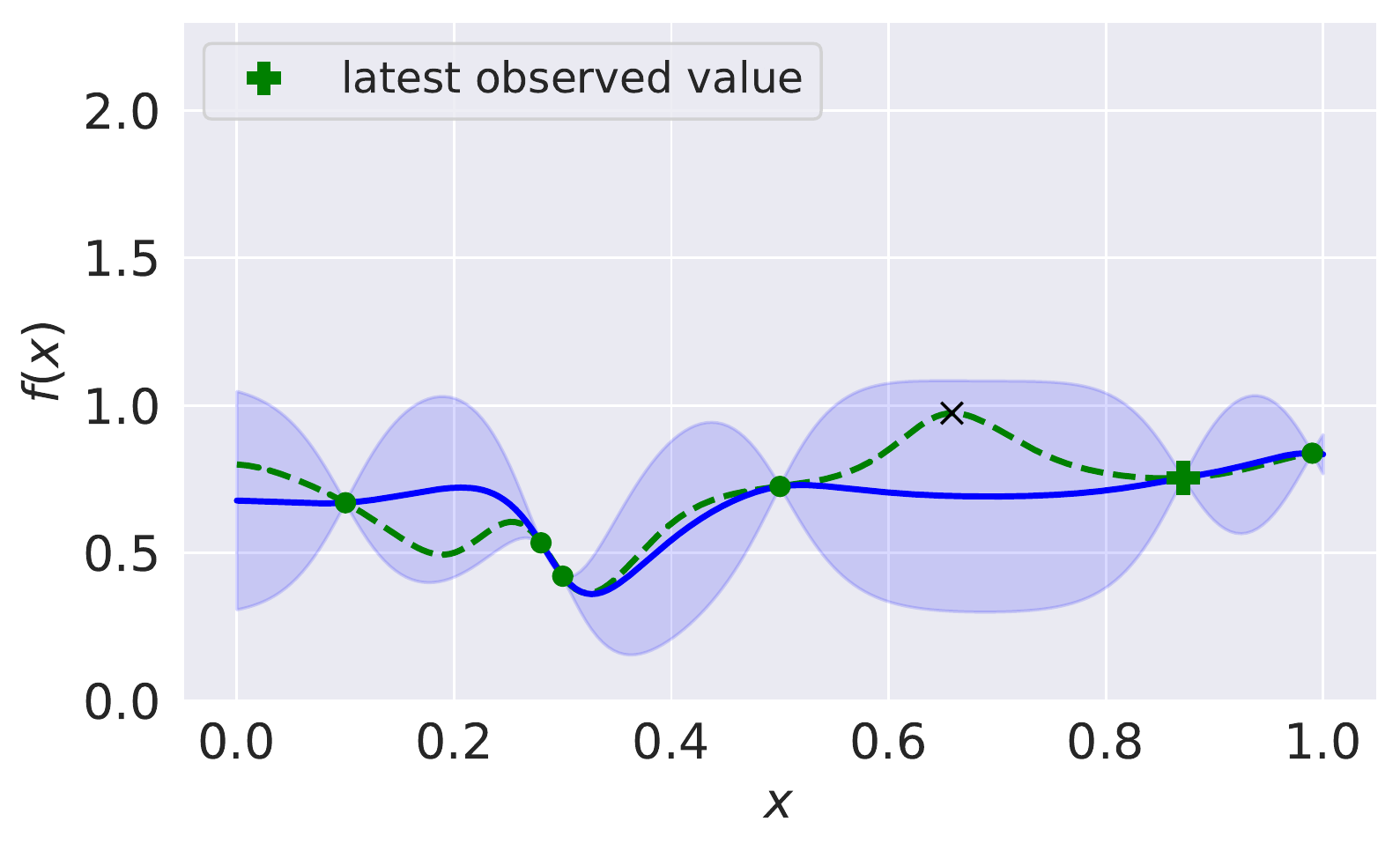}
  \includegraphics[width=0.33\textwidth]{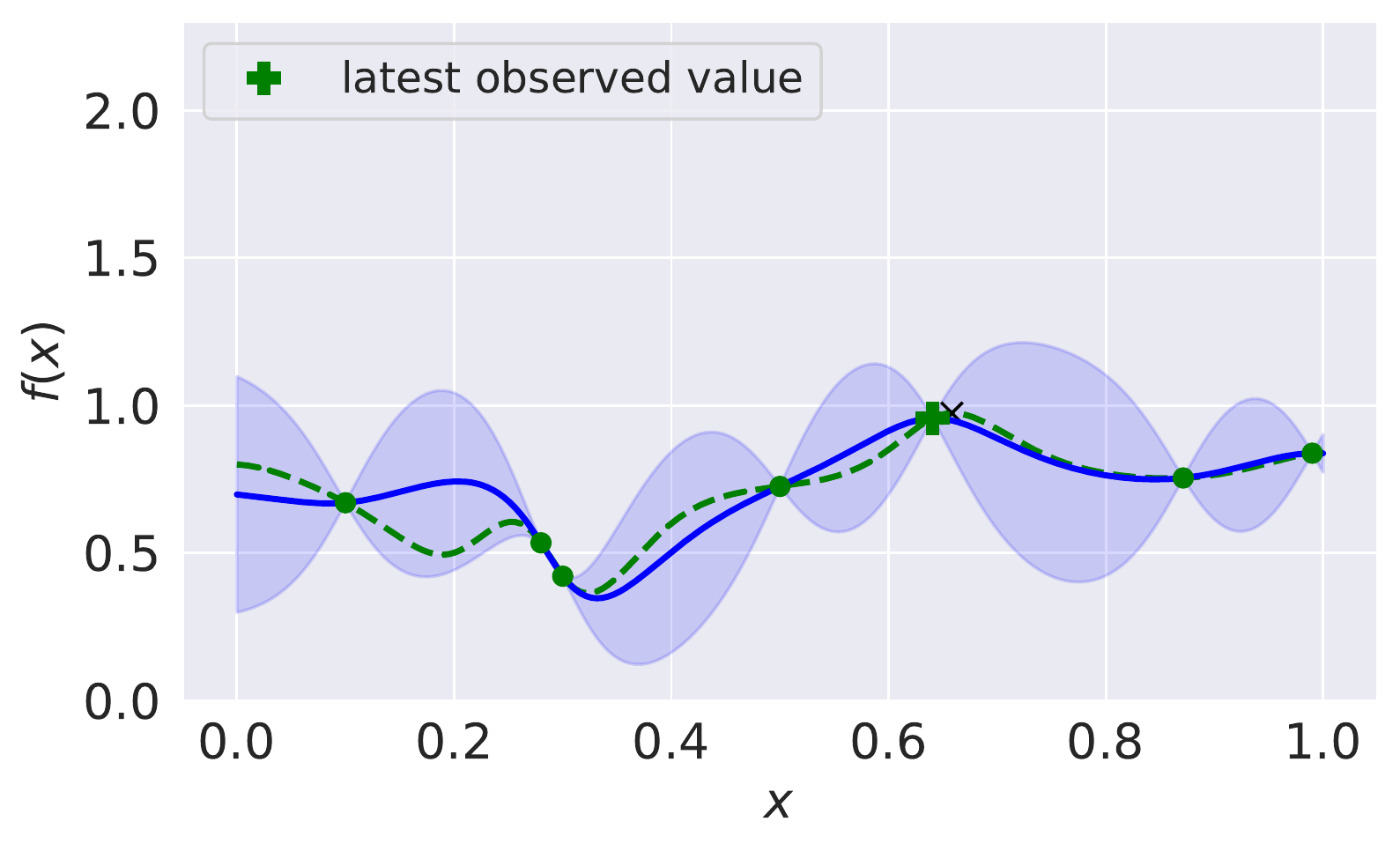}\\
    \includegraphics[width=0.33\textwidth]{figures/animation/cost0.pdf}
  \includegraphics[width=0.33\textwidth]{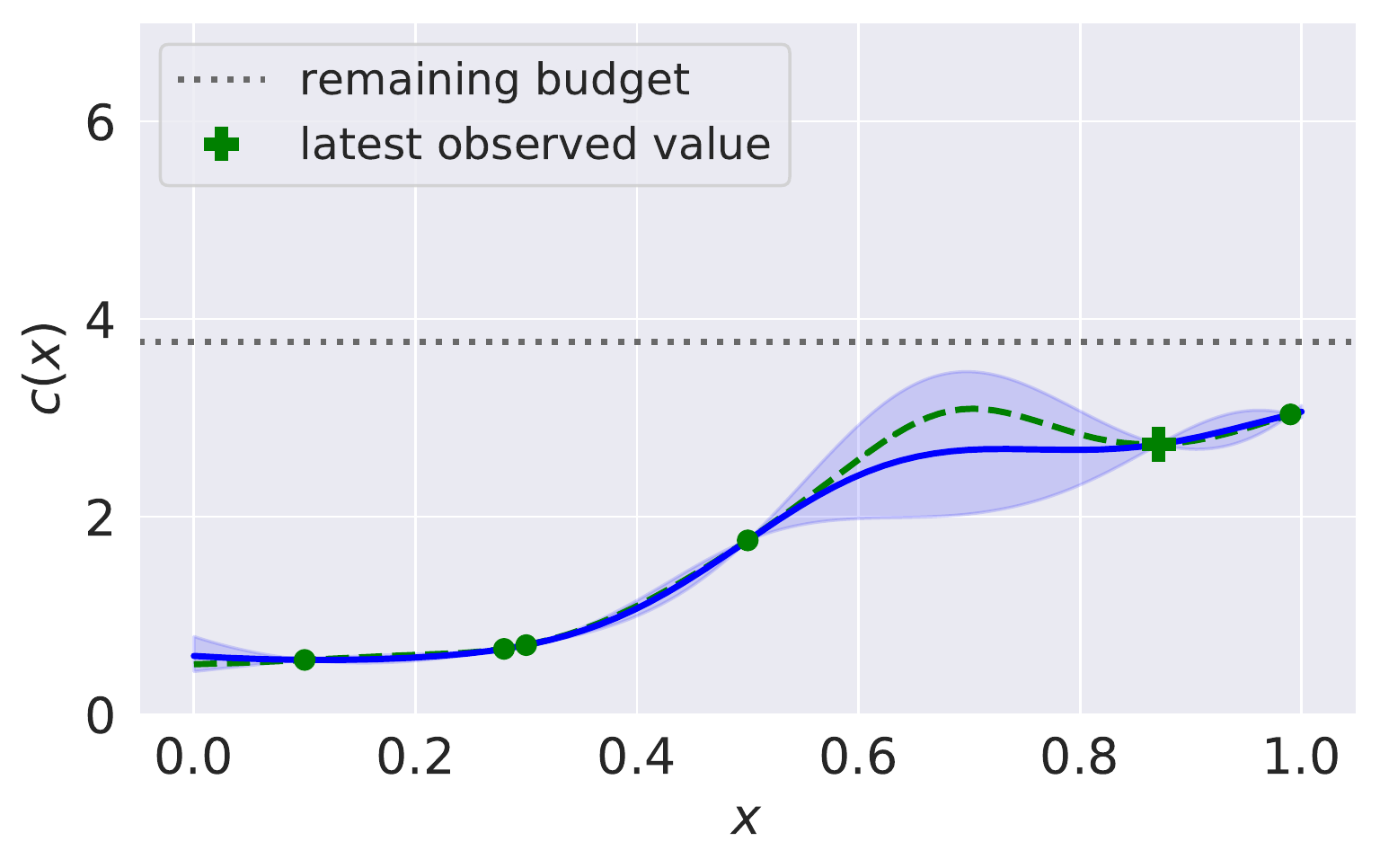}
  \includegraphics[width=0.33\textwidth]{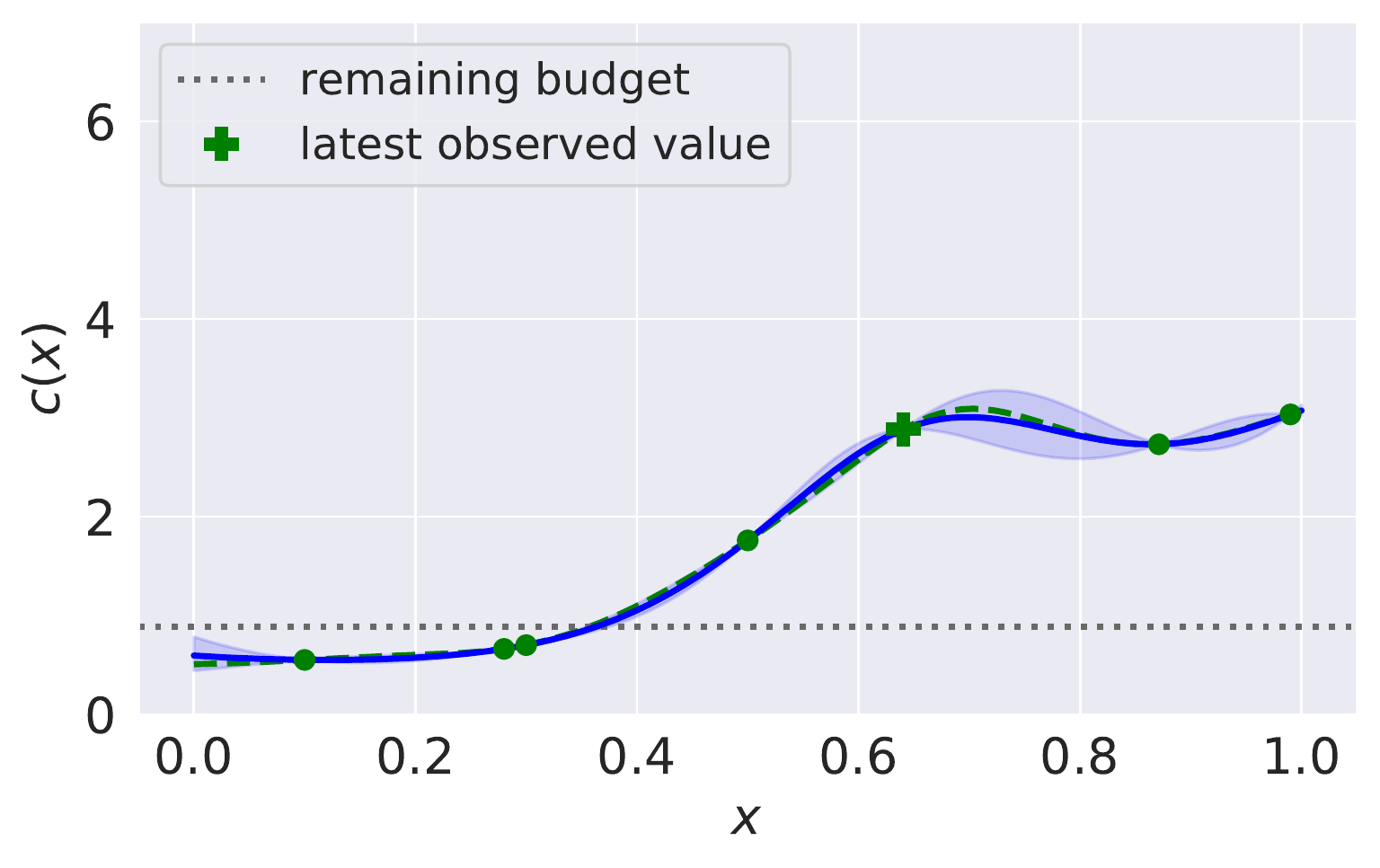}\\
  \includegraphics[width=0.33\textwidth]{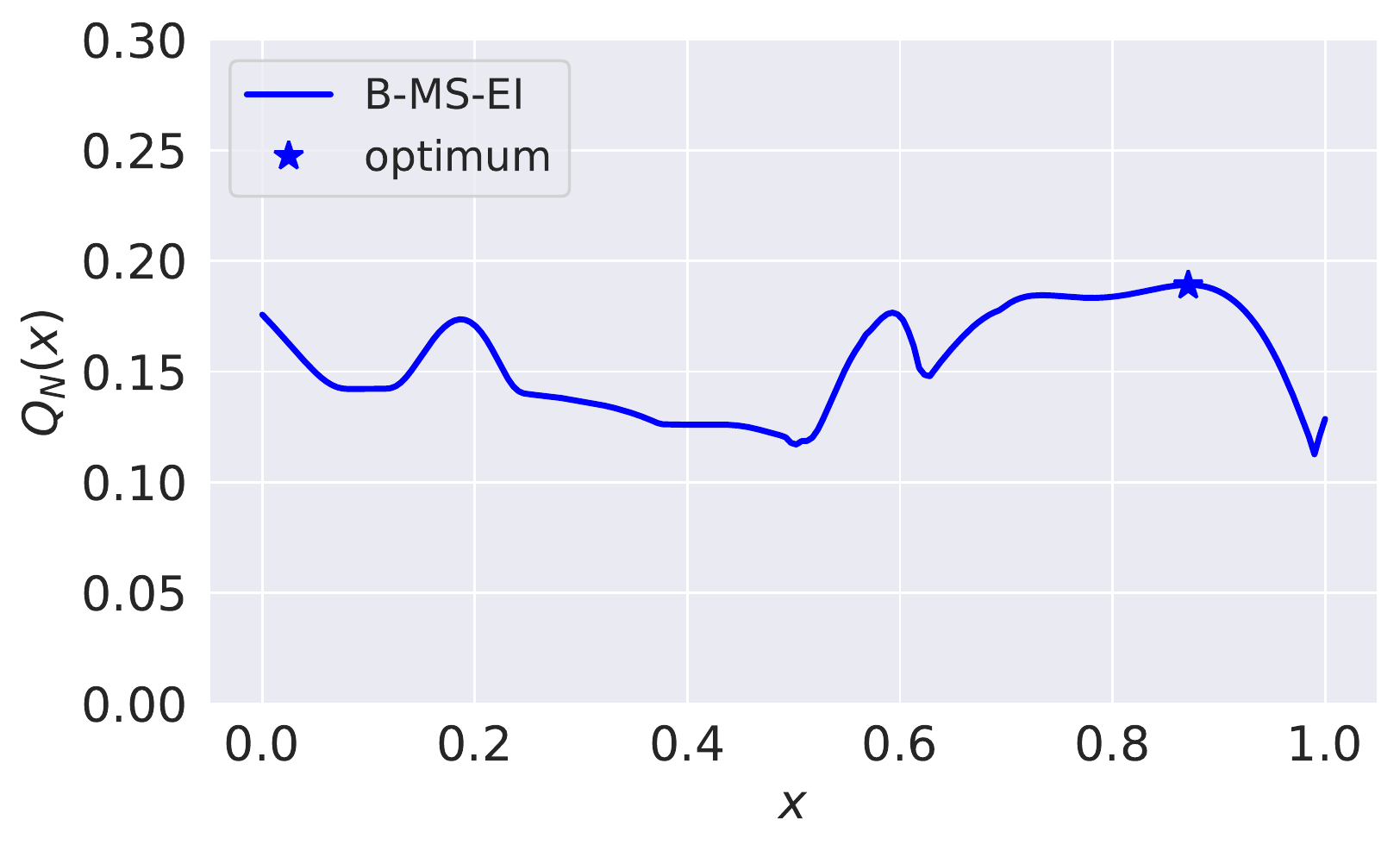}
  \includegraphics[width=0.33\textwidth]{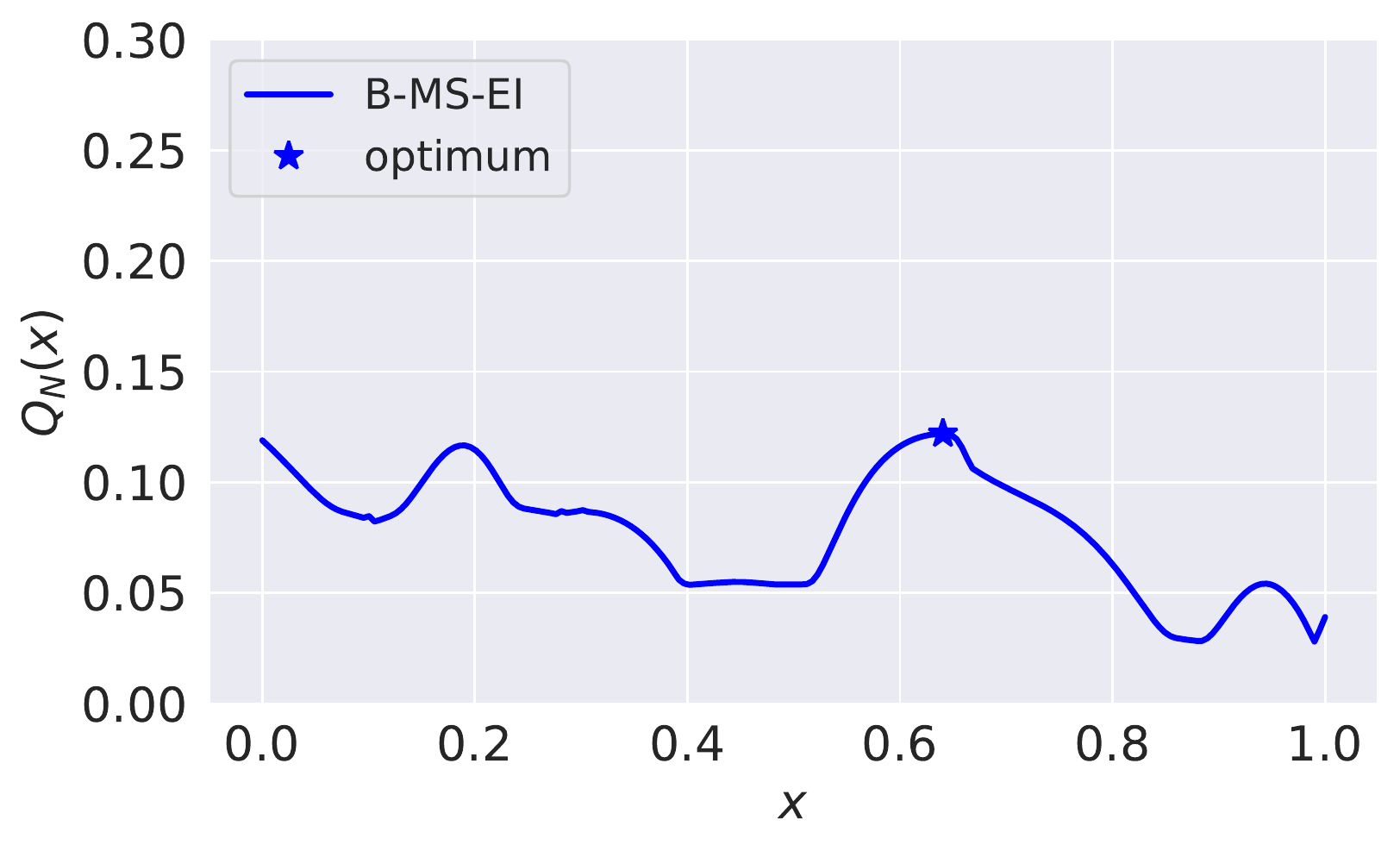}
  \includegraphics[width=0.33\textwidth]{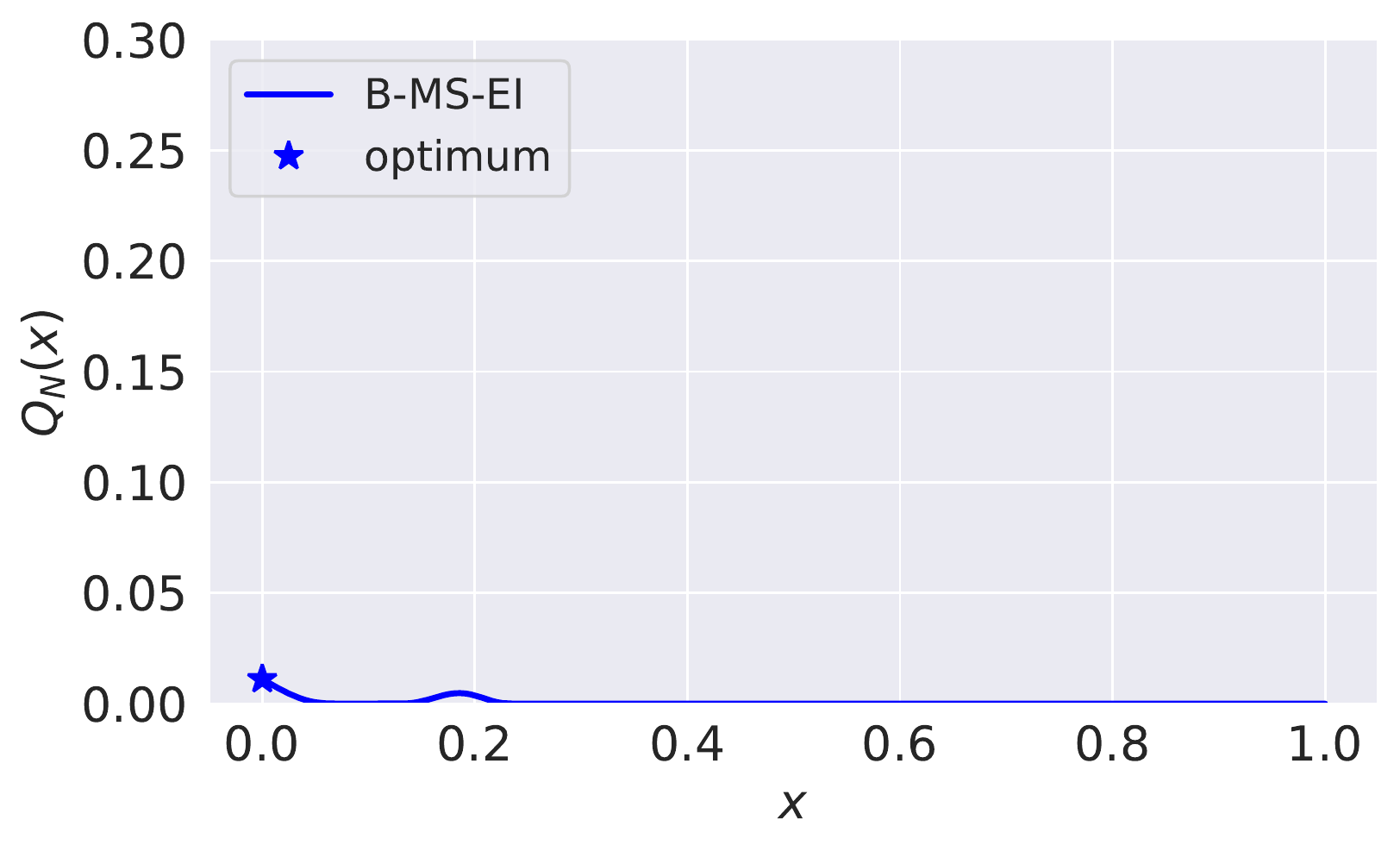}
\end{tabular}
}
\caption{The top six panels show the EI-PUC strategy and the bottom six show our \BMSEI{} strategy. In each group, the top and middle rows show the posterior on the objective and cost, respectively, and the bottom row shows the implied acquisition function. Time moves from left to right, where the first column shows the initial posteriors, and the second and third columns show the posteriors after one and two evaluations. EI-PUC evaluates low-variance, low-mean points that are low-cost but unlikely to reveal values near the global optimum. In contrast, B-MS-EI discovers a point near this global optimum within budget after two evaluations. While the remaining budget of EI-PUC after two evaluations is still relatively large, we include additional plots in Section~\ref{supp:additional} of the supplementary material showing that indeed B-MS-EI achieves a better objective value within budget.
\label{fig:animation}}
\end{figure}


Figure~\ref{fig:animation} illustrates this phenomenon in a continuous one-dimensional setting with a Gaussian process prior. Under the posterior in the first time slice, there is a lower-cost region on the left ($x < 0.5$) with low-variance and a mean that is significantly below the best point observed. There is a higher-cost region on the right ($x>0.5$) where the mean and variance are both larger. These means and variances are such that the global maximum of the function is likely to be in the right-hand region. In its first two measurements, the EI-PUC strategy measures at two low-cost points, which results in reduction of variance in this region but (as expected) no values that contend with the likely value of the global maximum on the right. In contrast, in its first two measurements, our multi-step B-MS-EI strategy evaluates on the right, finding values that are close to the global optimum.

\looseness-1 In this example, EI-PUC encounters the same difficulty it faces in the proof of Theorem~\ref{thm:counterexample}:  It overvalues low-cost points that improve relative to the status quo but not relative to where we hope to be near the end of the budget. While some budget remains to evaluate on the right after these evaluations, the evaluations on the left are likely unproductive toward the goal of finding a global maximum. 

\subsection{Decomposition and Truncation}
We now discuss how we can additively decompose and then truncate the problem in (\ref{eq:opt_policy}) so that it is more amenable as a BO acquisition function. Let $r(\D_{n-1}, \D_{n}) = u(\D_{n}) - u(\D_{n-1})$ be the increase in utility between successive states. Using a telescoping sum and rewriting the summation as an infinite sum, we have
\begin{align}
   V^*(\D) &= 
   \sup_{\pi\in\Pi} \E^{\pi}\Biggl[ \sum_{n=1}^{\infty} r(\D_{n-1},\D_n) \, \indicate{n \le N_B}  \, \Bigl | \, \D_0 = \D\Biggr].
   \label{eq:opt_telescope}
\end{align}
A truncated version of (\ref{eq:opt_telescope}) that is useful within our scenario tree optimization approach described in Section \ref{subsec:max} is given by
\begin{align}
   V_N(\D) = \sup_{\pi\in\Pi} \E^{\pi}\Biggl[ \sum_{n=1}^{N} r(\D_{n-1},\D_n) \, \indicate{n \le N_B} \, \Bigl | \, \D_0 = \D\Biggr],
   \label{eq:opt_policy_truncated_N}
\end{align}
where $N$ is a fixed number of ``look-ahead steps.'' When for each $c$ drawn from the prior, there exists a lower bound on $c(x)$ over $\mathbb X$ so that (\ref{eq:opt_policy}) is well-defined, it follows that the truncation becomes more accurate when $N$ becomes large: $\lim_{N \rightarrow \infty} V_N(\D) = V^*(\D)$. This serves as the motivation for our acquisition function described below.




\section{Budgeted Multi-Step Expected Improvement}
\label{sec:acqf}
\looseness-1 We now derive the \emph{budgeted multi-step expected improvement} (\BMSEI{}) acquisition function, where
the main idea is to solve the truncated problem given in (\ref{eq:opt_policy_truncated_N}). Accordingly, the acquisition function is parameterized by $N$, the maximum number of look-ahead steps. In Section~\ref{subsec:max}, we discuss how to approximately optimize this acquisition function using a scenario tree, leveraging the technique developed in \cite{jiang2020efficient}, and in Section~\ref{subsec:budget}, we discuss how to set the budget that defines our acquisition function when the actual remaining budget is too large to be meaningfully taken into account by a computationally feasible number of look-ahead steps.

\subsection{Dynamic Programming on the Truncated Problem}
\label{subsec:bellman}
Our derivation of \BMSEI{} starts with applying Bellman recursion to \eqref{eq:opt_policy_truncated_N}. To this end, we define the one-step marginal value of a measurement at point $x$ given an arbitrary set of observations $\D$ (i.e., a state-action value function, or $Q$-function) to be
\begin{align*}
    Q_1(x \mid \D) &= \E_{y,z}\left[ r(\D, \D \cup \{(x, y, z)\}) \indicate{s(\D) + z \leq B} \right]
    = \E_{y,z}\left[(y - u(\D))^+\; \indicate{s(\D) + z \leq B}\right].
\end{align*}
Proposition \ref{prop:bcei} below shows that, when $f$ and $\ln c$ are drawn from independent GPs, $Q_1$ admits an analytic expression similar to the one of constrained expected improvement \citep{schonlau1998global,gardner14}. The proof of this result can be found in Section~\ref{supp:prop1} of the supplementary material.

\begin{proposition}
\label{prop:bcei}
Suppose that $f$ and $\ln c$ follow independent Gaussian process prior distributions and that $\D$ is an arbitrary set of observations. Define $\mu^f_\D(x) = \mathbb{E}[ f(x) \,|\, \D]$, $\mu^{\ln c}_\D(x) = \mathbb{E}[ \ln c(x) \,|\, \D]$, $\sigma^f_\D(x) = \textnormal{Var}[ f(x) \,|\, \D]^{1/2}$, and $\sigma^{\ln c}_\D(x) = \textnormal{Var}[ \ln c(x) \,|\, \D]^{1/2}$. Then,
\begin{equation*}
    Q_1(x\mid \D) = \mathrm{EI}^{f}(x\,|\,\D) \, \Phi(\zeta)\, \indicate{s(\D) \le B} 
\end{equation*}
where $\textnormal{EI}^f$ is the classical expected improvement computed with respect to $f$, $\zeta = \bigl(\ln\left(B - s(\D)\right) - \mu^{\ln c}_{\D}(x)\bigr) / \sigma^{\ln c}_{\D}(x)$, 
and $\Phi$ is the standard normal cdf.
\end{proposition}

In contrast with homogeneous-cost non-myopic formulations, the indicator $\indicate{s(\D)\le B}$ truncates our reward sequence at the random time before the budget is first depleted. Analogously, we can define the $n$-step value function evaluated at $x$, $Q_n(x\mid \D)$, as the expected difference in utility after using the optimal policy to evaluate $n$ additional points, among which $x$ is the first of them. Using the Bellman recursion, one can write $Q_{n}(x\mid\D)$ as
\begin{align*}
    Q_{n}(x\mid\D) =  Q_1(x\mid \D) + \E_{y,z}\Bigl[\max_{x\in\X}Q_{n-1}\left(x\mid \D \cup \{(x, y, z)\}\right)\Bigr],
\end{align*}
which holds for all $n$. Our acquisition function \BMSEI{} is defined as $Q_N$, meaning that at every step our sampling policy evaluates $x_{n+1} \in \argmax_{x\in\X} Q_N(x \mid \D_n)$.

\subsection{Optimizing B-MS-EI via Budgeted One-Shot Multi-Step Trees}
\label{subsec:max}
Maximizing our acquisition function is challenging as, in principle, this requires nested stochastic optimization over a continuous domain. We build upon the multi-step scenario tree approach of \cite{jiang2020efficient} and devise an optimization method based on a sample average approximation of $\max_{x\in\X} Q_N(x\mid \D_n)$.
In our approach, each scenario is associated with its own decision variable, allowing the problem to be cast as a single deterministic optimization problem over a higher-dimensional domain instead of a sequence of nested stochastic optimization problems. We tackle the higher-dimensional problem using modern tools of automatic differentiation and batched linear algebra operations \citep{balandat2020botorch}. We begin by noting that, if we apply Bellman's recursion repeatedly, $Q_N$ can be rewritten as
\begin{align*}
    Q_{N}(x\mid\D) =  Q_1(x\mid \D) + \E_{y, z}\Bigl[\max_{x_2}\bigl\{Q_1\left(x_2\mid \D_1\right) +
     \E_{y,z}\bigl[\max_{x_3}\bigl\{Q_1\bigl(x_3\mid \D_2 \bigr) + \cdots \bigl\} \bigl ] \bigl\} \Bigr].
\end{align*}

Now we consider the Monte Carlo approximation of $Q_{N}(x\mid\D)$ given by
\begin{align*}
\widehat{Q}_{N}&(x\!\mid\!\D)\\
&= Q_1(x\!\mid\!\D) +  \frac{1}{m_1}\!\sum_{j_1}^{m_1}\!\bigg[\max_{x_2^{j_1}}\biggl\{Q_1(x_2^{j_1}\!\mid\! \D_1^{j_1}) +
 \frac{1}{m_2}\sum_{j_2}^{m_2}\biggl[\max_{x_3^{j_1 j_2}}\Bigl\{Q_1(x_3^{j_1 j_2}\!\mid\!\D_2^{j_1 j_2})\!+ \!\cdots\Bigr\}\biggr]\biggr\}\biggr],
\end{align*}
where $m_i, \ i = 1, \ldots, N-1$, is the number of samples used in step $i$, and the sets of observations are defined recursively by the equations $\D_i^{j_1} = \D \cup \bigl\{\bigl(x, y_i^{j_1}, z_i^{j_1}\bigr)\bigr\}$ and
\begin{equation*}
   \D_i^{j_1 \ldots j_i} = \D_i^{j_1 \ldots j_{i-1}} \cup \bigl\{\bigl(x_i^{j_1 \ldots j_{i-1}}, y_i^{j_1 \ldots j_i}, z_i^{j_1 \ldots j_i}\bigr)\bigr\}, 
\end{equation*}
with \textit{fantasy samples} (i.e., samples drawn from the model) $(y_i^{j_1 \ldots j_i}, z_i^{j_1 \ldots j_i}) \sim   p(\cdot \mid x_i^{j_1 \ldots j_{i-1}},  \D_i^{j_1 \ldots j_{i-1}})$
obtained via the \textit{reparametrization trick}\footnote{Note that since we employ Monte Carlo sampling to approximate $Q_N(x\mid \D_n)$, the independence assumption between $f$ and $c$ in Proposition~\ref{prop:bcei} is not necessary for our approach, and one can jointly model $f$ and $c$ (e.g. with a multi-task GP) if there is reason to believe that these functions are correlated, as is often the case in practice. For instance, randomness in training jobs based on variations in test-train splits or weight initialization may interact with adaptive learning rate scheduling and thus affect both model performance and training time.} \citep{kingma2013auto,wilson2018maximizing}. 



\begin{figure}
\floatbox[{\capbeside\thisfloatsetup{capbesideposition={left,center},capbesidewidth=0.55\textwidth}}]{figure}[\FBwidth]
{\caption{An illustration of the scenario tree representation of our acquisition function for $N=5$ stages, where $(m_1, m_2, m_3, m_4)=(4,2,2,1)$.
Some paths down the tree stop accumulating value earlier than others due to the budget being exhausted, i.e., when $s(\D_i^{j_1 \ldots j_i}) > B$ (ending in red circle nodes). Other paths may terminate due to the maximum look-ahead $N$ before the budget is exhausted (ending in black square nodes), representing some approximation error due to the truncation.}\label{fig:tree}}
{\includegraphics[clip,width=0.4\textwidth]{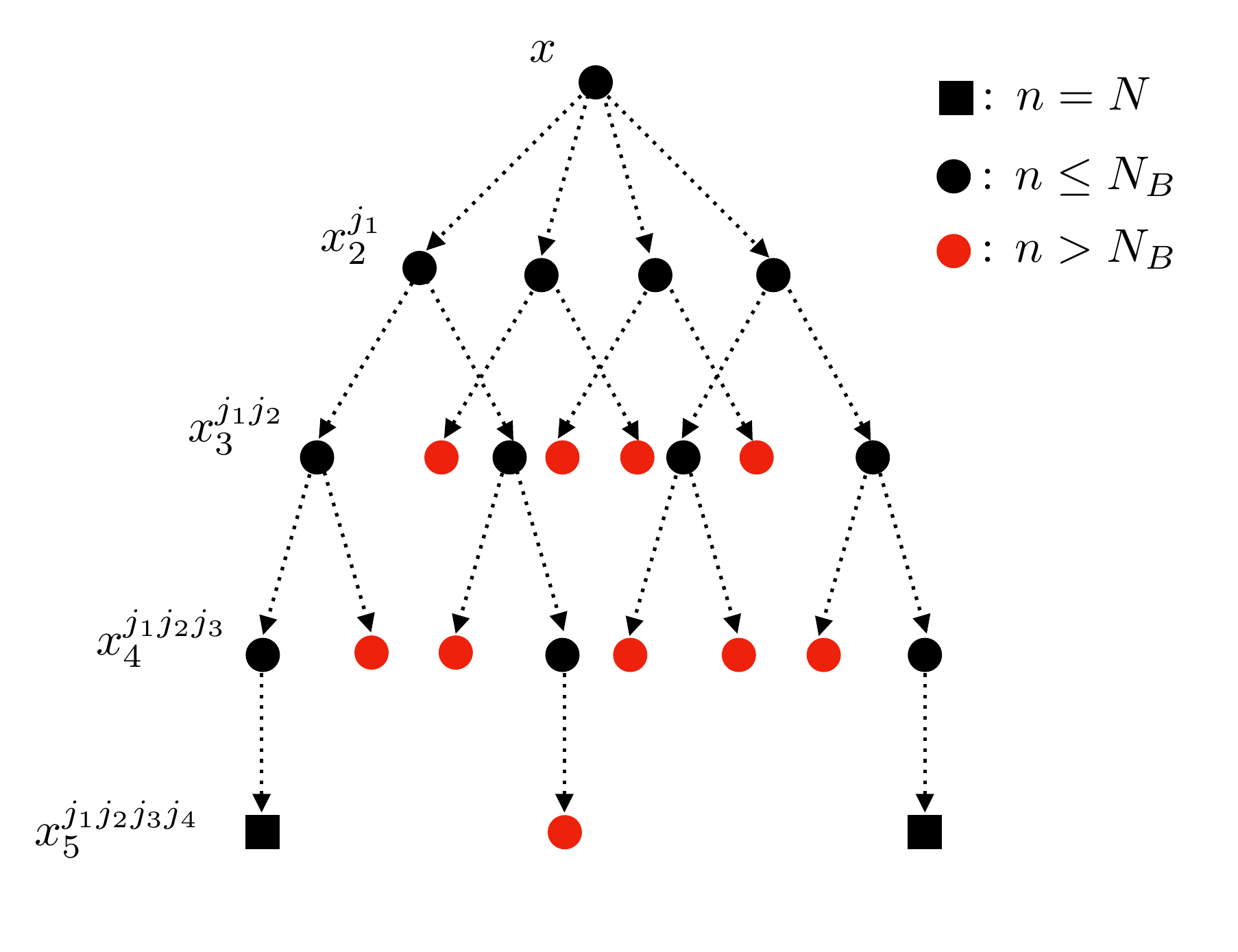}}
\end{figure}

\subsection{Budget Scheduling via Rollout of Base Sampling Policy}
\label{subsec:budget}
\looseness-1 Let $B_n = B - \sum_{i=1}^n z_i$ be the remaining budget at time $n$. Since the number of look-ahead steps, $N$, that can be performed in practice is relatively small due to computational limits, $B_n$  may be too large to be taken into account by our acquisition function in a meaningful way, especially during the first few evaluations (if the remaining budget is too large relative to the number of look-ahead steps, then the evaluation costs have little effect in our sampling decisions). Therefore, instead of  using the actual remaining budget at step $n$, our acquisition function uses a \textit{fantasy} budget set by a heuristic rule.

We propose a heuristic budgeting rule based on the use of a \textit{base sampling policy} that can be computed quickly. We fantasize $N$ sequential evaluations under this base sampling policy. More specifically, for each of the $N$ steps, we draw fantasy objective and cost values at the point recommended by this policy using the (joint) posterior distribution, then we update the posterior distribution by incorporating these fantasy evaluations, and repeat. The sum of the resulting $N$ fantasy costs determines a cumulative cost incurred by the base sampling policy. Our acquisition function then sets the budget to be the minimum between this cumulative cost and the true remaining budget. This budget is used by our method as if it were the actual budget until depletion, and then the heuristic rule is used again to compute a new fantasy budget. The intuition behind this heuristic is that our acquisition function will try to find the best non-myopic decision using the same budget as the base sampling policy, and thus should perform evaluations that are better or at least as good as those of the base sampling policy. In our experiments, we use EI-PUC-CC as the base policy.

\section{Experiments}
\label{sec:experiments}
We demonstrate the efficacy of \BMSEI{} on four synthetic and four real-world experiments. We report the performance of two variants of our acquisition function, the main variant that uses multiple fantasy samples per step, and the \textit{multi-step path} variant \citep{jiang2020efficient}, which is based on a degenerate tree with one fantasy sample per step. For each of these variants, we consider both $N=2$ and $N=4$ look-ahead steps. We denote the multi-step path variant with $N$ steps as $N$-B-MS-EI$_{\textnormal{p}}$ and the former simply as $N$-B-MS-EI.

In addition, we report the performance of three baseline acquisition functions from the literature: expected improvement (EI), expected improvement per unit of cost (EI-PUC) \citep{Snoek2012ML, lee2020cost}, and  expected improvement per unit of cost with cost cooling (EI-PUC-CC) \citep{lee2020cost}. EI-PUC and EI-PUC-CC are currently the de-facto standard approach to cost-aware BO. Since \cite{lee2020cost} assumes that the cost is known (while our experiments are run in the setting of unknown costs), we also integrate with respect to the uncertainty on the cost when computing these acquisition functions. Closed form analytical expressions for these acquisition functions can be found in Section~\ref{supp:closed_form} of the supplementary material.

\subsection{Description of Benchmark Problems}
\label{sec:description_benchmark}

\textbf{Synthetic Test Problems.} The first four problems use synthetic objective functions commonly found in the literature: Dropwave, Alpine1, Ackley, and Shekel5  (defined in Section~\ref{supp:synthetic} of the supplementary material). Each objective function is accompanied by a cost function of the form $c(x) = \exp\bigl[\frac{\alpha}{d} \sum_{i=1}^d \cos\bigl(\beta (x_i -x_i^* + \gamma)\bigr)\bigr]$,
where $x^*$ is the objective function's maximizer. In each replication we use a different cost function, obtained by sampling $\alpha$, $\beta$, and $\gamma$ uniformly at random over appropriate intervals. Our results thus average over a variety of cost functions. As these parameters vary, cost functions with different characteristics arise. More concretely, $\alpha$ controls the magnitude of the difference between the minimum and maximum costs; $\beta$ controls the variability of the cost; and $\gamma$ controls the cost at the optimum. In particular, if $\gamma = 0$, then $x^*$ is the most expensive point, whereas if $\gamma = \pi$, then $x^*$ is the cheapest point. This family of cost functions emulates a wide range of possible scenarios. 
\begin{figure}
  \centering
 \includegraphics[width=0.24\textwidth]{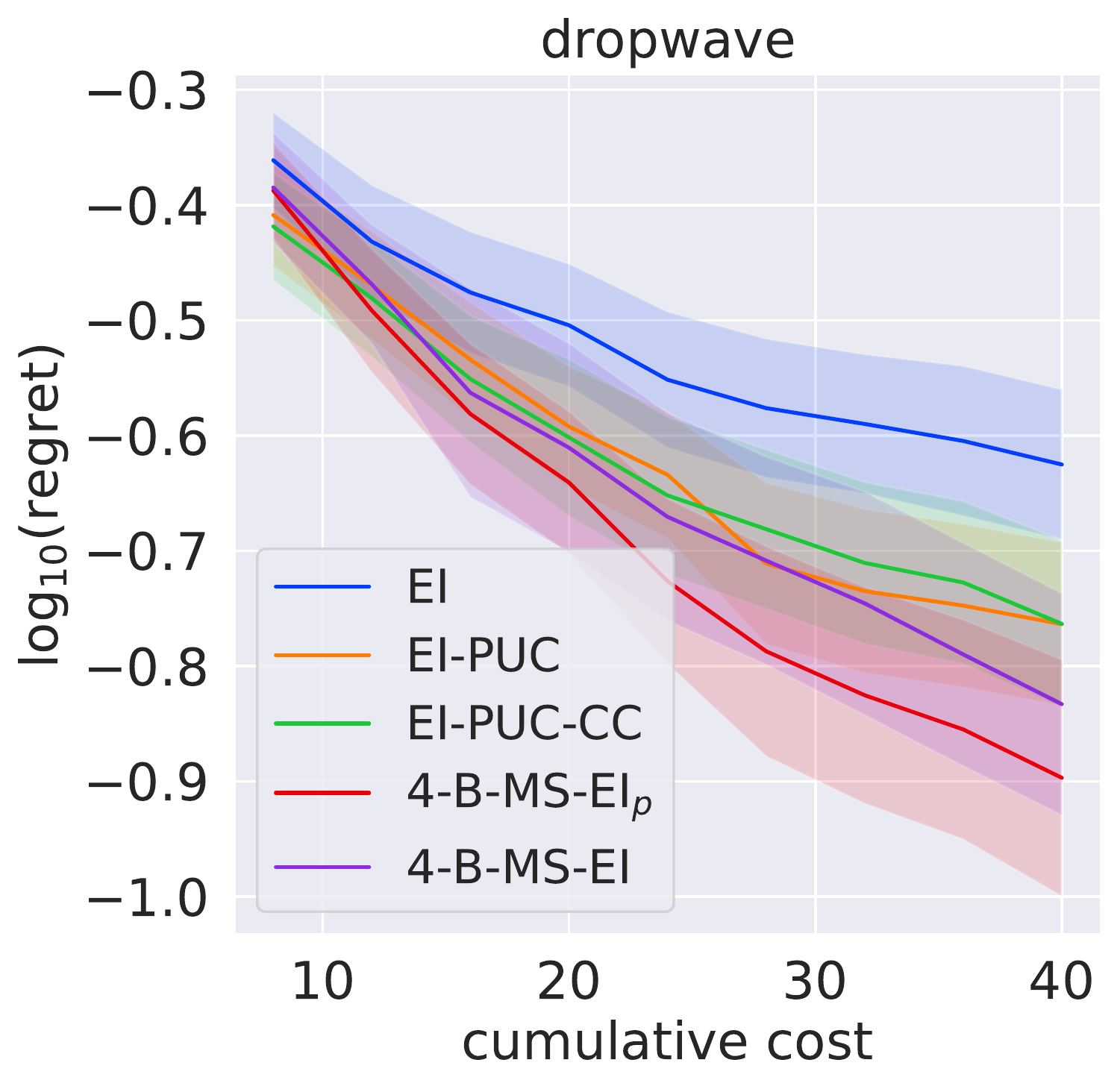}
 \includegraphics[width=0.24\textwidth]{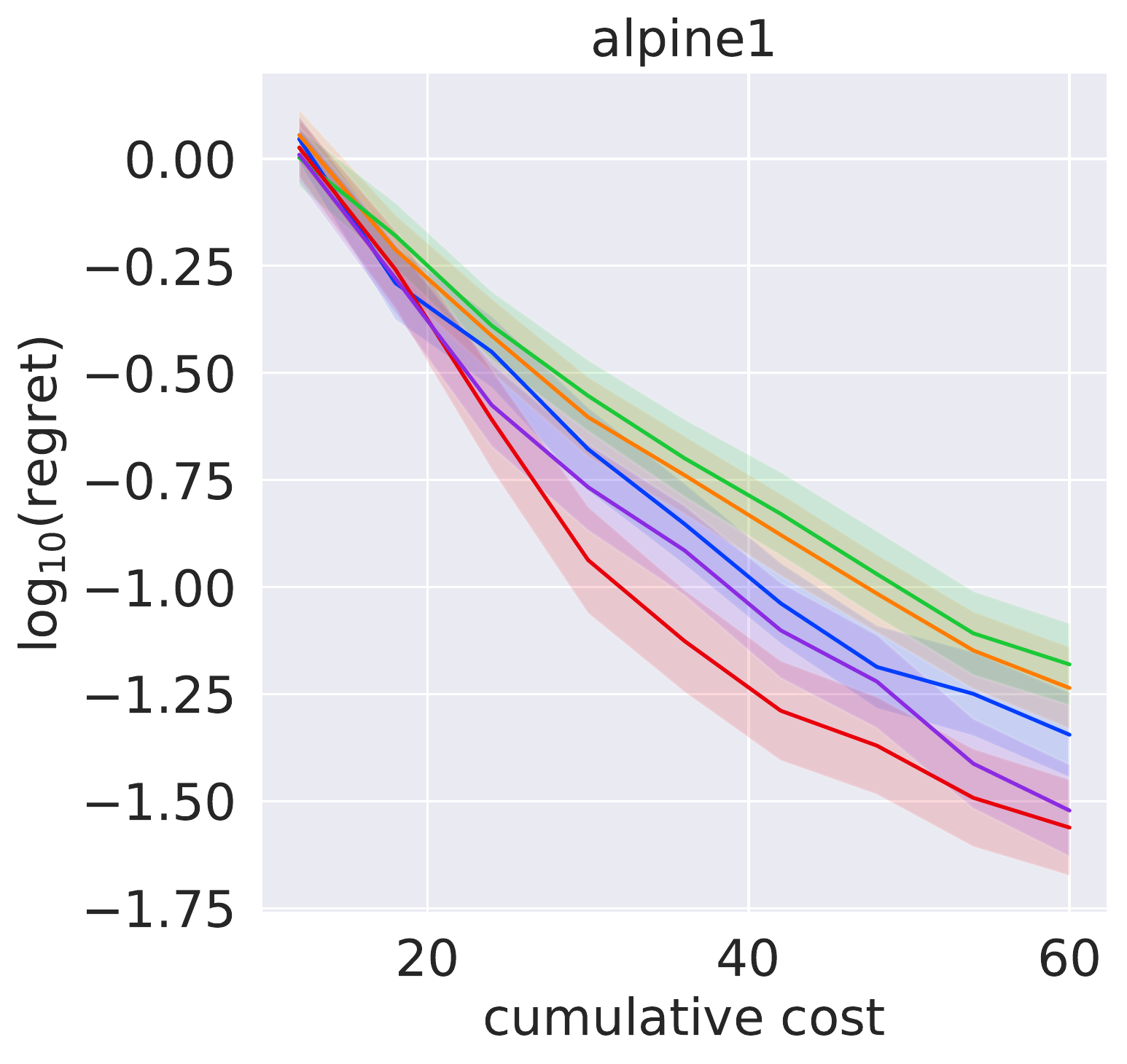}
 \includegraphics[width=0.24\textwidth]{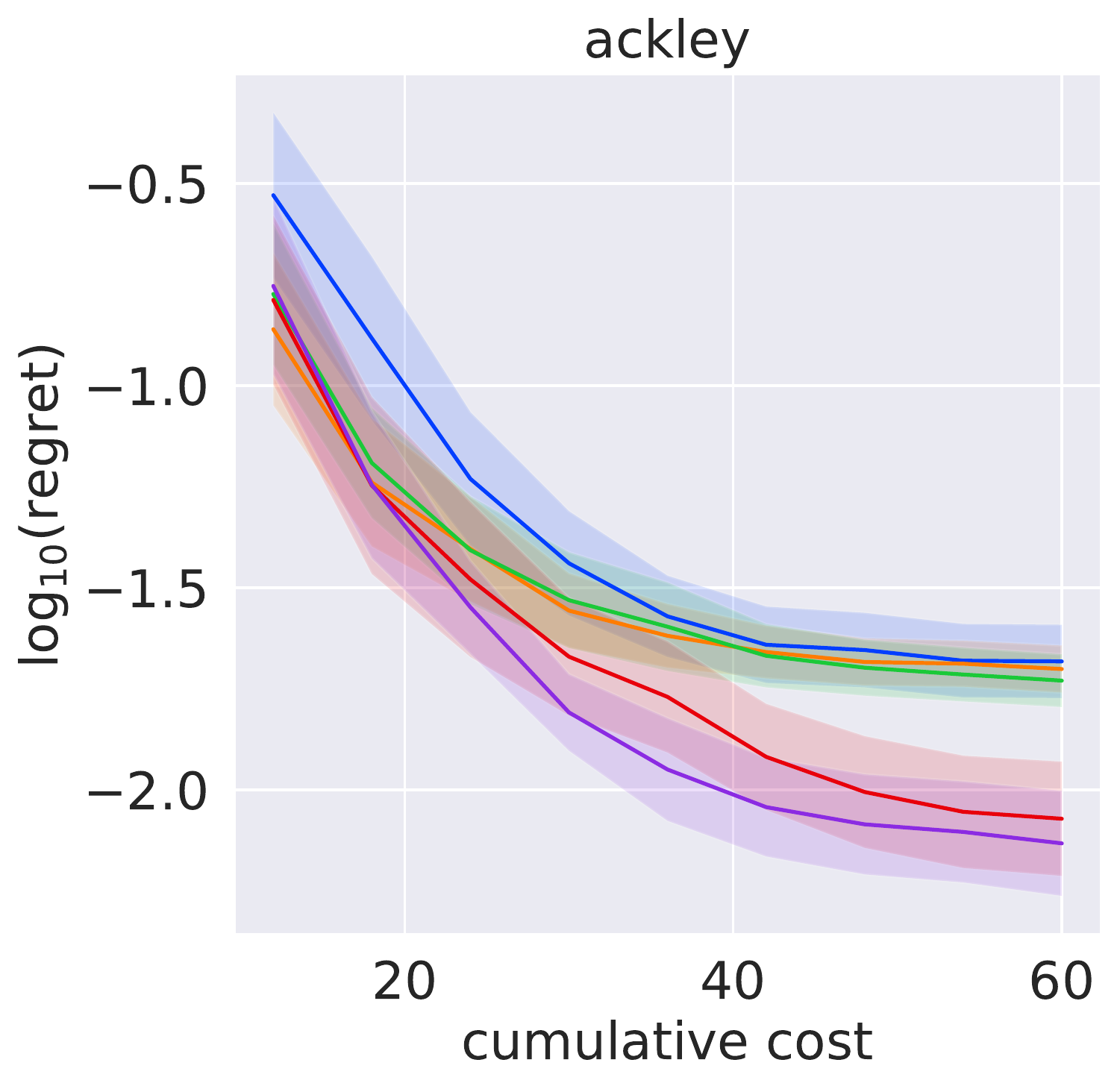}
 \includegraphics[width=0.24\textwidth]{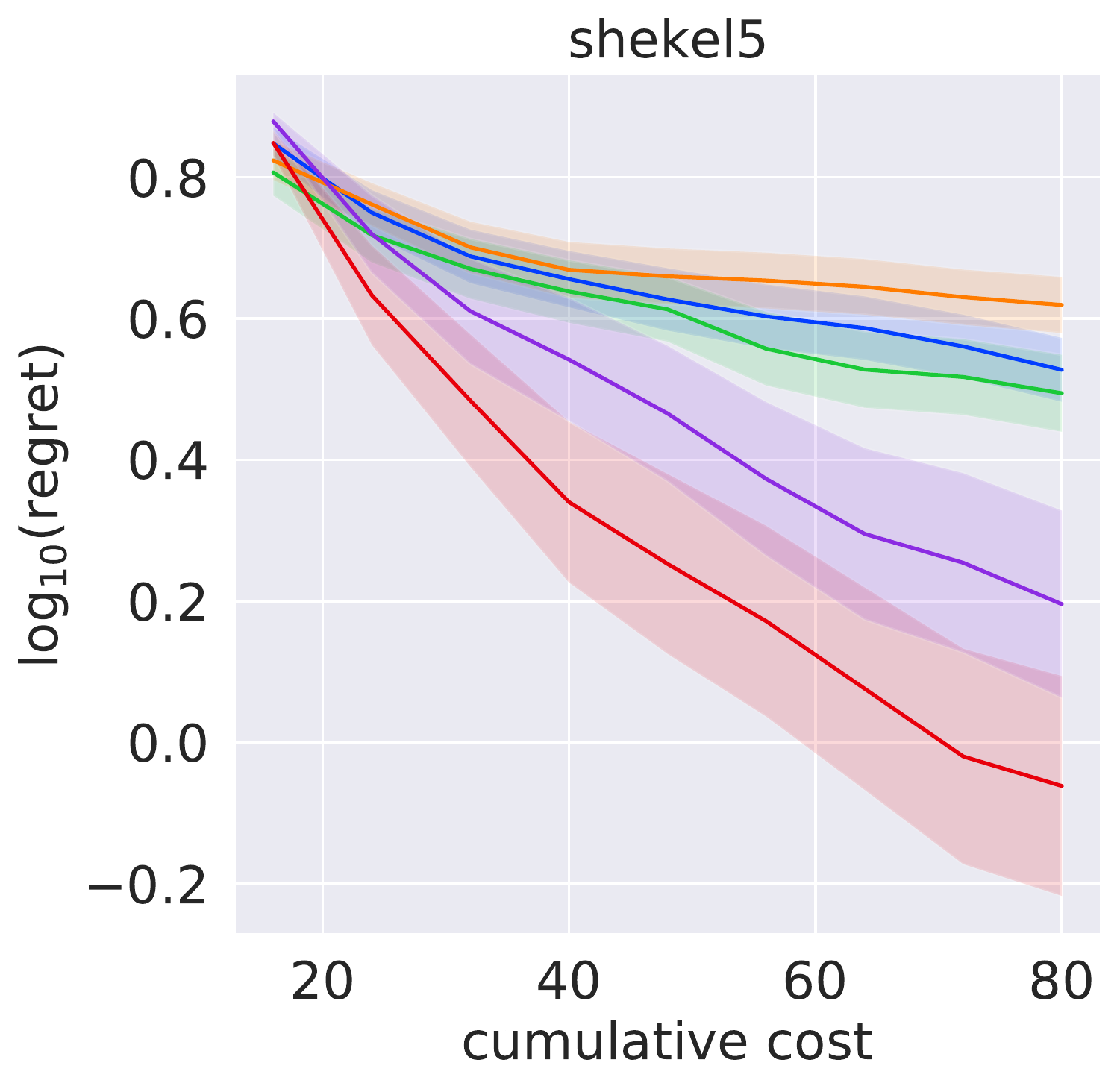}
 \\
 \includegraphics[width=0.235\textwidth]{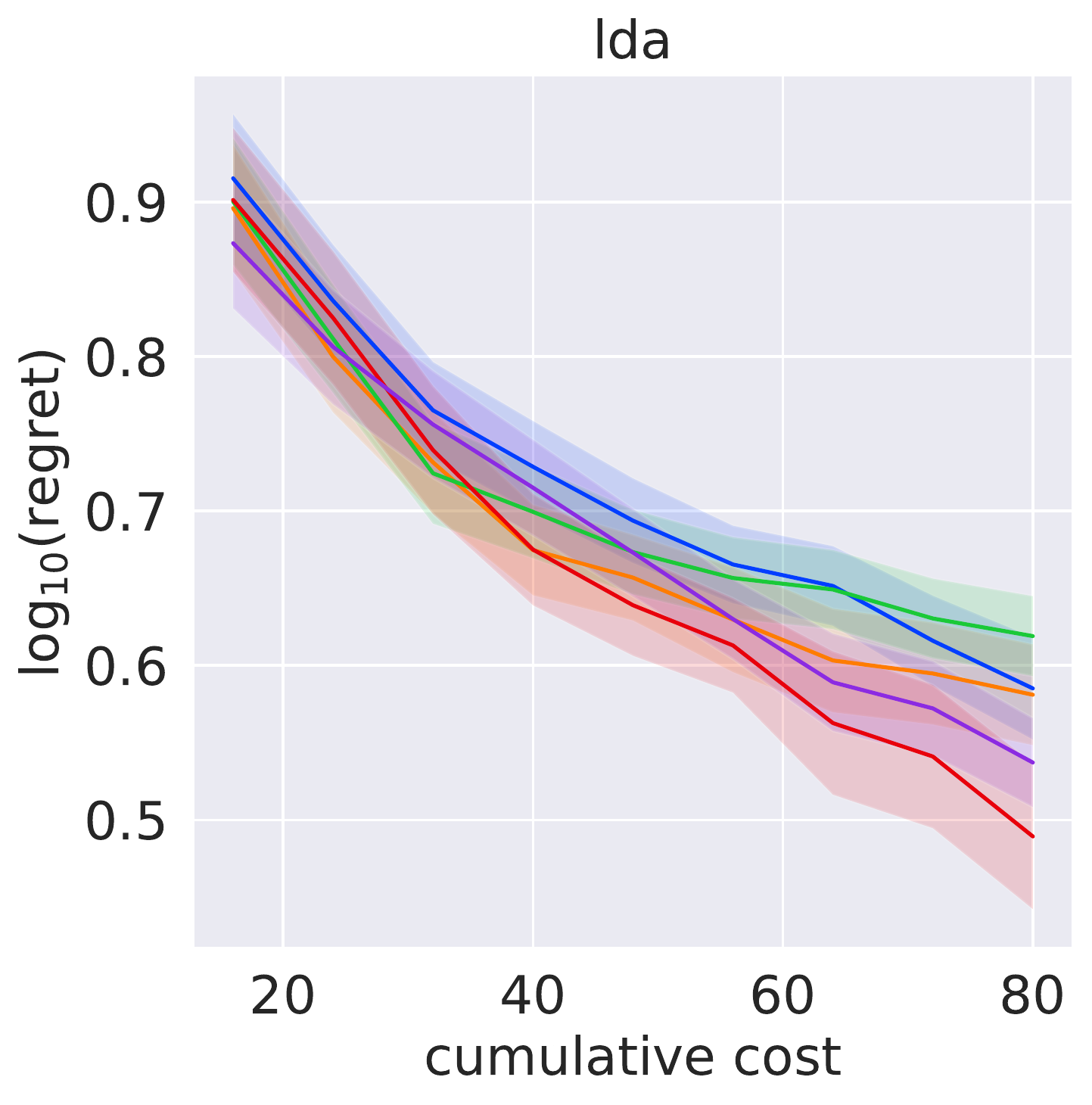}
 \includegraphics[width=0.245\textwidth]{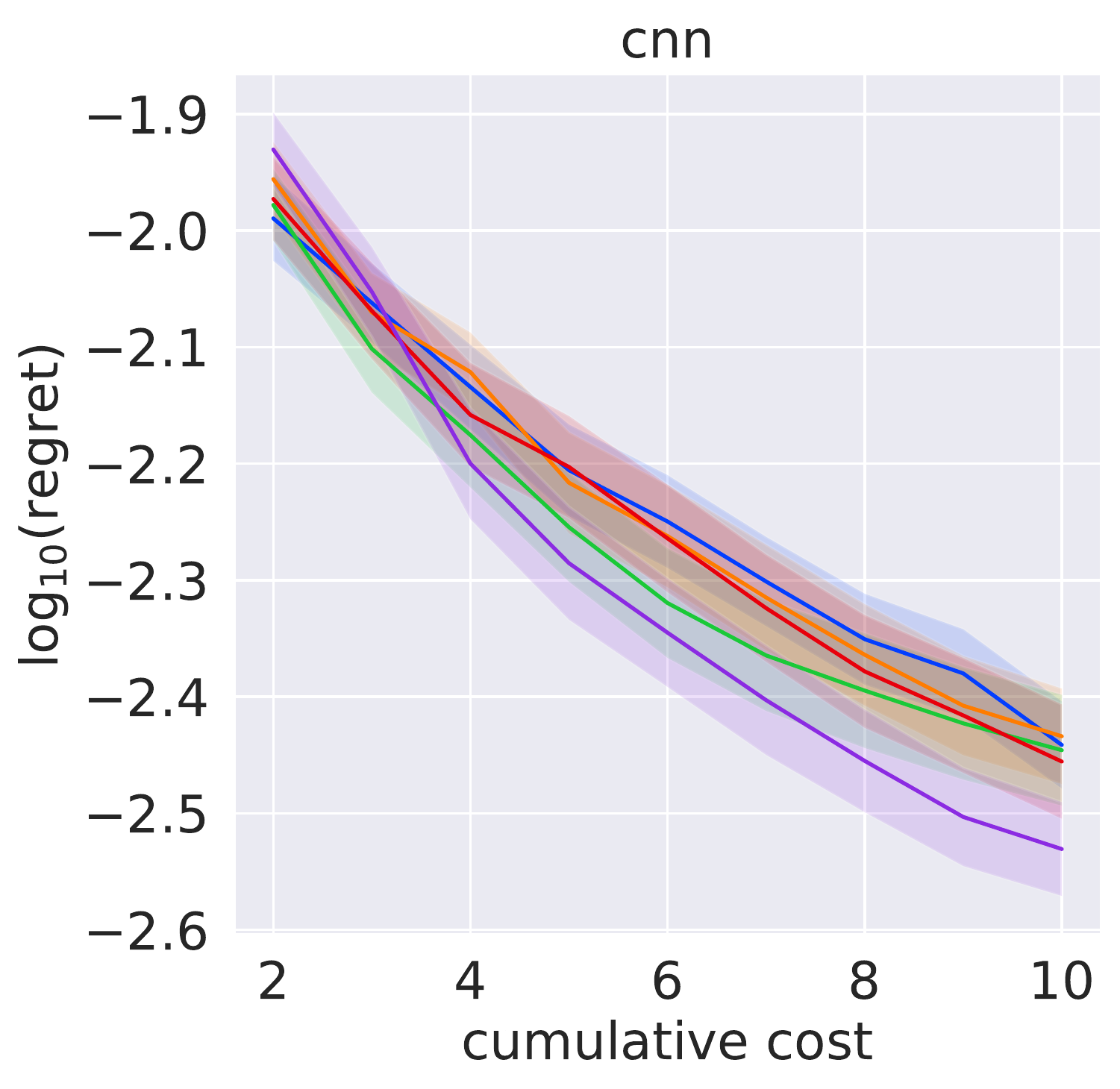}
  \includegraphics[width=0.25\textwidth]{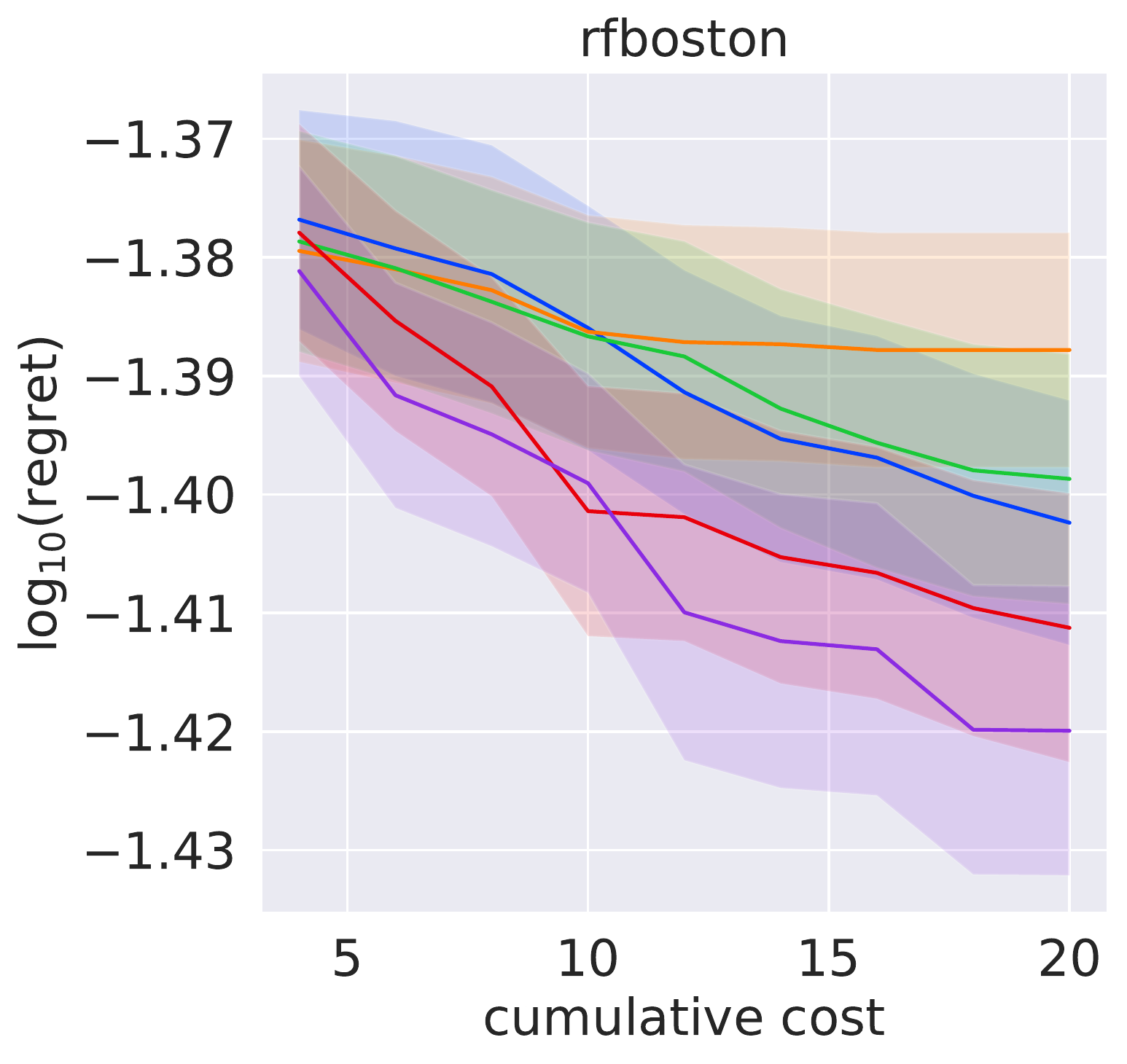}
\includegraphics[width=0.245\textwidth]{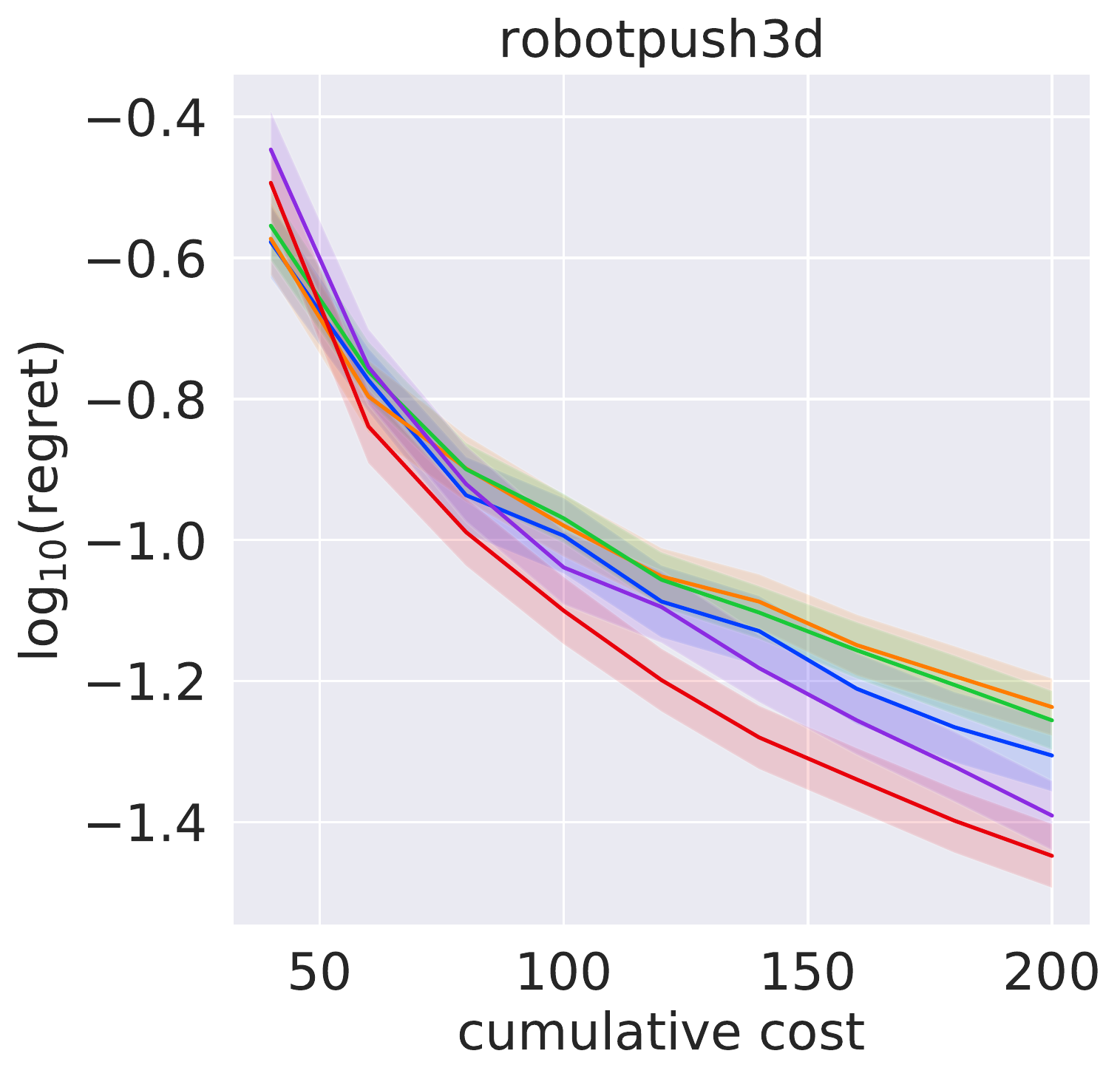}
  \caption{Log-regret of our non-myopic budget-aware BO method (4-B-MS-EI) compared with baseline acquisition functions on a range of problems. \label{fig:experiments}}
\end{figure}

\textbf{AutoML Benchmarks.} We consider three AutoML benchmark problems: LDA, CNN, and RF-Boston. The LDA and CNN problems use publicly available data sets from the HPOLib \citep{hpolib} and HPOLib1.5 \citep{hpolib15} hyperparameter optimization libraries. Following \cite{eggensperger2018efficient}, we use  surrogates of the underlying objective and cost functions to emulate the computationally expensive process of training the corresponding models. Details on the construction of the surrogates can be found in Section~\ref{supp:automl} of the supplementary material. 

The first data set, originally introduced by \cite{hoffman2010online}, was obtained by evaluating an online latent Dirichlet allocation algorithm for topic modeling with 3 hyperparameters: mini-batch size $S \in [1, 16384]$ (on a log2 scale), and learning rate parameters $\kappa \in [0.5, 1.0]$ (controls speed at which information is forgotten) and $\tau_0 \in [1, 1024]$ (on a log2 scale, downweights early iterations). These evaluations are expensive and heterogeneous, ranging from 2 to 10 hours each (see Figure~\ref{fig:lda_costs_hist}). 

The second data set was obtained by training a 3-layer convolutional neural network on the CIFAR-10 dataset with 5 hyperparameters: ``initial learning rate'' $\in [10^{-6}, 1.0]$ (on a log10 scale), ``batch size'' $\in [32, 512]$, and ``number of units in layer $k$'' $\in [16, 256]$ (on a log2 scale), for $k=1,2,3$. T

Finally, the RF-Boston problem considers optimization of the 5-fold cross validation error when using a random forest (RF) regressor
on the Boston dataset, both of which are from \texttt{sklearn} \citep{scikit-learn}.
We tune the following hyperparameters of the RF regressor: ``n\_estimators'' $\in [1, 256]$ (rounded to the nearest integer), ``max\_depth'' $\in [1, 64]$ (rounded to nearest integer), and ``max\_features'' $\in [0.1, 1]$ (on a log10 scale). The cost of each is evaluation is proportional to the training time of the model under the evaluated set of hyperparameters (we scale it by training time of a single initial evaluation). 


\looseness-1 \textbf{Energy-Aware Robot Pushing.} We consider a version of the robot pushing problem  introduced by \cite{wang2017max}, where a robot pushes an object from its origin, $w_{\mathrm{init}}=(0,0)$, to a target $w_{\mathrm{target}}\in\R^2$. In the our version, we draw 20 different target locations uniformly at random over $\{(w_1,w_2)\in\R^2: 1< |w_1|,|w_2| < 5\}$, which are distributed equally across the replications performed. The parameters to be optimized are the location of the robot, $z\in[-5,5]^2$, and the duration of the push, $t\in[1,30]$. The objective to minimize is the distance from the object's location after being pushed, $w_{\mathrm{push}}(z,t)$, and the target location, i.e., $f(z,t) = \|w_{\mathrm{target}} - w_{\mathrm{push}}(z,t)\|_2 $. The cost function is given by $c(z,t)= \|w_{\mathrm{push}}(z,t)\| + \epsilon$, where $\epsilon$ is a constant, and represents the energy spent by the robot by pushing the object.

\subsection{Results and Discussion}

Figure~\ref{fig:experiments}
plots a 95\% confidence interval on mean log-regret versus the budget used thus far. We perform experiments using a single budget in each problem; the goal 
is to minimize regret at the point when the budget is fully exhausted at the right-hand edge of each plot. To focus attention on  these budgets, we plot results over the range from 20\% to 100\% of this overall budget.
The results 
for $N=2$ look-ahead steps are deferred Section~\ref{supp:additional} to the supplementary material to improve readability and also because they are outperformed by the $N=4$ counterparts for most problems. 
All implementations use BoTorch \citep{balandat2020botorch}. The objective and  log-cost functions are modeled using independent GPs. Additional details and runtimes can also be found in Section~\ref{supp:additional} of the supplementary material. An implementation of our algorithms and numerical experiments can be found at \url{https://github.com/RaulAstudillo06/BudgetedBO}.

B-MS-EI (red and purple lines) performs favorably compared to existing methods. In some cases (Dropwave and CNN) EI performs worst, while in other cases (Alpine, Shekel, LDA), one of the value-to-cost ratio methods (EI-PUC or EI-PUC-CC) performs worst.
\looseness-1 B-MS-EI is more computationally intensive to optimize than the standard approaches. The average walltime per acquisition of 4-B-MS-EI, the most expensive variant of our algorithm, ranges from 42 seconds to 7 minutes across the problems, whereas the average walltimes for EI, EI-PUC, and EI-PUC-CC, range from 1 to 30 seconds. However, optimization times on the order of minutes\footnote{These times are not unusual for advanced BO methods; see, e.g., Section 4 of \cite{wu2019practical}.} are acceptable when the acquisition function is applied to real-world problems where the function evaluation often takes much longer (e.g., in Figure~\ref{fig:lda_costs_hist}, LDA evaluations require up to 10 hours). 
In such settings, 
extra computational time spent optimizing the acquisition function more than pays for itself  
with improved query efficiency and the savings in objective function evaluation time that results.
Nevertheless, there are opportunities to reduce the computational requirements of B-MS-EI; possibilities include re-using the optimized tree for multiple steps or devising better initial conditions for optimization. We leave these for future work.


\section{Conclusion}
\label{sec:conclusion}
We introduced the problem of budgeted BO under unknown and potentially heterogeneous evaluation costs. This arises in a wide range of practical settings where BO is applied. However, the most common paradigm for cost-aware BO, the value-to-cost ratio, fails to account for uncertainty in the cost function as well as the presence of the budget constraint as part of the exploration-exploitation trade-off in a principled way, as we show in Theorem~1. In this work, we provided a dynamic programming formulation of this problem, along with a non-myopic acquisition function that addresses the shortcomings of the existing methods. Our acquisition function is derived by truncating the aforementioned dynamic program, and is approximately maximized following a one-shot multi-step tree approach. Our experiments show that our acquisition function outperforms existing methods on a number of synthetic and real-world examples that exhibit heterogeneity in evaluation costs. In the future, we hope to extend our approach to the multi-fidelity setting, where evaluation costs play an even larger role.

\section*{Acknowledgments}
Authors RA and PF were partially supported by AFOSR FA9550-19-1-0283 and FA9550-20-1-0351.

\bibliographystyle{apalike} 
\bibliography{bibl} 

\begin{thebibliography}{}

\bibitem[Badanidiyuru et~al., 2013]{badanidiyuru2013bandits}
Badanidiyuru, A., Kleinberg, R., and Slivkins, A. (2013).
\newblock Bandits with knapsacks.
\newblock In {\em 2013 IEEE 54th Annual Symposium on Foundations of Computer
  Science}, pages 207--216. IEEE.

\bibitem[Balandat et~al., 2020]{balandat2020botorch}
Balandat, M., Karrer, B., Jiang, D., Daulton, S., Letham, B., Wilson, A.~G.,
  and Bakshy, E. (2020).
\newblock Botorch: A framework for efficient monte-carlo bayesian optimization.
\newblock In Larochelle, H., Ranzato, M., Hadsell, R., Balcan, M.~F., and Lin,
  H., editors, {\em Advances in Neural Information Processing Systems},
  volume~33, pages 21524--21538. Curran Associates, Inc.

\bibitem[Bertsekas and Tsitsiklis, 1991]{bertsekas1991analysis}
Bertsekas, D.~P. and Tsitsiklis, J.~N. (1991).
\newblock An analysis of stochastic shortest path problems.
\newblock {\em Mathematics of Operations Research}, 16(3):580--595.

\bibitem[Byrd et~al., 1995]{byrd1995limited}
Byrd, R.~H., Lu, P., Nocedal, J., and Zhu, C. (1995).
\newblock {A Limited Memory Algorithm for Bound Constrained Optimization}.
\newblock {\em SIAM Journal on Scientific Computing}, 16(5):1190--1208.

\bibitem[Calandra et~al., 2016]{calandra2016bayesian}
Calandra, R., Seyfarth, A., Peters, J., and Deisenroth, M.~P. (2016).
\newblock Bayesian optimization for learning gaits under uncertainty.
\newblock {\em Annals of Mathematics and Artificial Intelligence}, 76(1):5--23.

\bibitem[Eggensperger et~al., 2018]{eggensperger2018efficient}
Eggensperger, K., Lindauer, M., Hoos, H.~H., Hutter, F., and Leyton-Brown, K.
  (2018).
\newblock Efficient benchmarking of algorithm configurators via model-based
  surrogates.
\newblock {\em Machine Learning}, 107(1):15--41.

\bibitem[Field, 1999]{field1999practical}
Field, M.~J. (1999).
\newblock {\em A practical introduction to the simulation of molecular
  systems}.
\newblock Cambridge University Press.

\bibitem[Frazier, 2018]{frazier2018tutorial}
Frazier, P.~I. (2018).
\newblock A tutorial on {B}ayesian optimization.
\newblock {\em arXiv preprint arXiv:1807.02811}.

\bibitem[Gardner et~al., 2014]{gardner14}
Gardner, J., Kusner, M., Xu, Z., Weinberger, K., and Cunningham, J. (2014).
\newblock Bayesian optimization with inequality constraints.
\newblock In Xing, E.~P. and Jebara, T., editors, {\em Proceedings of the 31st
  International Conference on Machine Learning}, volume~32 of {\em Proceedings
  of Machine Learning Research}, pages 937--945, Bejing, China. PMLR.

\bibitem[Gonzalez et~al., 2016]{gonzalez2016glasses}
Gonzalez, J., Osborne, M., and Lawrence, N. (2016).
\newblock {GLASSES}: {R}elieving the myopia of {B}ayesian optimisation.
\newblock In Gretton, A. and Robert, C.~C., editors, {\em Proceedings of the
  19th International Conference on Artificial Intelligence and Statistics},
  volume~51 of {\em Proceedings of Machine Learning Research}, pages 790--799,
  Cadiz, Spain. PMLR.

\bibitem[Griffiths and Hern{\'a}ndez-Lobato, 2020]{griffiths2020constrained}
Griffiths, R.-R. and Hern{\'a}ndez-Lobato, J.~M. (2020).
\newblock Constrained {B}ayesian optimization for automatic chemical design
  using variational autoencoders.
\newblock {\em Chemical science}, 11(2):577--586.

\bibitem[Hoffman et~al., 2010]{hoffman2010online}
Hoffman, M., Bach, F., and Blei, D. (2010).
\newblock Online learning for latent {D}irichlet allocation.
\newblock In Lafferty, J., Williams, C., Shawe-Taylor, J., Zemel, R., and
  Culotta, A., editors, {\em Advances in Neural Information Processing
  Systems}, volume~23, pages 856--864. Curran Associates, Inc.

\bibitem[Jiang et~al., 2020a]{jiang2020binoculars}
Jiang, S., Chai, H., Gonzalez, J., and Garnett, R. (2020a).
\newblock {BINOCULARS} for efficient, nonmyopic sequential experimental design.
\newblock In III, H.~D. and Singh, A., editors, {\em Proceedings of the 37th
  International Conference on Machine Learning}, volume 119 of {\em Proceedings
  of Machine Learning Research}, pages 4794--4803. PMLR.

\bibitem[Jiang et~al., 2020b]{jiang2020efficient}
Jiang, S., Jiang, D., Balandat, M., Karrer, B., Gardner, J., and Garnett, R.
  (2020b).
\newblock Efficient nonmyopic bayesian optimization via one-shot multi-step
  trees.
\newblock In Larochelle, H., Ranzato, M., Hadsell, R., Balcan, M.~F., and Lin,
  H., editors, {\em Advances in Neural Information Processing Systems},
  volume~33, pages 18039--18049. Curran Associates, Inc.

\bibitem[Kamath, 2015]{kamath2015bounds}
Kamath, G. (2015).
\newblock Bounds on the expectation of the maximum of samples from a
  {G}aussian.
\newblock {\em URL http://www. gautamkamath. com/writings/gaussian max. pdf}.

\bibitem[Kandasamy et~al., 2016]{kandasamy2016gaussian}
Kandasamy, K., Dasarathy, G., Oliva, J.~B., Schneider, J., and Poczos, B.
  (2016).
\newblock Gaussian process bandit optimisation with multi-fidelity evaluations.
\newblock In Lee, D., Sugiyama, M., Luxburg, U., Guyon, I., and Garnett, R.,
  editors, {\em Advances in Neural Information Processing Systems}, volume~29,
  pages 992--1000. Curran Associates, Inc.

\bibitem[Kandasamy et~al., 2017]{kandasamy2017multi}
Kandasamy, K., Dasarathy, G., Schneider, J., and P{\'o}czos, B. (2017).
\newblock Multi-fidelity {B}ayesian optimisation with continuous
  approximations.
\newblock In Precup, D. and Teh, Y.~W., editors, {\em Proceedings of the 34th
  International Conference on Machine Learning}, volume~70 of {\em Proceedings
  of Machine Learning Research}, pages 1799--1808. PMLR.

\bibitem[Kingma and Welling, 2013]{kingma2013auto}
Kingma, D.~P. and Welling, M. (2013).
\newblock Auto-encoding variational bayes.
\newblock {\em arXiv preprint arXiv:1312.6114}.

\bibitem[Kushner, 1964]{kushner1964new}
Kushner, H.~J. (1964).
\newblock A new method of locating the maximum point of an arbitrary multipeak
  curve in the presence of noise.
\newblock {\em Journal of Basic Engineering}, 86(1):97--106.

\bibitem[Lam et~al., 2016]{lam2016bayesian}
Lam, R., Willcox, K., and Wolpert, D.~H. (2016).
\newblock Bayesian optimization with a finite budget: {A}n approximate dynamic
  programming approach.
\newblock In Lee, D., Sugiyama, M., Luxburg, U., Guyon, I., and Garnett, R.,
  editors, {\em Advances in Neural Information Processing Systems}, volume~29,
  pages 883--891. Curran Associates, Inc.

\bibitem[Lee et~al., 2020a]{lee2020efficient}
Lee, E., Eriksson, D., Bindel, D., Cheng, B., and Mccourt, M. (2020a).
\newblock Efficient rollout strategies for {B}ayesian optimization.
\newblock In Peters, J. and Sontag, D., editors, {\em Proceedings of the 36th
  Conference on Uncertainty in Artificial Intelligence}, volume 124 of {\em
  Proceedings of Machine Learning Research}, pages 260--269. PMLR.

\bibitem[Lee et~al., 2021]{lee2021}
Lee, E.~H., Eriksson, D., Perrone, V., and Seeger, M. (2021).
\newblock A nonmyopic approach to cost-constrained {B}ayesian optimization.
\newblock In {\em Conference on Uncertainty in Artificial Intelligence}.

\bibitem[Lee et~al., 2020b]{lee2020cost}
Lee, E.~H., Perrone, V., Archambeau, C., and Seeger, M. (2020b).
\newblock Cost-aware bayesian optimization.
\newblock {\em arXiv preprint arXiv:2003.10870}.

\bibitem[Pedregosa et~al., 2011]{scikit-learn}
Pedregosa, F., Varoquaux, G., Gramfort, A., Michel, V., Thirion, B., Grisel,
  O., Blondel, M., Prettenhofer, P., Weiss, R., Dubourg, V., Vanderplas, J.,
  Passos, A., Cournapeau, D., Brucher, M., Perrot, M., and {{\'E}}douard
  Duchesnay (2011).
\newblock Scikit-learn: {M}achine learning in {P}ython.
\newblock {\em Journal of Machine Learning Research}, 12(85):2825--2830.

\bibitem[Poloczek et~al., 2017]{poloczek2017multi}
Poloczek, M., Wang, J., and Frazier, P.~I. (2017).
\newblock Multi-information source optimization.
\newblock In Guyon, I., Luxburg, U.~V., Bengio, S., Wallach, H., Fergus, R.,
  Vishwanathan, S., and Garnett, R., editors, {\em Advances in Neural
  Information Processing Systems}, pages 4289--4299. Curran Associates, Inc.

\bibitem[Schonlau et~al., 1998]{schonlau1998global}
Schonlau, M., Welch, W.~J., and Jones, D.~R. (1998).
\newblock Global versus local search in constrained optimization of computer
  models.
\newblock {\em Lecture Notes-Monograph Series}, 34:11--25.

\bibitem[Shahriari et~al., 2016]{shahriari2015survey}
Shahriari, B., Swersky, K., Wang, Z., Adams, R.~P., and de~Freitas, N. (2016).
\newblock Taking the human out of the loop: {A} review of {B}ayesian
  optimization.
\newblock {\em Proceedings of the IEEE}, 104(1):148--175.

\bibitem[Snoek et~al., 2012]{Snoek2012ML}
Snoek, J., Larochelle, H., and Adams, R.~P. (2012).
\newblock Practical {B}ayesian optimization of machine learning algorithms.
\newblock In Pereira, F., Burges, C. J.~C., Bottou, L., and Weinberger, K.~Q.,
  editors, {\em Advances in Neural Information Processing Systems}, volume~25,
  pages 2951--2959. Curran Associates, Inc.

\bibitem[Song et~al., 2019]{song2019general}
Song, J., Chen, Y., and Yue, Y. (2019).
\newblock A general framework for multi-fidelity {B}ayesian optimization with
  {G}aussian processes.
\newblock In Chaudhuri, K. and Sugiyama, M., editors, {\em Proceedings of the
  Twenty-Second International Conference on Artificial Intelligence and
  Statistics}, volume~89 of {\em Proceedings of Machine Learning Research},
  pages 3158--3167. PMLR.

\bibitem[Sutton and Barto, 2018]{sutton2018reinforcement}
Sutton, R.~S. and Barto, A.~G. (2018).
\newblock {\em Reinforcement learning: An introduction}.
\newblock MIT press.

\bibitem[Swersky et~al., 2013]{swersky2013multi}
Swersky, K., Snoek, J., and Adams, R.~P. (2013).
\newblock Multi-task {B}ayesian optimization.
\newblock In Burges, C. J.~C., Bottou, L., Welling, M., Ghahramani, Z., and
  Weinberger, K.~Q., editors, {\em Advances in Neural Information Processing
  Systems}, volume~26, pages 2004--2012. Curran Associates, Inc.

\bibitem[{The HPOlib authors}, 2014]{hpolib}
{The HPOlib authors} (2014).
\newblock {\em {HPOlib}}.
\newblock \textit{Available at:} \url{ https://github.com/automl/HPOlib}.

\bibitem[{The HPOlib1.5 authors}, 2017]{hpolib15}
{The HPOlib1.5 authors} (2017).
\newblock {\em {HPOlib1.5}}.
\newblock \textit{Available at:} \url{ https://github.com/automl/HPOlib1.5}.

\bibitem[Wang and Jegelka, 2017]{wang2017max}
Wang, Z. and Jegelka, S. (2017).
\newblock Max-value entropy search for efficient {B}ayesian optimization.
\newblock In Precup, D. and Teh, Y.~W., editors, {\em Proceedings of the 34th
  International Conference on Machine Learning}, volume~70 of {\em Proceedings
  of Machine Learning Research}, pages 3627--3635. PMLR.

\bibitem[Williamson and Shmoys, 2011]{williamson2011design}
Williamson, D.~P. and Shmoys, D.~B. (2011).
\newblock {\em The design of approximation algorithms}.
\newblock Cambridge university press.

\bibitem[Wilson et~al., 2018]{wilson2018maximizing}
Wilson, J., Hutter, F., and Deisenroth, M. (2018).
\newblock Maximizing acquisition functions for {B}ayesian optimization.
\newblock In Bengio, S., Wallach, H., Larochelle, H., Grauman, K.,
  Cesa-Bianchi, N., and Garnett, R., editors, {\em Advances in Neural
  Information Processing Systems}, volume~31, pages 9884--9895. Curran
  Associates, Inc.

\bibitem[Wu and Frazier, 2019]{wu2019practical}
Wu, J. and Frazier, P. (2019).
\newblock Practical two-step lookahead {B}ayesian optimization.
\newblock In Wallach, H., Larochelle, H., Beygelzimer, A., Fox, E., and
  Garnett, R., editors, {\em Advances in Neural Information Processing
  Systems}, volume~32, pages 9813--9823. Curran Associates, Inc.

\bibitem[Wu et~al., 2020]{wu2020practical}
Wu, J., Toscano-Palmerin, S., Frazier, P.~I., and Wilson, A.~G. (2020).
\newblock Practical multi-fidelity {B}ayesian optimization for hyperparameter
  tuning.
\newblock In Adams, R.~P. and Gogate, V., editors, {\em Proceedings of The 35th
  Uncertainty in Artificial Intelligence Conference}, volume 115 of {\em
  Proceedings of Machine Learning Research}, pages 788--798. PMLR.

\bibitem[Xia et~al., 2016]{xia2016budgeted}
Xia, Y., Ding, W., Zhang, X.-D., Yu, N., and Qin, T. (2016).
\newblock Budgeted bandit problems with continuous random costs.
\newblock In Holmes, G. and Liu, T.-Y., editors, {\em Asian Conference on
  Machine Learning}, volume~45 of {\em Proceedings of Machine Learning
  Research}, pages 317--332, Hong Kong. PMLR.

\bibitem[Xia et~al., 2015]{xia2015thompson}
Xia, Y., Li, H., Qin, T., Yu, N., and Liu, T.-Y. (2015).
\newblock Thompson sampling for budgeted multi-armed bandits.
\newblock In {\em Proceedings of the 24th International Joint Conference on
  Artificial Intelligence}, pages 3960--3966.

\bibitem[Yue and Kontar, 2020]{yue2020non}
Yue, X. and Kontar, R.~A. (2020).
\newblock Why non-myopic {B}ayesian optimization is promising and how far
  should we look-ahead? {A} study via rollout.
\newblock In Chiappa, S. and Calandra, R., editors, {\em Proceedings of the
  Twenty Third International Conference on Artificial Intelligence and
  Statistics}, volume 108 of {\em Proceedings of Machine Learning Research},
  pages 2808--2818. PMLR.

\bibitem[Zhang et~al., 2020]{zhang2020bayesian}
Zhang, Y., Apley, D.~W., and Chen, W. (2020).
\newblock Bayesian optimization for materials design with mixed quantitative
  and qualitative variables.
\newblock {\em Scientific reports}, 10(1):1--13.

\end{thebibliography}

\setcounter{theorem}{0}
\setcounter{proposition}{0}
\setcounter{lemma}{0}

\clearpage
\appendix
\section{Proof of Theorem 1}
\label{supp:thm1}
\begin{theorem}
\label{prop:counterexample}
The approximation ratios provided by the EI and EI-PUC policies are unbounded. That is, for any $\rho > 0$ and each policy $\pi\in\{\mathrm{EI}, \mathrm{EI\textnormal{-}PUC}\}$, there exists a Bayesian optimization problem instance (a  prior probability distribution over objective and cost functions, a budget, and a set of initial observations $\D_0$) where 
\[ V^*(\D_0) > \rho \, V^{\pi}(\D_0). \]
\end{theorem}
\begin{proof}
We prove the result in two parts, first focusing on EI-PUC, and then on EI.

\textbf{Part 1: EI-PUC.}
To show the result for EI-PUC, we construct a problem instance with a discrete finite domain and no observation noise. Let $\epsilon$ and $\delta$ be two strictly positive real numbers and let $K = \lceil (1+\delta)/\epsilon \rceil$. Suppose the domain is $\mathbb X = \{0, 1, \ldots, K+1\}$ and that the overall budget of the problem is $1+\delta$. The prior on $f(x)$ is independent across $x \in \mathbb X$, all points have zero mean under the prior, and the costs $c(x)$ are all known. We assume that the point $x=0$ has already been measured and has value $f(0)=0$.
The next $K$ points are ``low-variance, low-cost'' and the final point is ``high-variance, high-cost.''
\begin{itemize}
    \item For low-variance, low-cost points, $x \in \{1,\ldots,K\}$, $f(x)$ is normally distributed with mean 0 and variance $\epsilon^2$ and $c(x) = \epsilon$. Thus, $\mathrm{EI}(x) / c(x)$ for these arms is $\E[(f(x) - 0)^+] / \epsilon = \E[\epsilon Z^+] / \epsilon = \E[Z^+]$, where $Z$ is a standard normal random variable.
    \item For the high-variance, high-cost point, $x = K+1$, $f(x)$ is normally distributed with mean 0 and variance 1, and $c(x) = 1+\delta$. Thus, $\mathrm{EI}(x) / c(x)$ for these arms is $\E[(f(x) - 0)^+] / (1+\delta) = \E[Z^+] / (1+\delta)$.
\end{itemize}

One feasible policy for the problem is to ``measure the high-variance point once.'' Under this policy, the budget is exhausted after the first measurement and the overall value is $\E[Z^+]$. 

Let us consider the EI-PUC policy. By the calculation above, on the first evaluation, it measures a low-variance arm. Then, after that, the remaining budget is strictly less than $1+\delta$ and it can only measure low-variance, low-cost points. It measures $N = \lfloor (1+\delta)/\epsilon \rfloor$ such points in total (since there is no observation noise, EI of a measured point is zero and a new point is measured each time). We show, in a calculation below, that the value of this policy goes to $0$ as $\epsilon \to 0$ while we hold $\delta$ fixed.
That is, we show that 
\begin{equation}
\lim_{\epsilon \to 0} V^{\mathrm{EI\textnormal{-}PUC}}(\D_0) = 0.
\label{toshow}
\end{equation}

There are no policies other than the two just described, so the ``measure the high-variance point once'' policy becomes optimal for all $\epsilon$ sufficiently close to $0$.
Thus, $\lim_{\epsilon\to0} V^*(\D_0) = \E[Z^+] > 0$. Thus, given any finite strictly positive $\rho$, there is a $\epsilon$ small enough
whose corresponding problem instance satisfies 
\begin{equation*}
V^*(\D_0) > \rho \, V^{\mathrm{EI\textnormal{-}PUC}}(\D_0).
\end{equation*}

To complete the proof for EI-PUC, we now show (\ref{toshow}).
Under the EI-PUC policy, the values $\{y_n\}_{n=1}^N$ from the sequence of measured points (after the initial point $x=0$) are independent and identically distributed normal random variables, with mean 0 and variance $\epsilon^2$. The expected value under this policy is then $\E[M]$, where $M = \max\left\{0, y_1, \ldots, y_N \right\}$. 

Let $t$ be an arbitrary positive real number. Using Jensen's inequality, we obtain
\begin{align*}
\exp \bigl(\E[tM] \bigr) &\leq \E\bigl[\exp(tM)\bigr]\\
&=\E\bigl[\max\left\{1, \exp(ty_1), \dots, \exp(ty_N) \right\}\bigr]\\
&\leq\E\left[1 + \sum_{n=1}^N\exp(ty_n)\right]\\
&=1 + N\exp(t^2\epsilon^2/2)\\
&\leq (N+1)\exp(t^2\epsilon^2/2),
\end{align*}
where we note in the last step that 
$\exp(t^2\epsilon^2/2) \ge 1$.

Thus, $\E[M] \leq \ln(N + 1)/t + (t\epsilon^2)/2$, and taking $t=\sqrt{2\ln(N+1)}/\epsilon$ we obtain 
\begin{equation*}
   \E[M]\leq \epsilon\sqrt{2\ln(N+1)} \le \epsilon \sqrt{2\ln\left(\frac{1+\delta + \epsilon}{\epsilon}\right)}.
\end{equation*}

Using L'H\^{o}pital's rule, we see that the right hand side of the above inequality converges to 0 as $\epsilon \to 0$ and, therefore, so does $\E[M]$.

\textbf{Part 2: EI.}
We now show the result for EI. 
We consider almost the same setup as in the example above, with the only difference being that the low-cost points now have variance $(1-\epsilon)^2$.

This time $\mathrm{EI}(x) = (1-\epsilon) \, \E[Z^+]$ for the low-cost points and $\mathrm{EI}(x) = \E[Z^+]$ for the high-cost point. Hence, the EI policy evaluates the high-cost point and exhausts the budget after that single evaluation, thus implying that $V^\mathrm{EI}(\D_0) = \E[Z^+]$.

Now consider the policy that measures low-cost points only. The expected value under this policy is $\E[M]$, where where $M = \max\left\{0, y_1, \ldots, y_N \right\}$, and $y_1,\ldots, y_N$ are independent identically distributed normal random variables with mean $0$ and variance $(1-\epsilon)^2$.

We have
\begin{equation*}
    \E[M] \geq \E\left[\max_{n=1,\ldots, N}y_n\right]
    \geq (1-\epsilon)\sqrt{a\ln N}
     \ge (1-\epsilon) \sqrt{a\ln\left(\frac{1+\delta}{\epsilon}-1\right)},
\end{equation*}
where $a=(\pi\ln 2)^{-1}$ and the second inequality follows by Theorem 1 in  \cite{kamath2015bounds}. The expression on the right-hand side of the above inequality goes to $\infty$ as $\epsilon\to 0$. Since $V^*(\D_0)\geq \E[M]$, it follows that, for any $\rho > 0$, there exists $\epsilon$ small enough whose corresponding problem instance satisfies
\begin{equation*}
V^*(\D_0) > \rho \, V^\mathrm{EI}(\D_0),
\end{equation*}
which completes the proof.
\end{proof}

\section{Proof of Proposition 1}
\label{supp:prop1}
\begin{proposition}
\label{prop:bcei_app}
Suppose that $f$ and $\ln c$ follow independent Gaussian process prior distributions and that $\D$ is an arbitrary set of observations. Define $\mu^f_\D(x) = \mathbb{E}[ f(x) \,|\, \D]$, $\mu^{\ln c}_\D(x) = \mathbb{E}[ \ln c(x) \,|\, \D]$, $\sigma^f_\D(x) = \textnormal{Var}[ f(x) \,|\, \D]^{1/2}$, and $\sigma^{\ln c}_\D(x) = \textnormal{Var}[ \ln c(x) \,|\, \D]^{1/2}$. Then,
\begin{equation*}
    Q_1(x\mid \D) = \begin{cases} \textnormal{EI}^f(x\mid\D) \, \Phi\left(\frac{\ln\left((B - s(\D)\right) - \mu_\D^{\ln c}(x)}{\sigma_\D^{\ln c}(x)} \right), \textnormal{ if } s(\D) < B,\\
0, \textnormal{ otherwise,}
\end{cases}
\end{equation*}
where
\begin{equation*}
\textnormal{EI}^f(x\mid\D) = \Delta(x)\Phi\left(\frac{\Delta_\D(x)}{\sigma_\D^f(x)}\right) + \sigma^f(x)\varphi\left(\frac{\Delta_\D(x)}{\sigma_\D^f(x)}\right)    
\end{equation*}
is the classical expected improvement computed with respect to $f$; $\varphi$ and $\Phi$ are the standard normal pdf and cdf, respectively; and $\Delta_\D(x) = \mu^f(x) -  u(\mathcal{S})$.
\end{proposition}
\begin{proof}
Recall that
\begin{equation*}
     Q_1(x \mid \D) = \E\left[(f(x) - u(\D))^+\; \indicate{s(\D) + c(x) \leq B}\right].
\end{equation*}
But, since $f$ and $\ln c$ are assumed to be independent, 
\begin{align*}
     Q_1(x \mid \D) &= \E\left[(f(x) - u(\D))^+\right] \E\left[\indicate{s(\D) + c(x) \leq B}\right]\\
     &= \mathrm{EI}^f(x\mid\D) \mathbb{P}\left(s(\D) + c(x) \leq B\right).
\end{align*}

Now note that, if $s(\D) \geq B$, then $\mathbb{P}\left(s(\D) + c(x) \leq B\right)=0$. If, on the other hand, $s(\D) < B$, then
\begin{align*}
    \mathbb{P}\left(s(\D) + c(x) \leq B\right) & = \mathbb{P}\left(\ln c(x) \leq \ln\left(B - s(\D)\right)\right)\\
    &= \Phi\left(\frac{\ln\left((B - s(\D)\right) - \mu_\D^{\ln c}(x)}{\sigma_\D^{\ln c}(x)} \right),
\end{align*}
which concludes the proof.
\end{proof}

\section{Closed-Form Expressions of EI-PUC and EI-PUC-CC}
\label{supp:closed_form}
Recall that the EI-PUC and EI-PUC-CC acquisition functions are defined by
\begin{equation*}
    \mathrm{EI\textnormal{-}PUC}(x\mid\D) = \E_\D\left[\frac{\{f(x)-f_n^*\}^+}{c(x)}\right],   
\end{equation*}
and
\begin{equation*}
 \mathrm{EI\textnormal{-}PUC\textnormal{-}CC}(x\mid\D) = \E_\D\left[\frac{\{f(x)-f_n^*\}^+}{c(x)^\nu}\right],   
\end{equation*}
respectively, where $\nu$ is the ratio between the remaining budget and the initial budget.

Here we show that, under the same conditions of Proposition 1, EI-PUC and EI-PUC-CC have closed form expressions. This is summarized in Proposition 2 below. Note that this result gives the closed-form expression of EI-PUC as a special case by taking $\nu=1$.

\begin{proposition}
\label{prop:eipuc}
Let $\nu$ be an arbitrary positive real number and suppose that the conditions in Proposition 1 are satisfied. Then,
\begin{equation*}
 \E_\D\left[\frac{\{f(x)-f_n^*\}^+}{c(x)^\nu}\right]   = \mathrm{EI}^{f}(x\mid \D)\exp(-\nu \mu_{\D}^{\ln c}(x) + \nu^2\sigma_\D^{\ln c}(x)^2/2),
\end{equation*}
\end{proposition}
\begin{proof}
Since $f$ and $\ln c$ are assumed to be independent,
\begin{align*}
  \E_\D\left[\frac{\{f(x)-f_n^*\}^+}{c(x)^\nu}\right]  &=  \E_\D\left[\{f(x)-f_n^*\}^+\right]\E_\D\left[1/c(x)^\nu\right]\\
  &= \mathrm{EI}^{f}(x\mid\D) \, \E_\D\left[\exp(-\nu \ln c(x))\right].
\end{align*}
The well-known formula for the moment generating function of normal random variable gives
\begin{equation*}
  \E_\D\bigl[\exp(-\nu \ln c(x))\bigr] = \exp\bigl(-\nu \mu_{\D}^{\ln c}(x) + \nu^2\sigma_\D^{\ln c}(x)^2/2 \bigr), 
\end{equation*}
which completes the proof.
\end{proof}

\section{Synthetic Test Problems Details}
\label{supp:synthetic}
Below we provide additional details on the synthetic test problems used in the numerical experiments.
\subsubsection*{Dropwave:}
\begin{itemize}
    \item $f(x)=\left(1 + 12\sqrt{x_1^2 + x_2^2}\right)/\left(0.5(x_1^2 + x_2^2) + 2\right)$.
    \item $\X = [-5.12, 5.12]^2$.
    \item $\alpha\in[0.75, 1.5]$, $\beta\in[2\pi/5.12, 6\pi/5.12]$, $\gamma\in[0, 2\pi]$.
\end{itemize}

\subsubsection*{Alpine1:}
\begin{itemize}
\item $f(x)= \sum_{i=1}^d\left|x_i \sin(x_i) + 0.1 x_i\right|$.
\item $\X = [-10, 10]^d$.
\item $\alpha\in[0.75, 1.5]$, $\beta\in[2\pi, 6\pi]$, $\gamma\in[0, 2\pi]$.
\item In our experiment, we set $d=3$.
\end{itemize}

\subsubsection*{Ackley:}
\begin{itemize}
\item $f(x)= 20 \exp\left(-\frac{1}{5}\sqrt{\frac{1}{d}\sum_{i=1}^D x_i^2}\right) + \exp\left(\frac{1}{D}\sum_{i=1}^D \cos(2\pi x_i)\right) - 20 - \exp(1)$.
\item $\X = [-1, 1]^d$.
\item $\alpha\in[0.75, 1.5]$, $\beta\in[2\pi, 6\pi]$, $\gamma\in[0, 2\pi]$.
\item In our experiment, we set $d=3$.
\end{itemize}

\subsubsection*{Shekel5:}
\begin{itemize}
\item $f(x)= \sum_{j=1}^5\left(\sum_{i=1}^4\left(x_i - C_{ij}\right)^2 + b_j\right)^{-1}$, where $b = (1, 2, 2, 4, 4)/10$, and
    \begin{align*}
        C &= \begin{pmatrix}
4 & 1 & 8 & 6 & 3\\
4 & 1 & 8 & 6 & 7\\
4 & 1 & 8 & 6 & 3\\
4 & 1 & 8 & 6 & 7
\end{pmatrix}.
    \end{align*}
\item $\X = [0, 10]^4$.
\item $\alpha\in[0.75, 1.5]$,$\beta\in[2\pi/4, 3\pi/4]$, $\gamma\in[0, 2\pi]$.
\end{itemize}

\section{AutoML Surrogate Models}
\label{supp:automl}
The LDA and CNN problems use surrogate models to emulate the computationally expensive process of training the corresponding ML model. The objective and cost surrogate models for CNN are obtained directly from the HPOLib library, and were originally constructed by fitting independent random forest regression models to the objective and cost evaluations over a uniform grid of points. These evaluations are obtained by training a 3-layer convolutional neural network on the CIFAR-10 dataset. For the LDA problem, we fit independent GPs to the objective and log-cost evaluations over a uniform grid of size 288, and then use the resulting posterior means as the surrogate models. These surrogate models can be found at \url{https://github.com/RaulAstudillo06/BudgetedBO}.

\section{Acquisition Function Optimization}
\label{supp:acqfopt}
All acquisition functions are optimized as follows. First, we evaluate the acquisition value at $200d$ points over $\X$, and $10d$ of these points are selected using the initialization heuristic described in Appendix F.1 of \cite{balandat2020botorch}. Then, we run L-BFGS-B \citep{byrd1995limited} starting from each point. The point with highest acquisition value of the $10d$ points resulting after running L-BFGS-B is then chosen as the next point to evaluate. 

For the variants of B-MS-EI, the $200d$ initial points are chosen using the warm-start initialization strategy described in Appendix D of \cite{jiang2020efficient}. This strategy uses the solution found when optimizing B-MS-EI during the previous iteration and identifies the branch originating at the root of the tree whose fantasy sample is closest to the value actually observed when evaluating the suggested candidate on the true function. For the other acquisition functions, the $200d$ initial points are chosen using a Sobol sampling design.

\section{Additional Plots, Initial Design, Hyperparameter Estimation, Runtimes, and Licenses}
\label{supp:additional}

In each experiment, an initial stage of evaluations is performed using $2(d+1)$ points chosen according to a quasi-random Sobol sampling design over $\X$. Experiments are replicated either 100 (Ackley and Robot Pushing) or 200 (all other test problems) independent trials and we report the average performance and 95\% confidence intervals. For all algorithms, the objective function and the log-cost are modeled using independent GPs with constant mean and Mat\'ern 5/2 covariance function. The length scales of these GP models are estimated via maximum a posteriori (MAP) estimation with Gamma priors.

The average runtimes of the algorithms in each problem are summarized in  Table~\ref{table:runtimes}. Figure~\ref{fig:experiments_supp} plots a 95\% confidence interval on mean log-regret versus the budget. In contrast with the plots in the main paper, these plots also include the results for the 2-step variants of our algorithm. Figure~\ref{fig:animation_supp} below shows the additional evaluations performed by EI-PUC and B-MS-EI within budget in the problem discussed in Figure 2 of the main paper.

\begin{table}
\caption{Average per-acquisition runtimes (in seconds) of the algorithms in each problem.}
\label{table:runtimes}
\begin{center}
\begin{footnotesize}
\begin{tabular}{lrrrrrrr}
\toprule
& EI & EI-PUC & EI-PUC-CC &  2-B-MS-EI$_\textnormal{p}$ &  2-B-MS-EI &  4-B-MS-EI$_\textnormal{p}$ &  4-B-MS-EI\\
\midrule
Dropwave  &  $4.1$ & $5.5$ & $6.2$ & $24.1$ & $26.7$ & $42.6$ & $87.2$\\
Alpine1    & $7.2$  & $9.9$ & $10.8$ & $36.44$ & $44.1$ & $106.1$ & $236.7$\\
Ackley    & $6.4$ & $7.3$ & $5.7$ & $16.6$ & $26.0$ & $95.8$ & $103.5$\\
Shekel5        & $4.5$ & $5.5$ &  $5.6$ & $18.3$ & $12.8$ & $105.8$ & $165.5$ \\
LDA & $5.4$ & $5.9$ &  $5.5$ & $20.7$ & $29.6$ & $141.9$ & $202.1$ \\
CNN & $11.7$ & $26.9$ &  $19.3$ & $79.3$ & $95.0$ & $286.1$ & $443.2$ \\
Robot & $5.9$ & $4.0$ &  $4.7$ & $32.2$ & $34.3$ & $108.2$ & $193.7$ \\
{Boston} & $1.8$ & $1.6$ &  $1.5$ & $15.7$ & $19.7$ & $65.1$ & $90.5$ \\

\bottomrule
\end{tabular}
\end{footnotesize}
\end{center}
\end{table}

\begin{figure}
  \centering
 \includegraphics[width=0.24\textwidth]{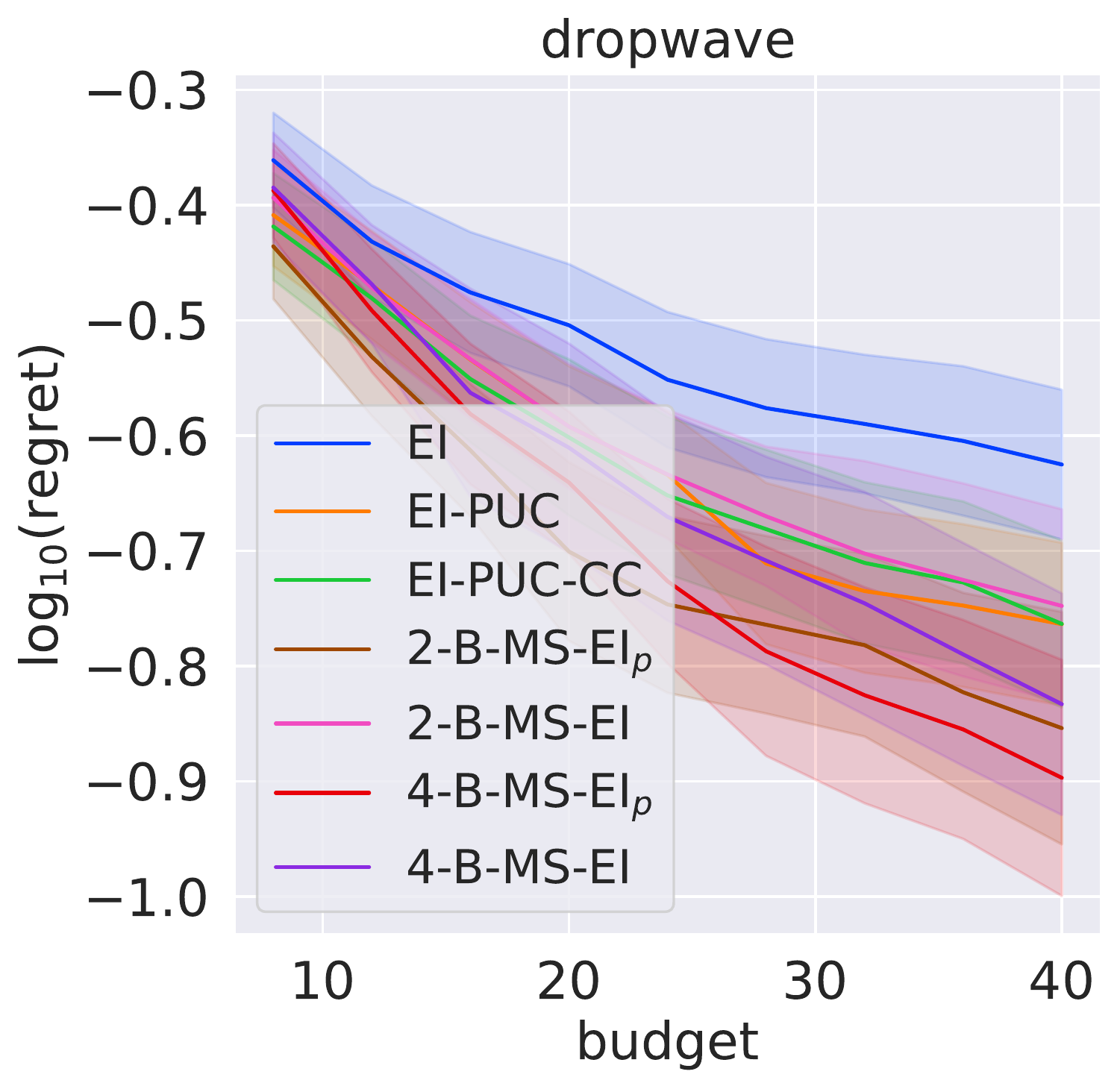}
 \includegraphics[width=0.24\textwidth]{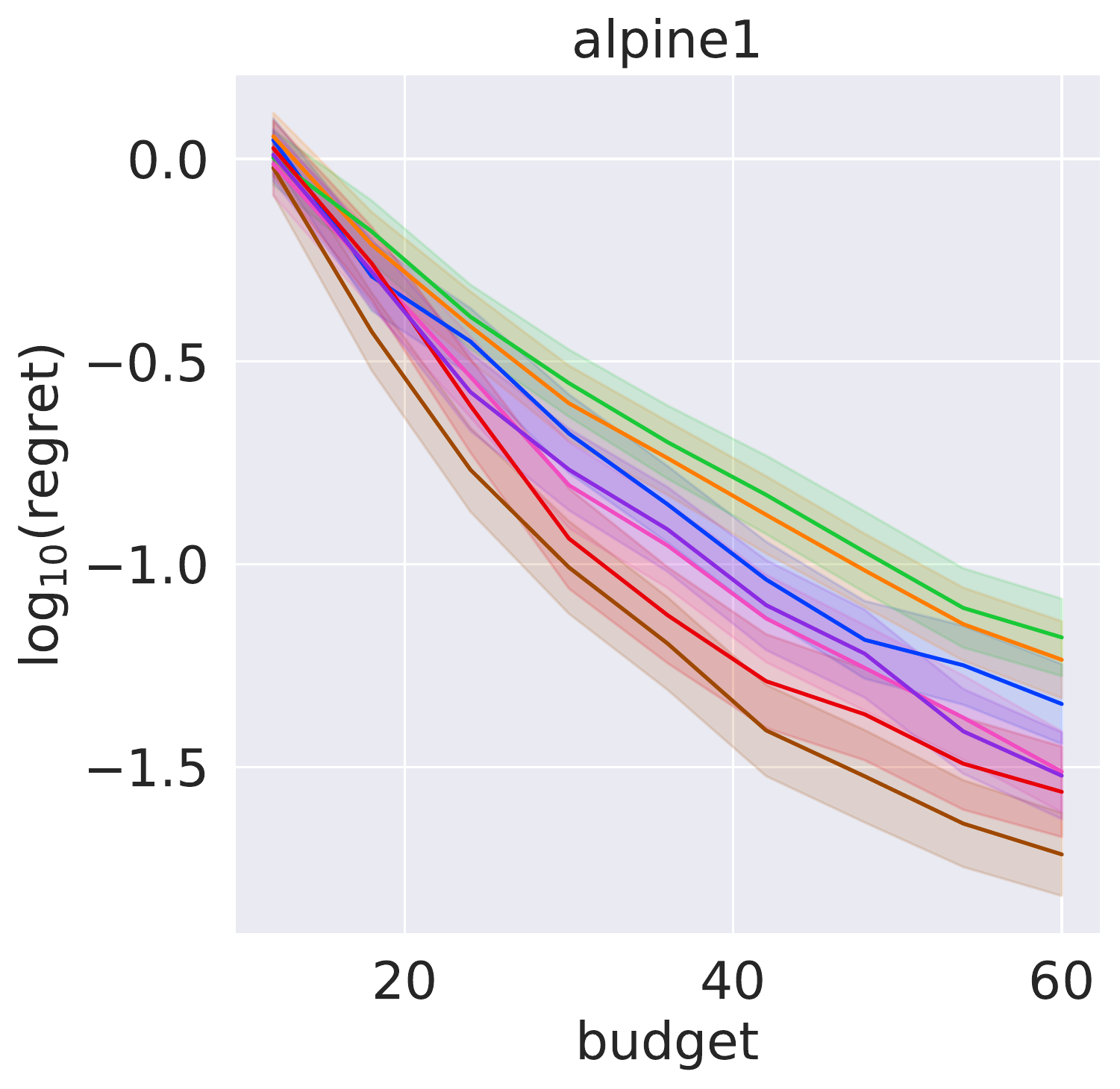}
 \includegraphics[width=0.24\textwidth]{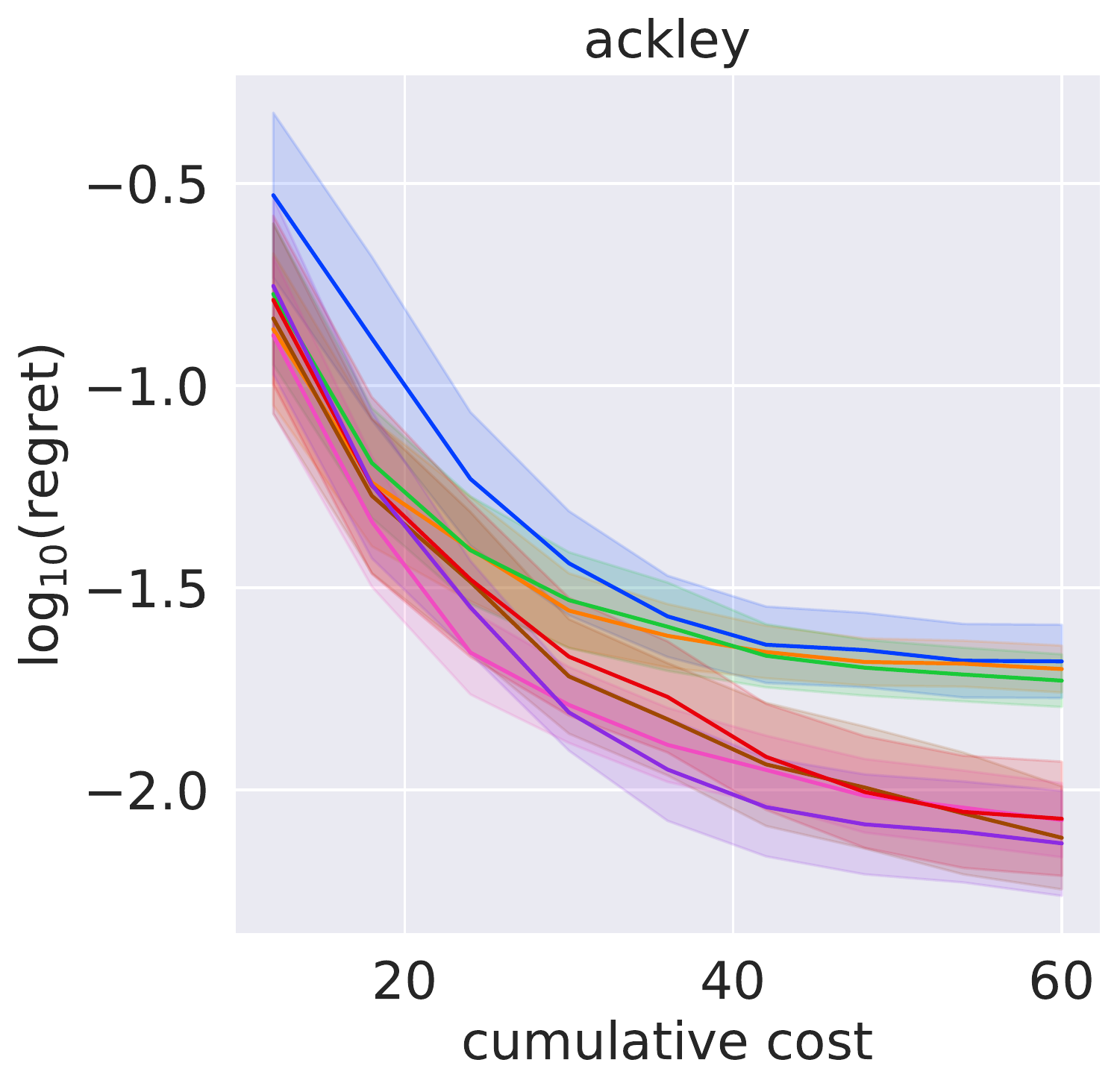}
 \includegraphics[width=0.24\textwidth]{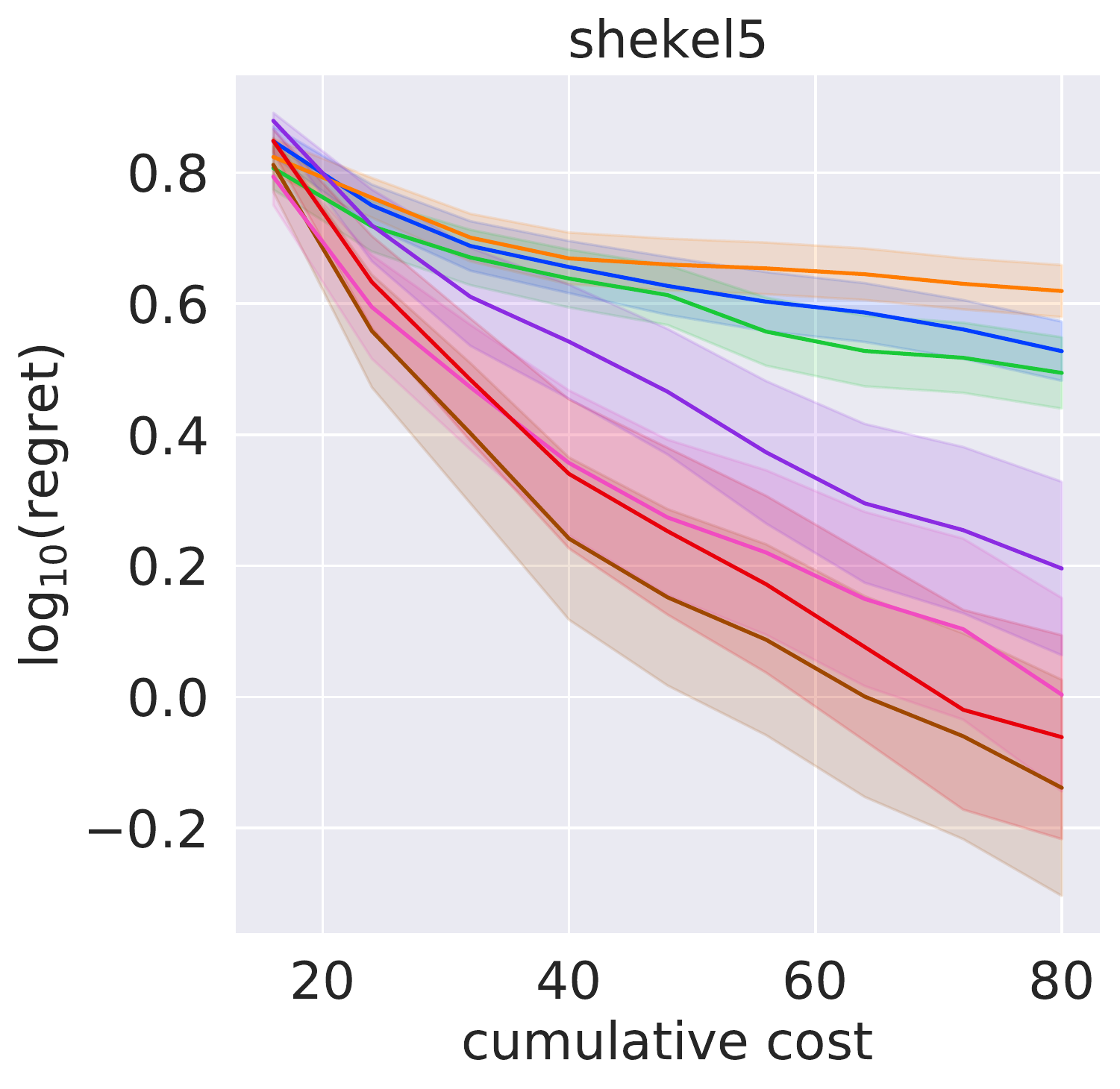}
 \\
 \includegraphics[width=0.235\textwidth]{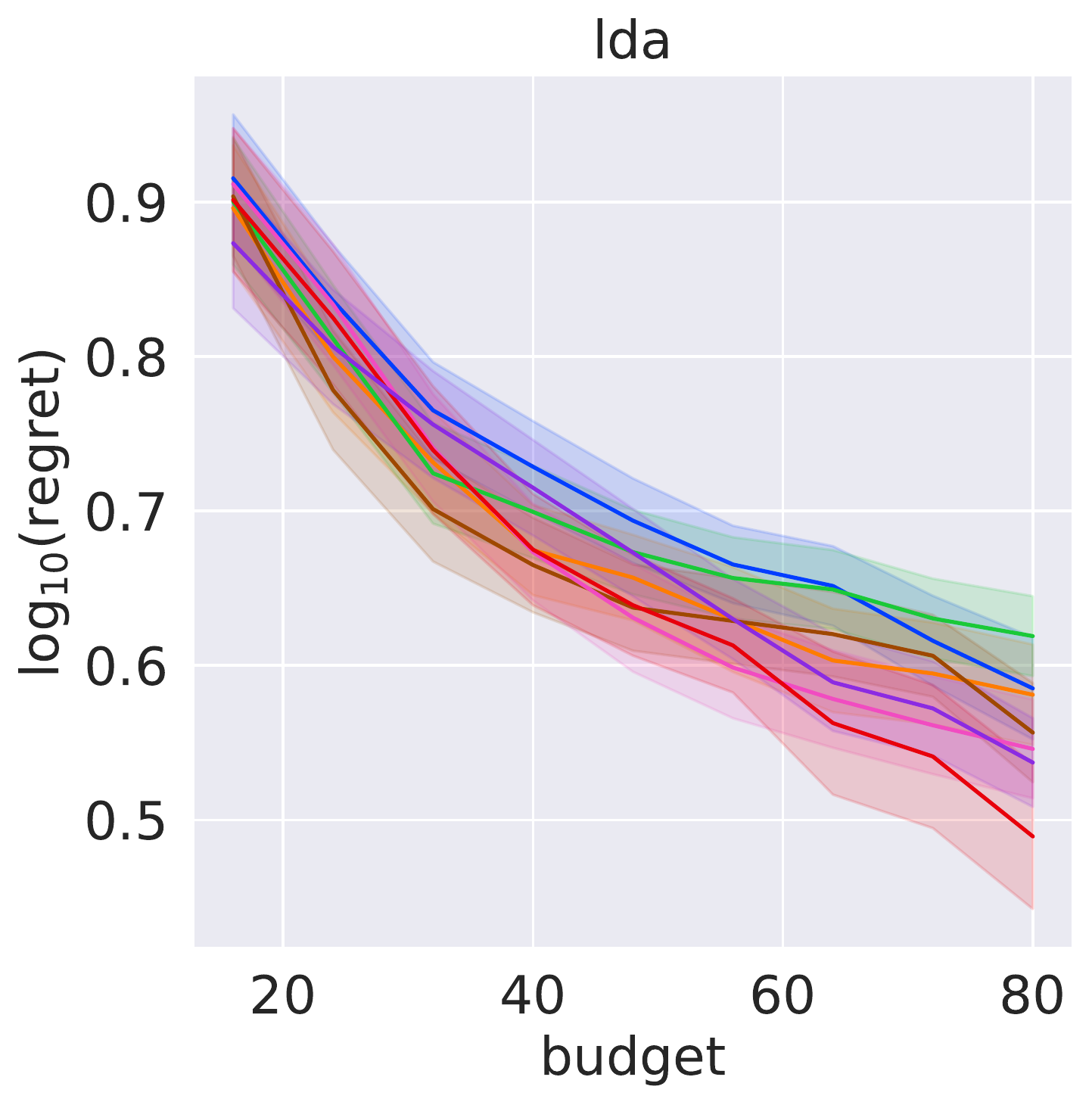}
 \includegraphics[width=0.245\textwidth]{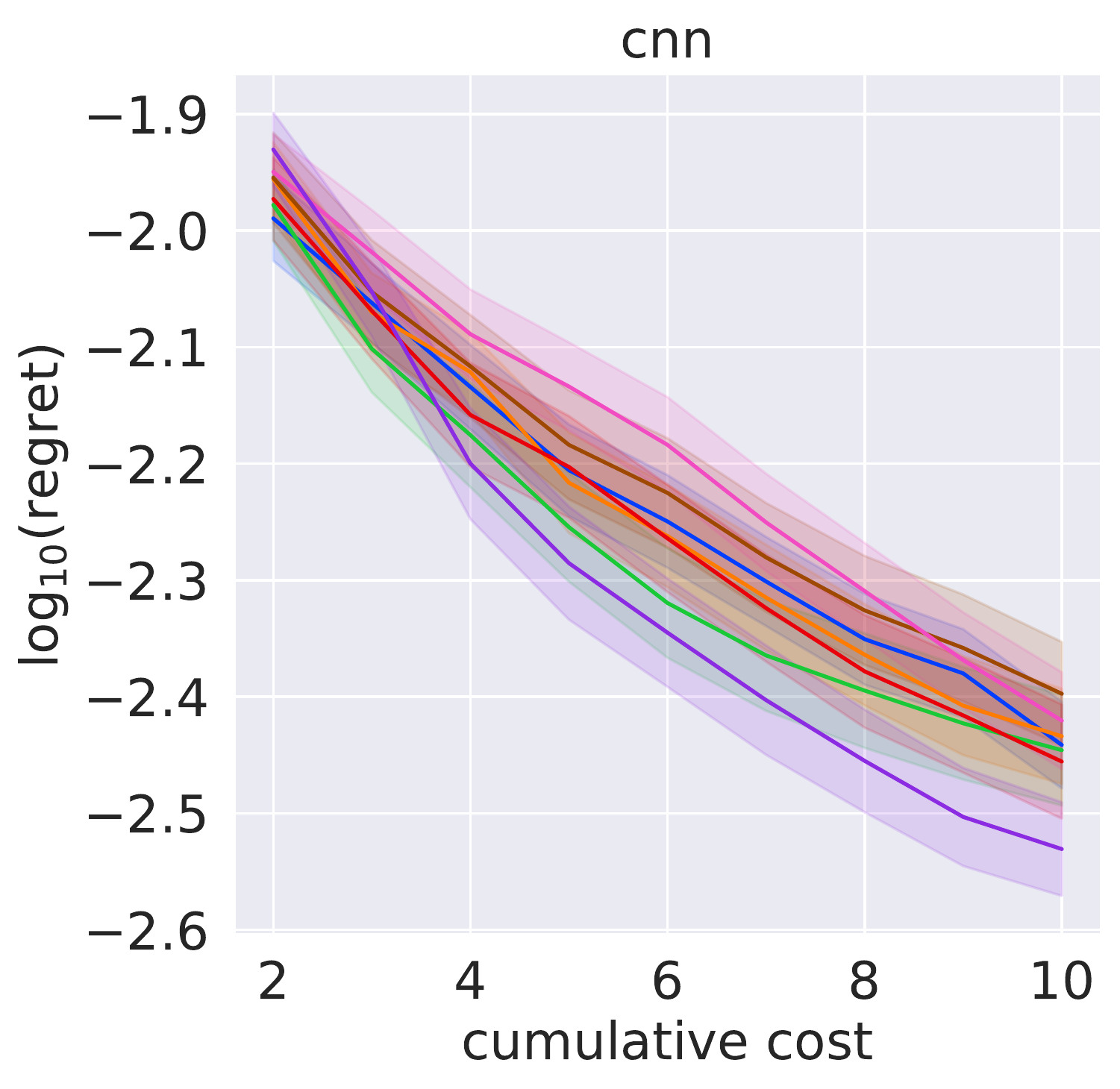}
\includegraphics[width=0.25\textwidth]{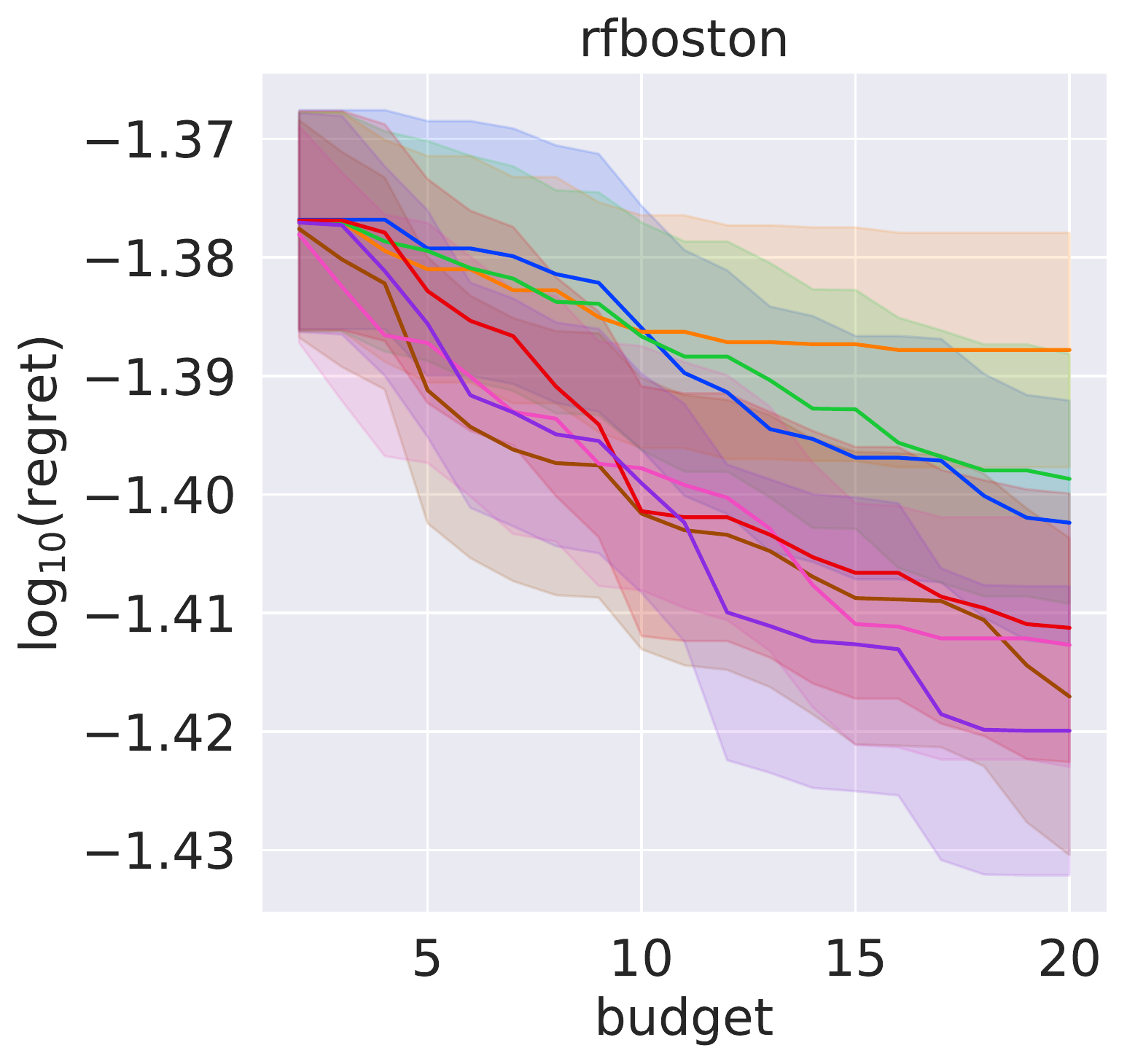}
\includegraphics[width=0.245\textwidth]{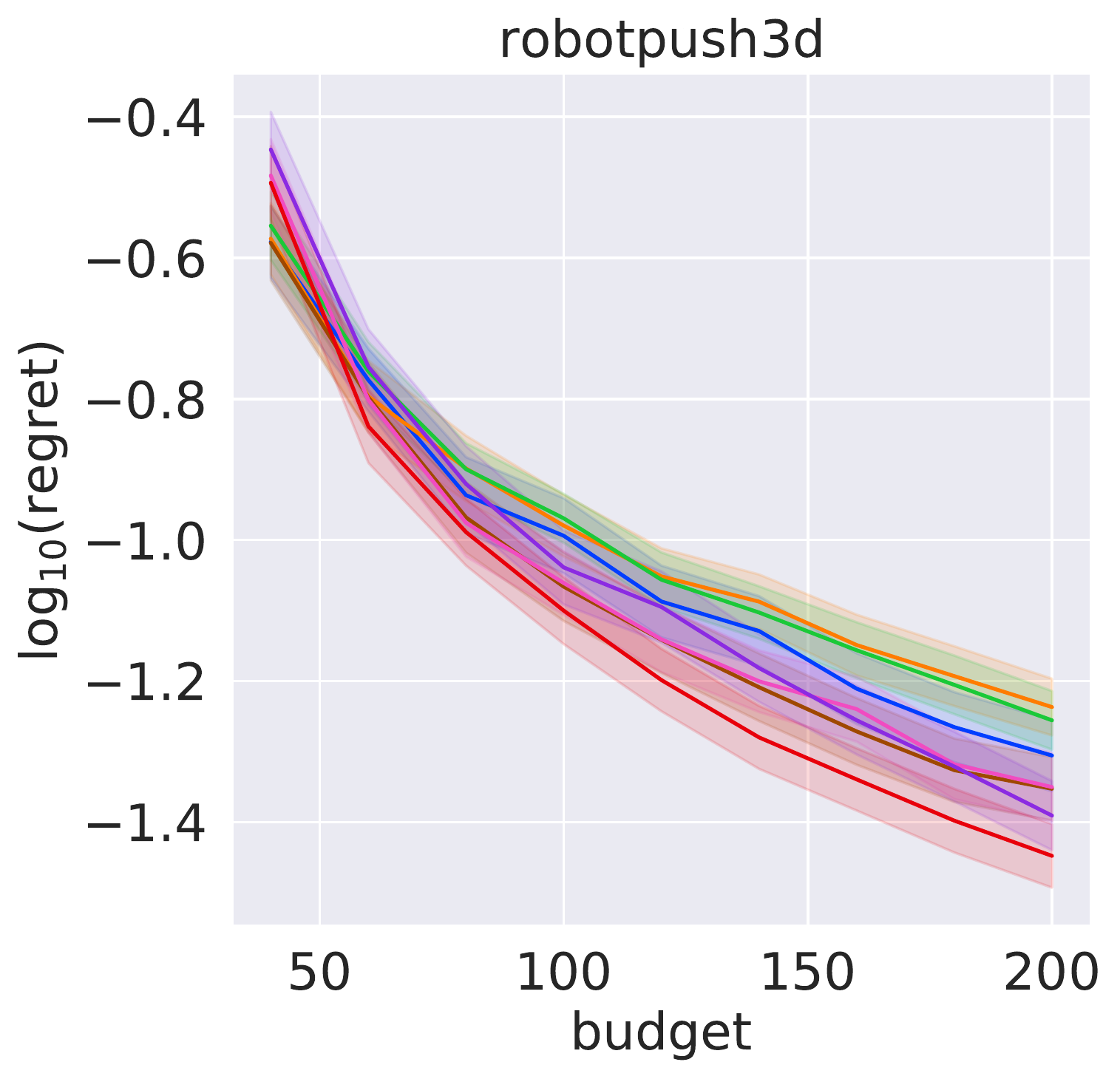}
  \caption{{Log-regret of our non-myopic budget-aware BO methods compared with baseline acquisition functions on a range of problems.} \label{fig:experiments_supp}}
\end{figure}

\begin{figure}
\centering
\subfloat[EI-PUC]{%
\begin{tabular}[b]{c}%
\includegraphics[width=0.33\textwidth]{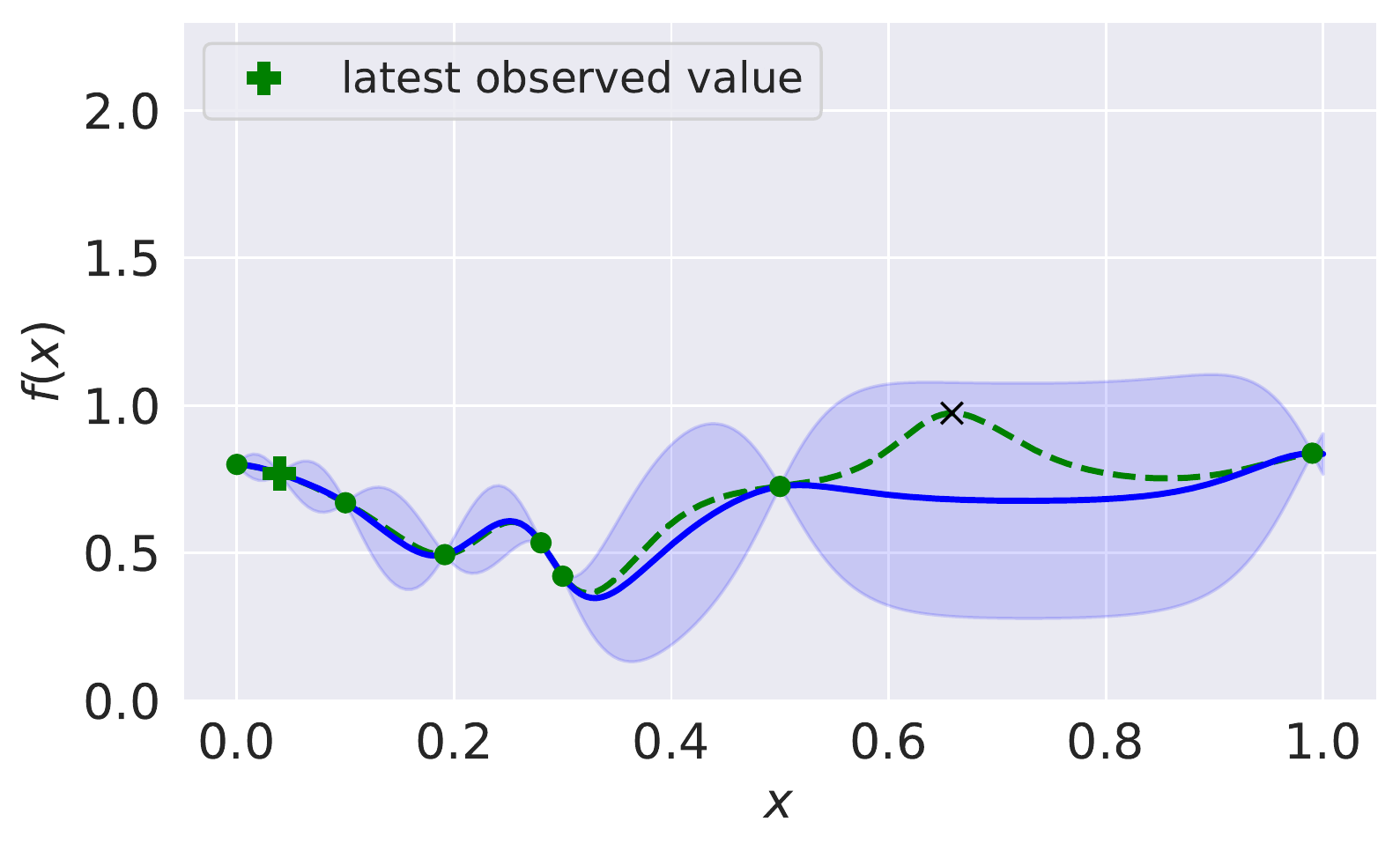}
  \includegraphics[width=0.33\textwidth]{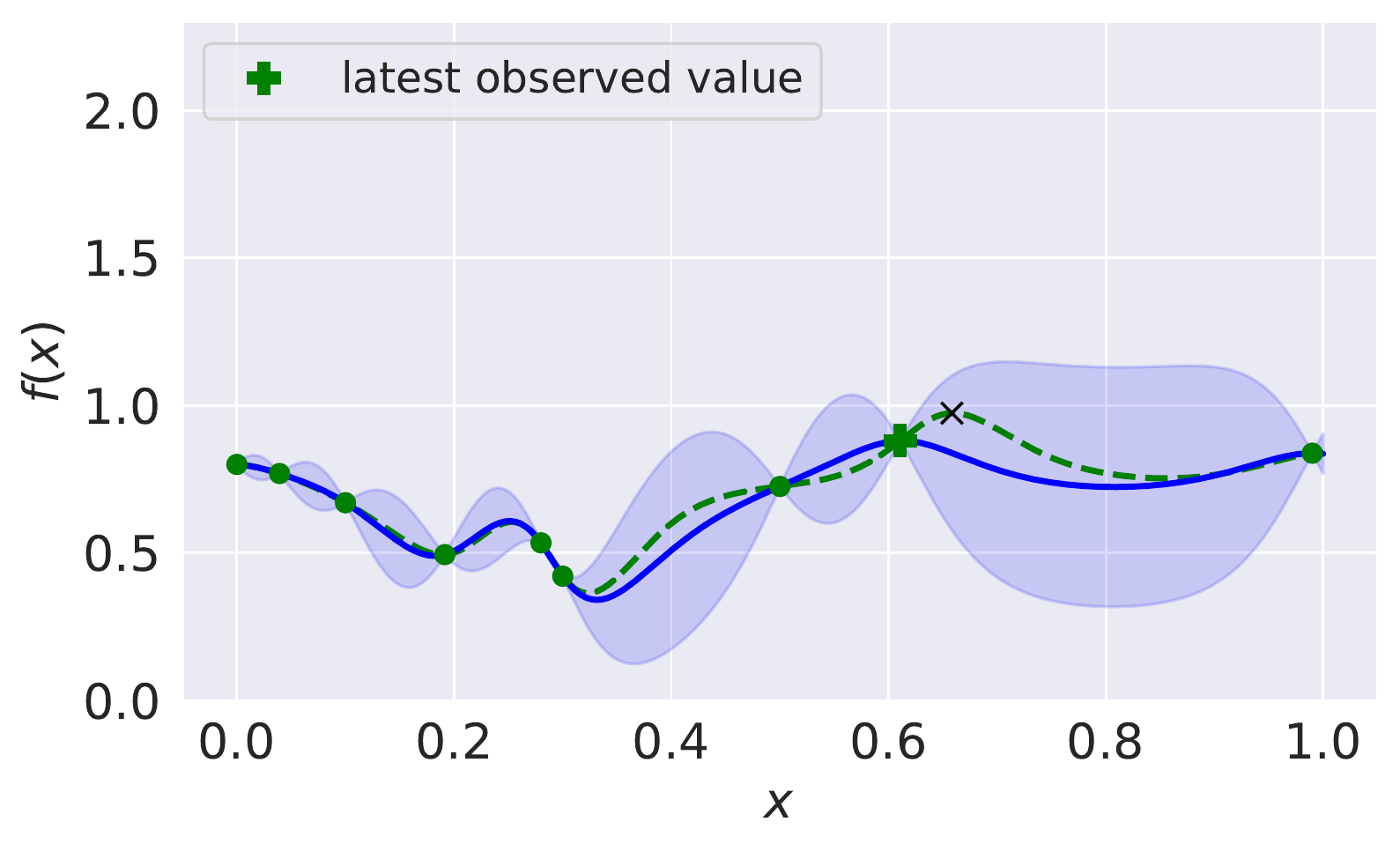}
  \phantom{\includegraphics[width=0.33\textwidth]{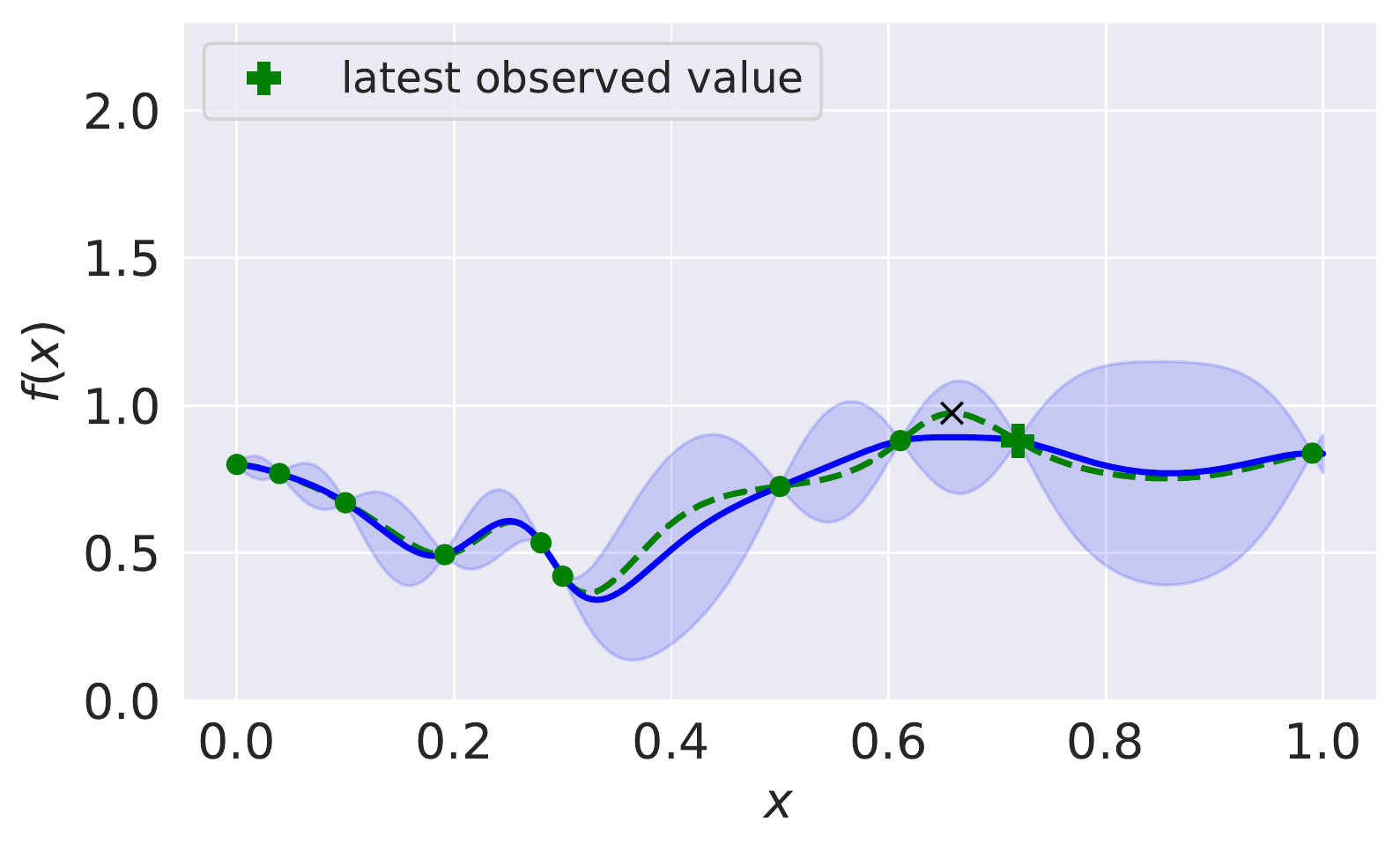}}\\
    \includegraphics[width=0.33\textwidth]{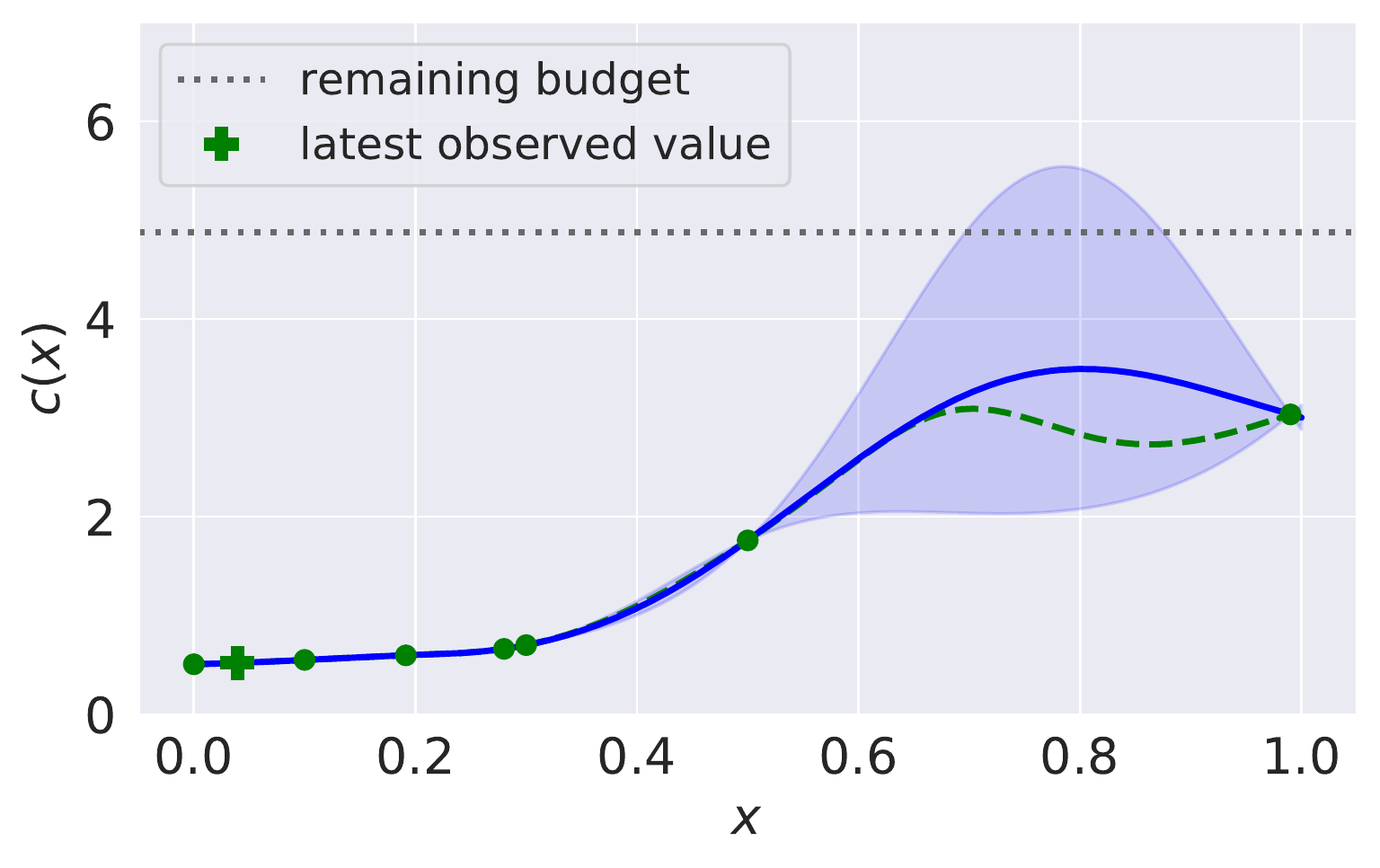}
  \includegraphics[width=0.33\textwidth]{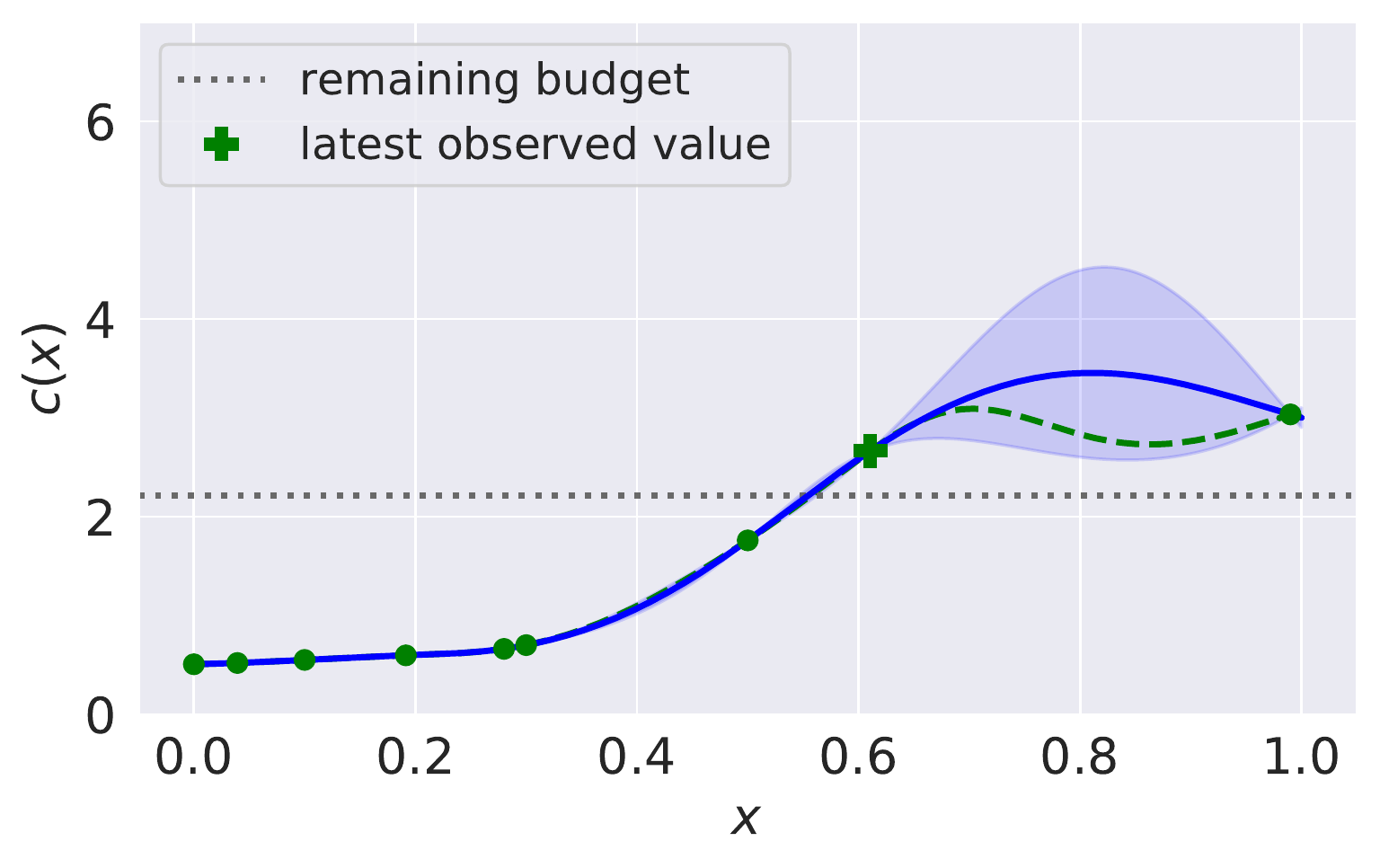}
  \phantom{\includegraphics[width=0.33\textwidth]{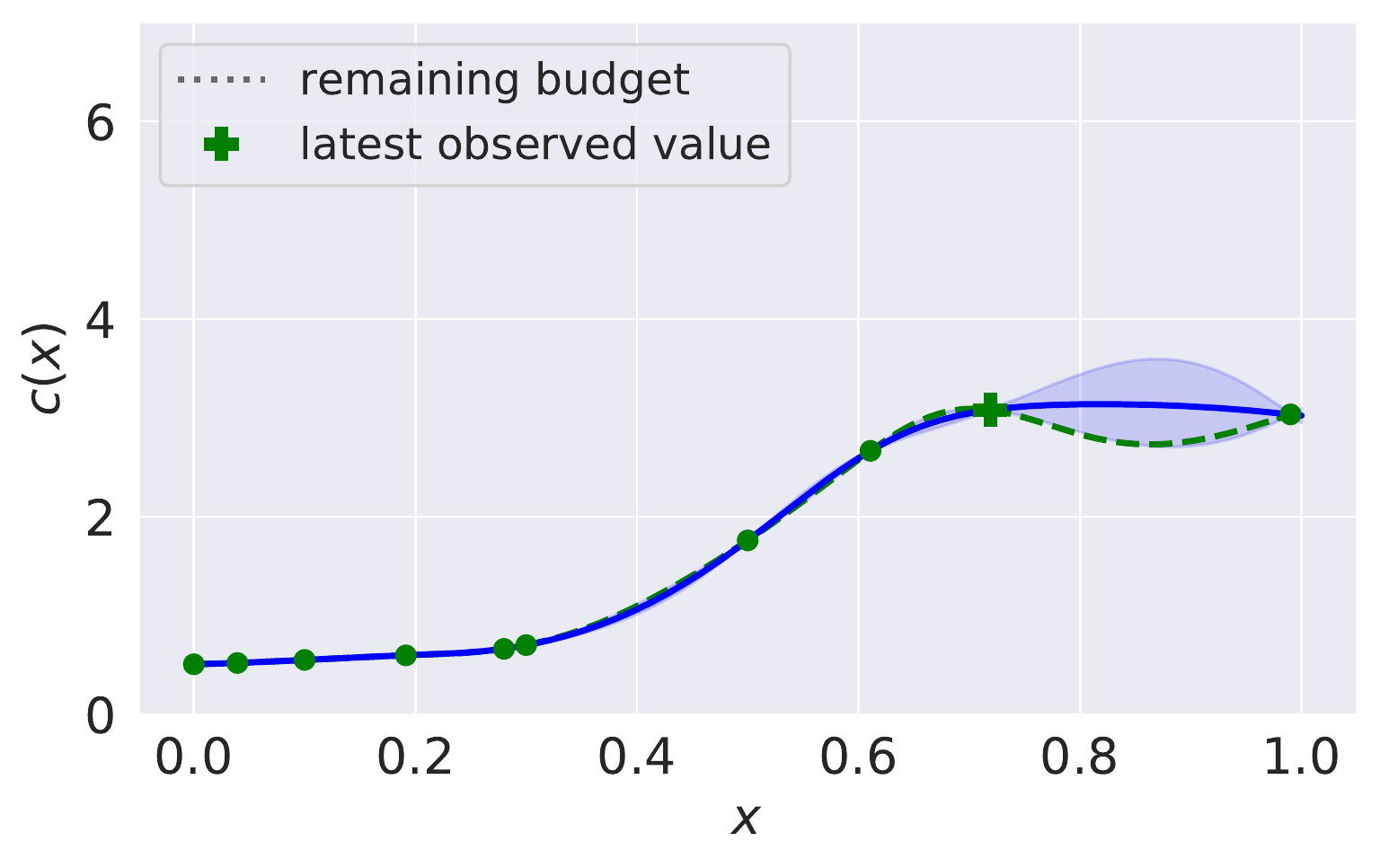}}\\
   \includegraphics[width=0.33\textwidth]{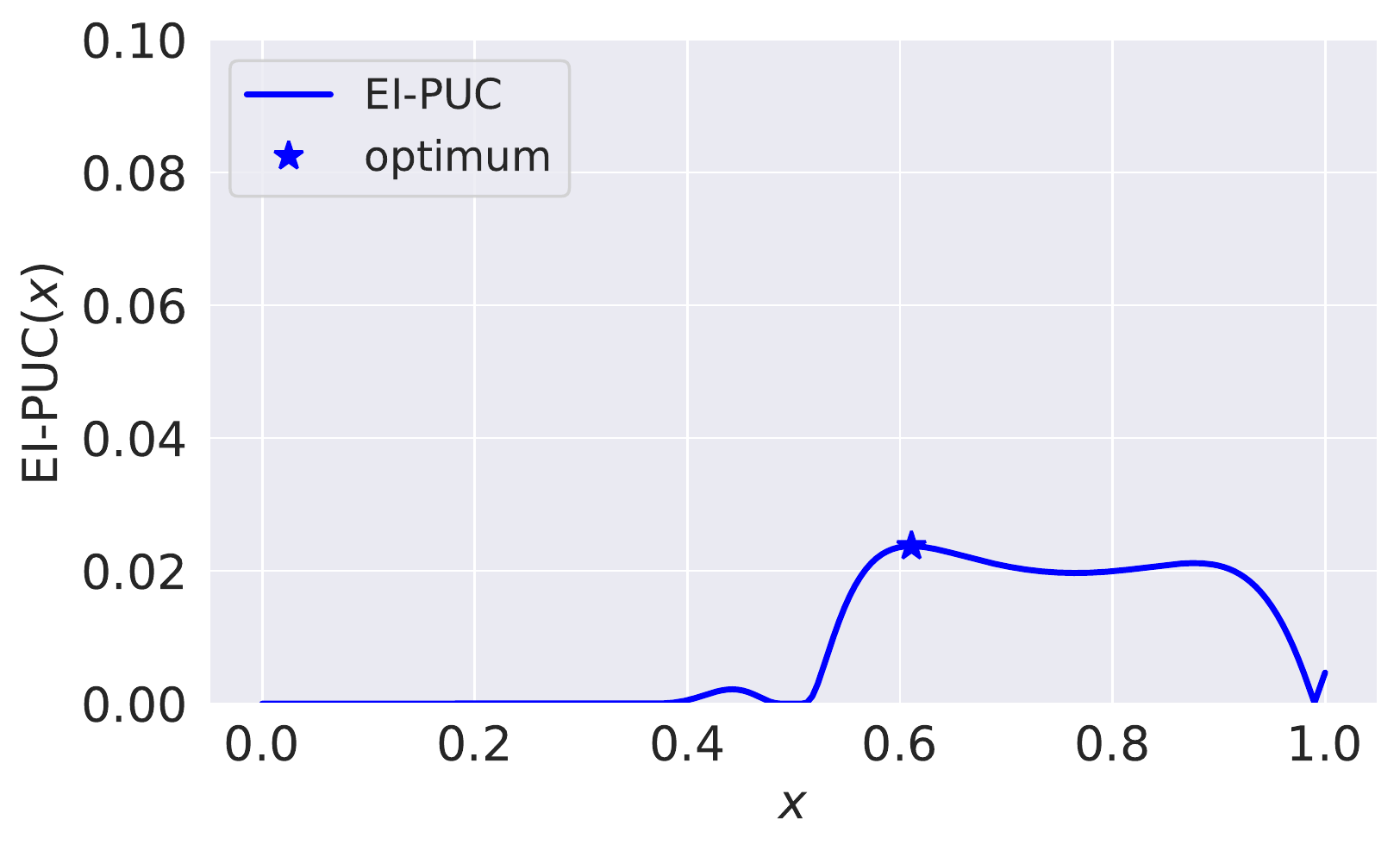}
  \includegraphics[width=0.33\textwidth]{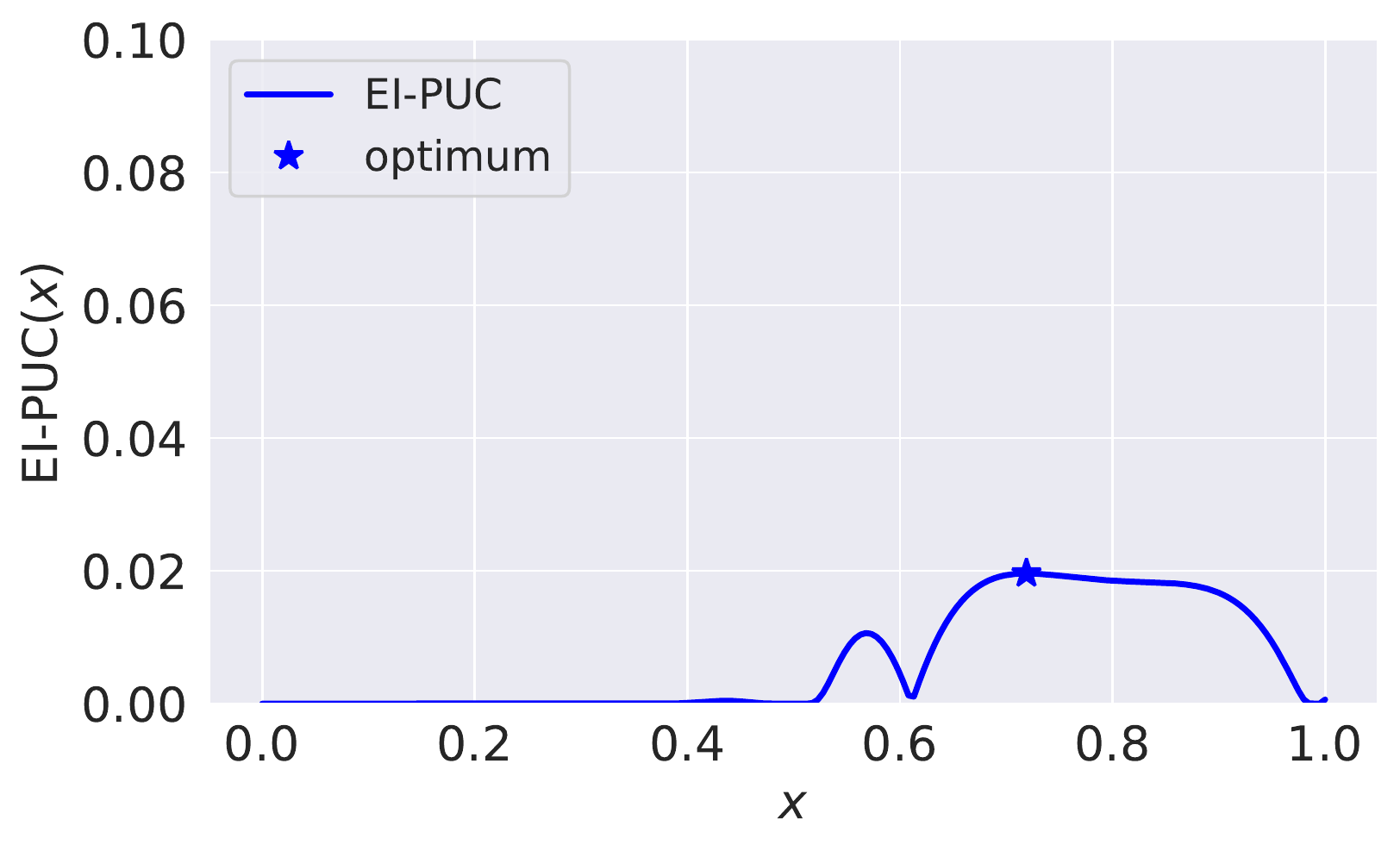}
  \phantom{\includegraphics[width=0.33\textwidth]{figures/animation/acqf5_eipuc.pdf}}
 \end{tabular}
}\\
\subfloat[B-MS-EI]{%
\begin{tabular}[b]{c}
 \includegraphics[width=0.33\textwidth]{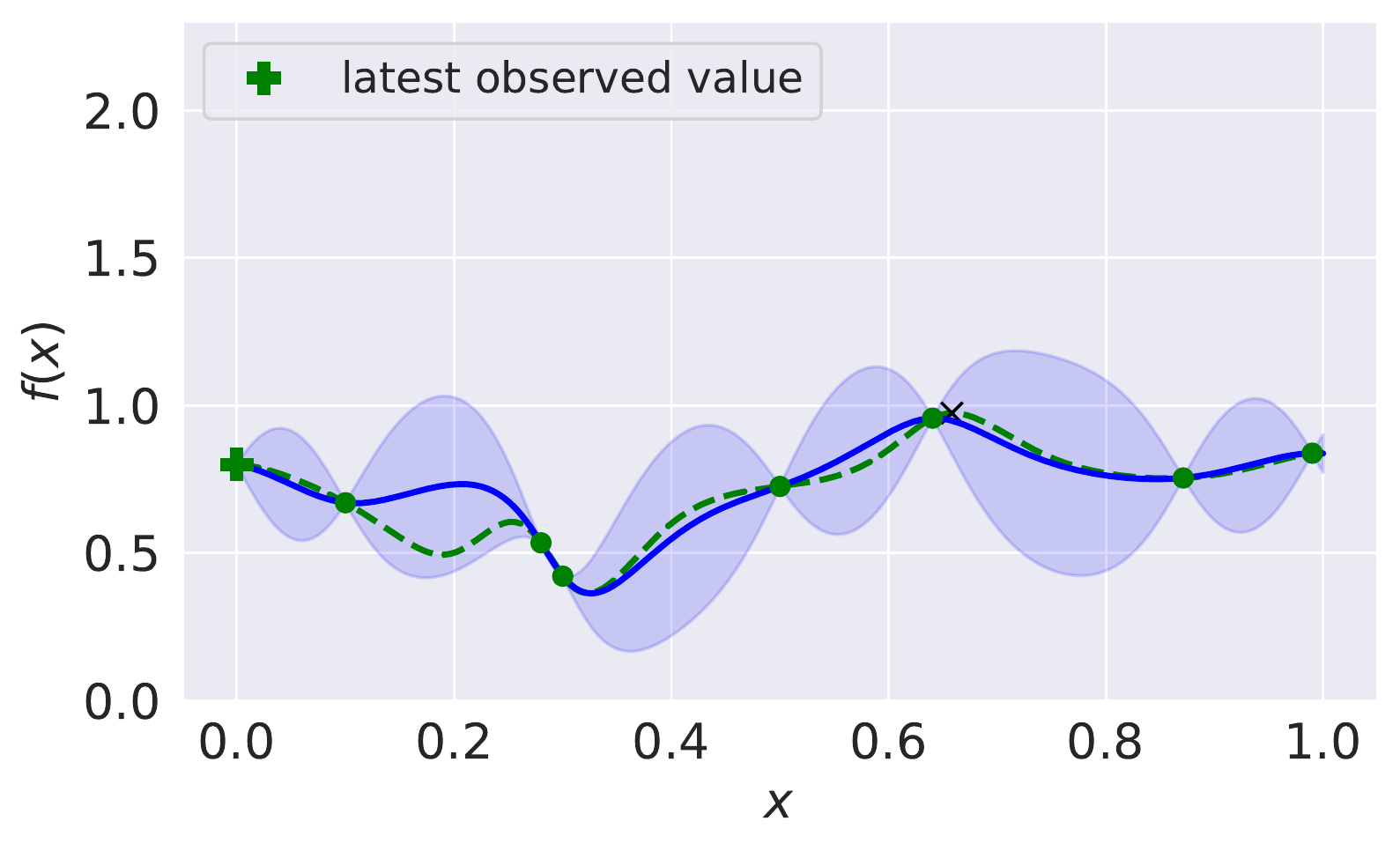}
  \phantom{\includegraphics[width=0.33\textwidth]{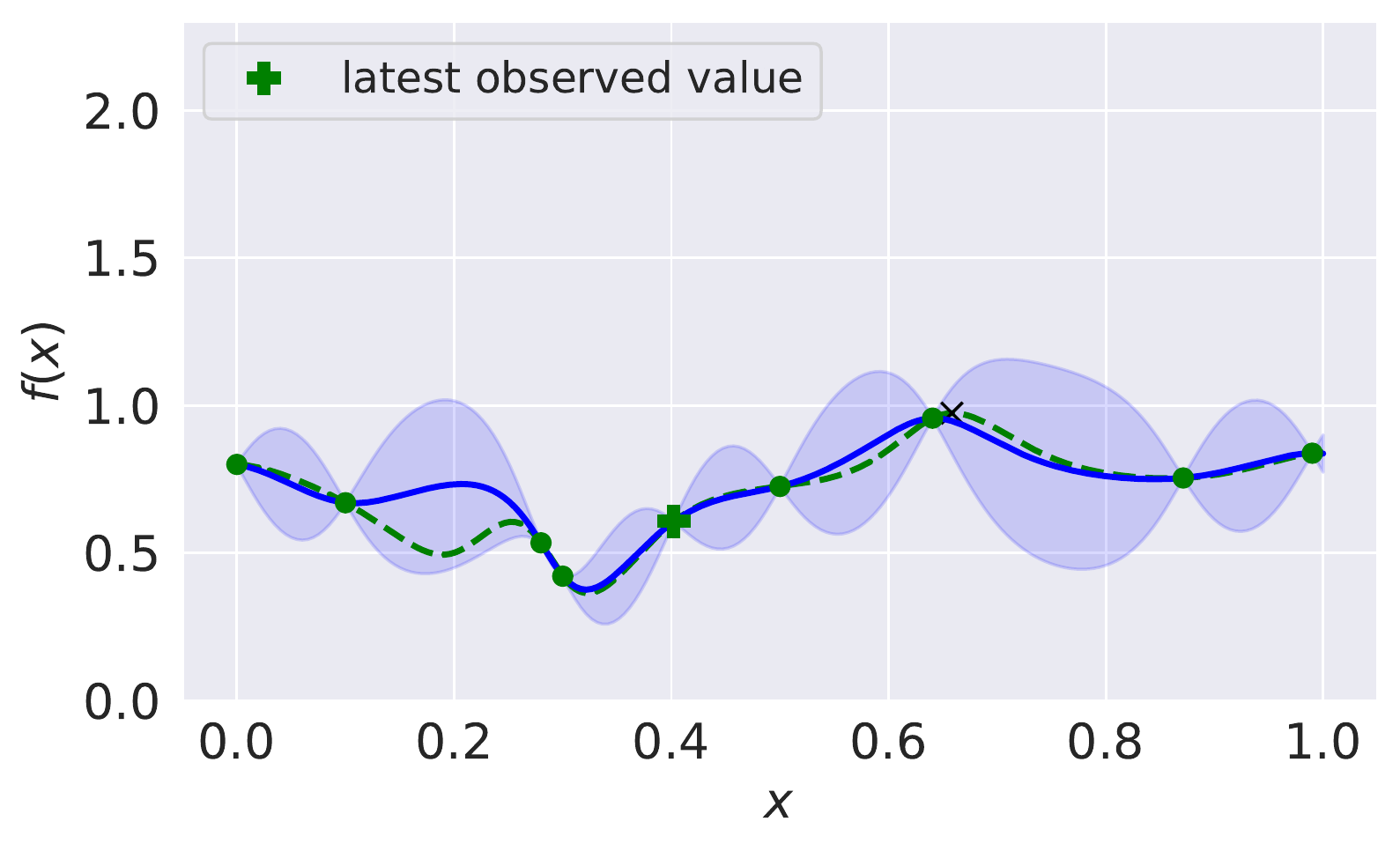}}
  \phantom{\includegraphics[width=0.33\textwidth]{figures/animation/obj4_bmsei.pdf}}\\
    \includegraphics[width=0.33\textwidth]{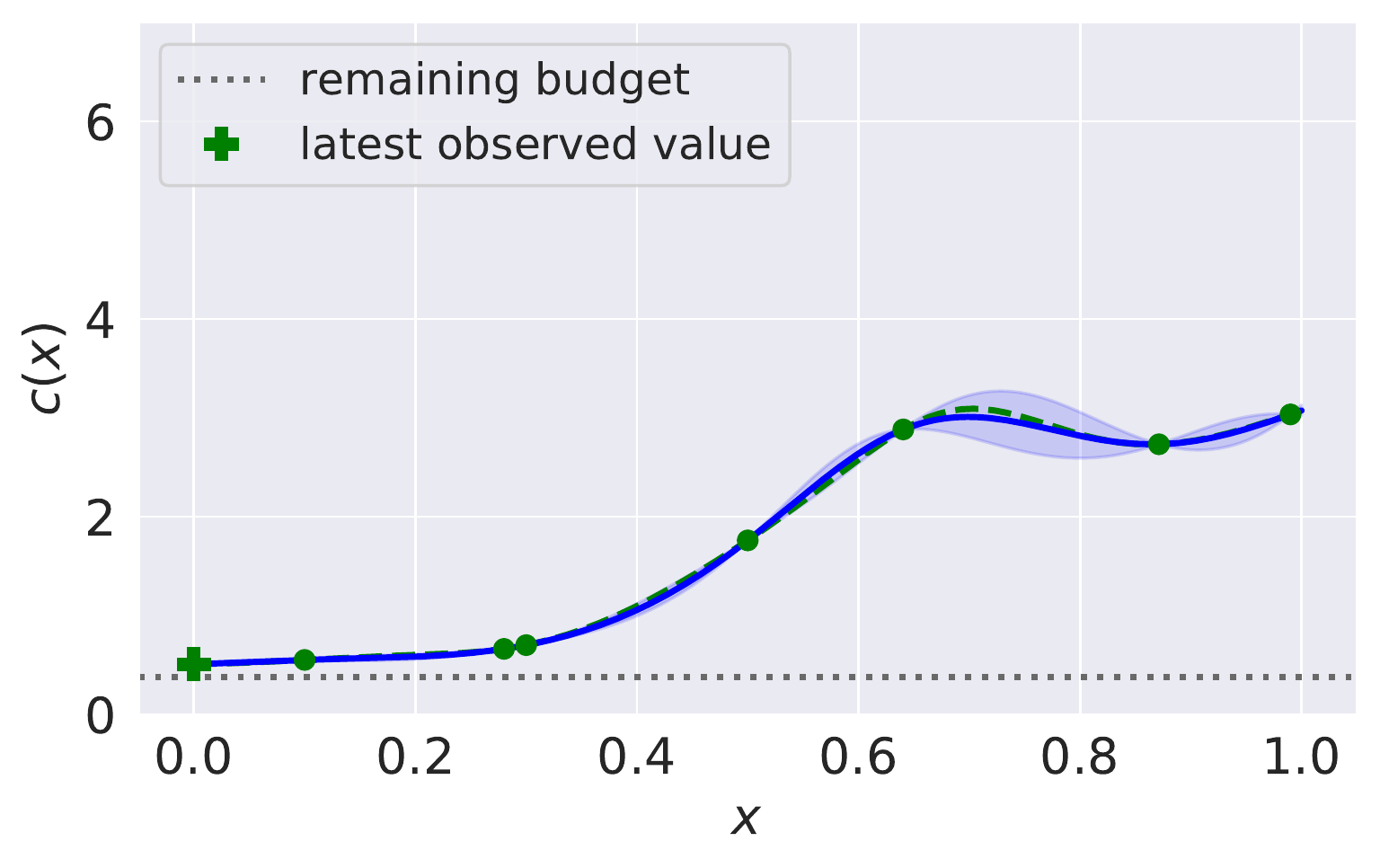}
  \phantom{\includegraphics[width=0.33\textwidth]{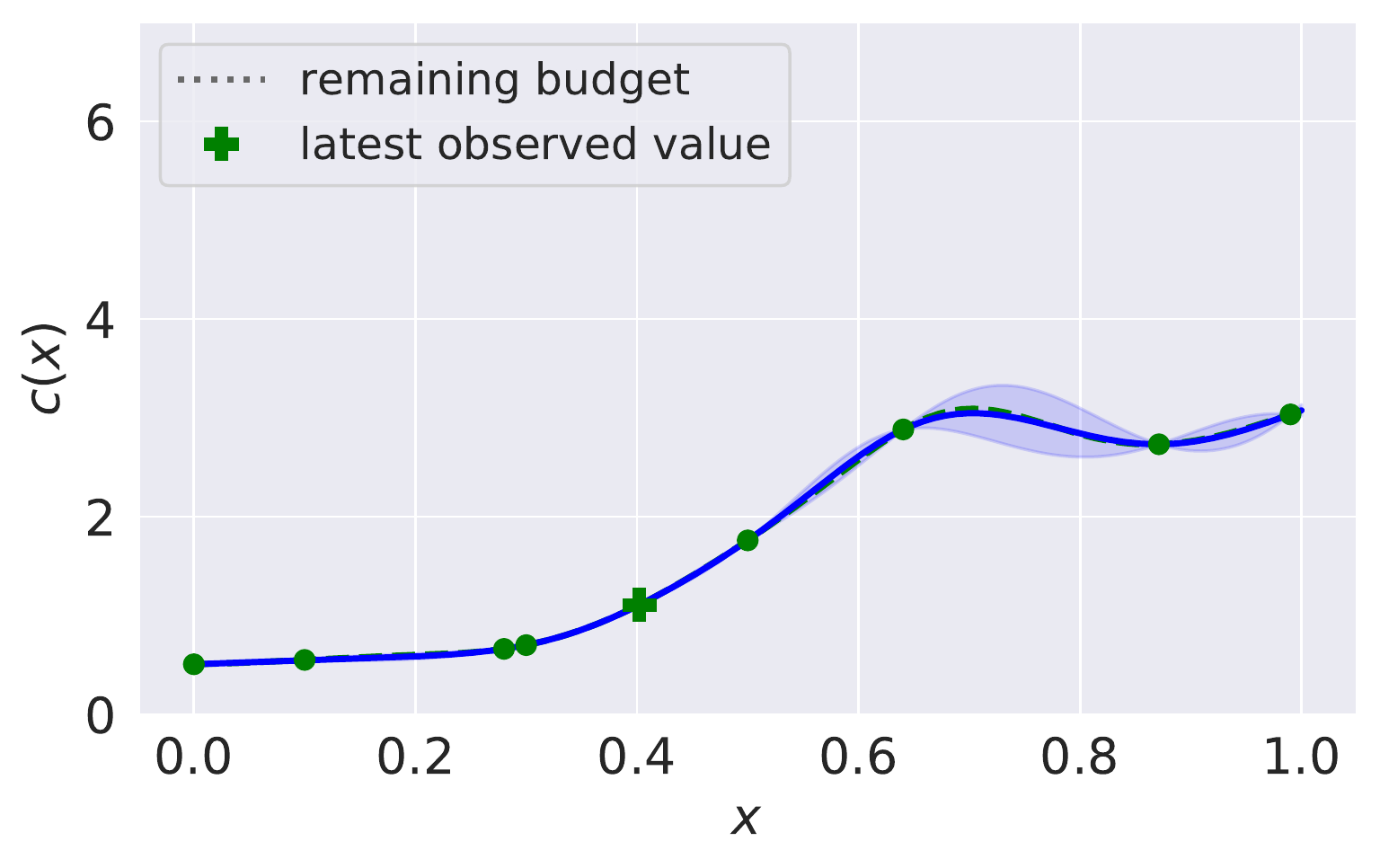}}
  \phantom{\includegraphics[width=0.33\textwidth]{figures/animation/cost4_bmsei.pdf}}\\
  \includegraphics[width=0.33\textwidth]{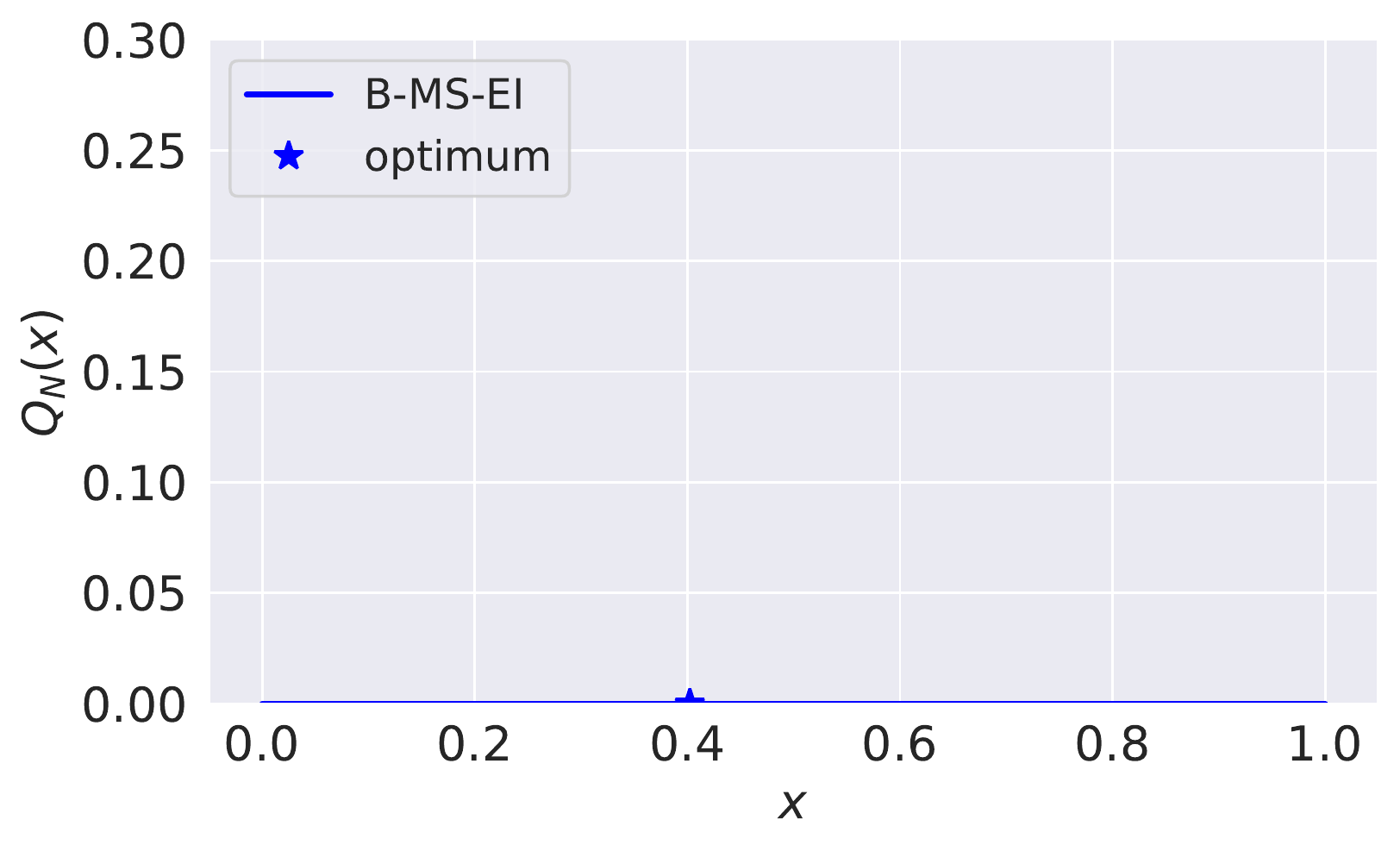}
  \phantom{\includegraphics[width=0.33\textwidth]{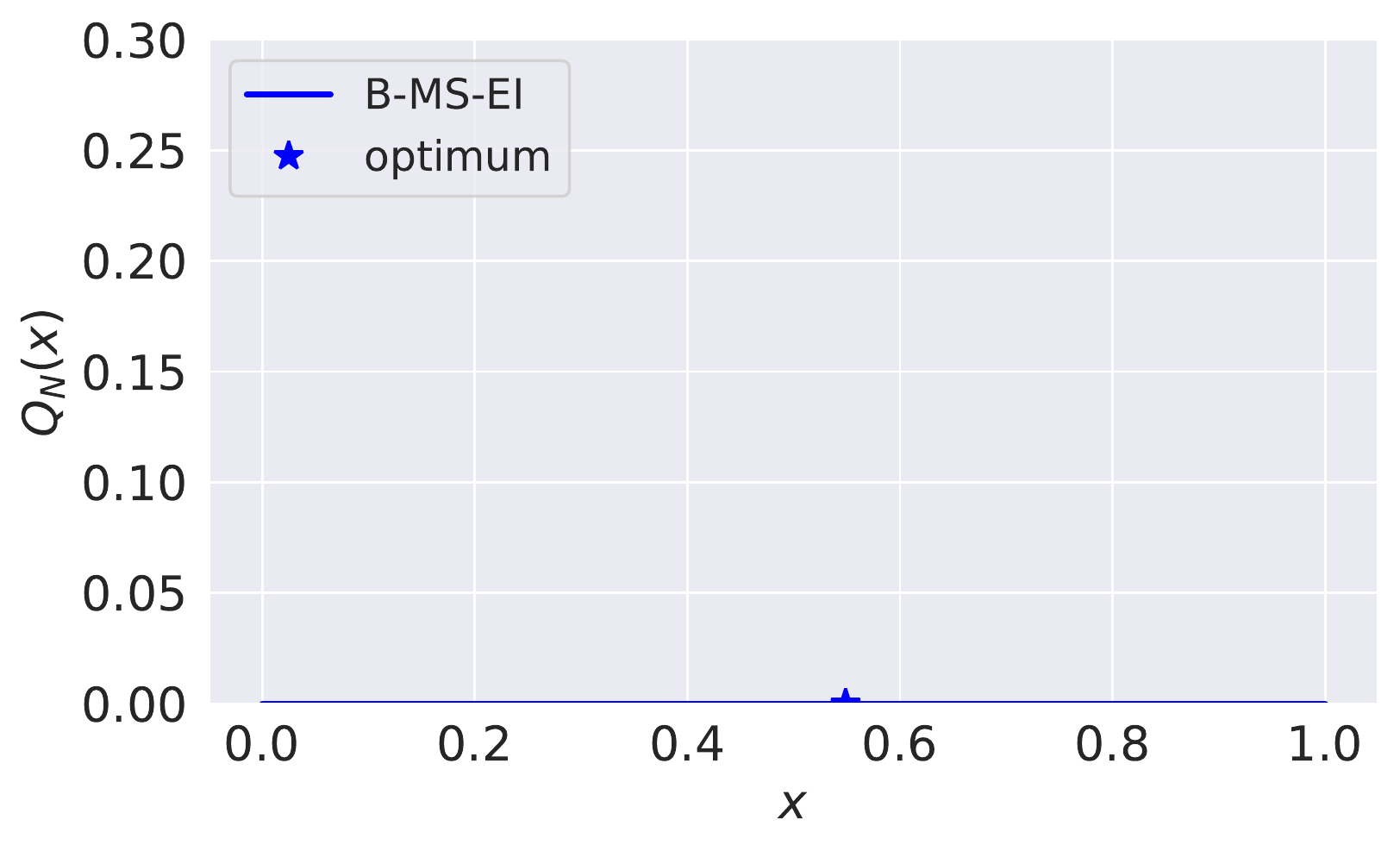}}
  \phantom{\includegraphics[width=0.33\textwidth]{figures/animation/acqf5_bmsei.pdf}}
\end{tabular}
}
\caption{Additional plots showing the evaluations performed within budget by EI-PUC and B-MS-EI. Subsequent evaluations are not plotted because the budget is exhausted after their completion and thus are not taken into account to report performance (i.e., the cost of the next point suggested exceeds the remaining budget). Note that EI-PUC performs two additional evaluations within budget, whereas B-MS-EI performs only one additional evaluation. B-MS-EI achieves a better final performance within budget than the one achieved by EI-PUC.
\label{fig:animation_supp}}
\end{figure}

The BoTorch package is publicly available under the MIT License. The datasets obtained from the HPOLib and HPOLib 1.5 libraries are publicly available under the GNU General Public License. The source code of the robot pushing problem is publicly available under the MIT License.

\section{Budgets Analysis}
\label{supp:budget}
To understand the effect of the budget on the performance of B-MS-EI, we evaluate its performance in three of our test problems (Dropwave, LDA, and Robot Pushing) using half of the original budget. We also report the performance of EI-PUC-CC, which is the only other benchmark method that is budget-aware. For comparison, we also include the performance of both algorithms under the original budget. The results of this experiment are shown in Figure~\ref{fig:experiments_budgets}. Remarkably,  B-MS-EI seems to benefit from knowing the budget constraint in advance. This does not seem to be the case for EI-PUC-CC, however. 

\begin{figure}[H]
  \centering
 \includegraphics[width=0.24\textwidth]{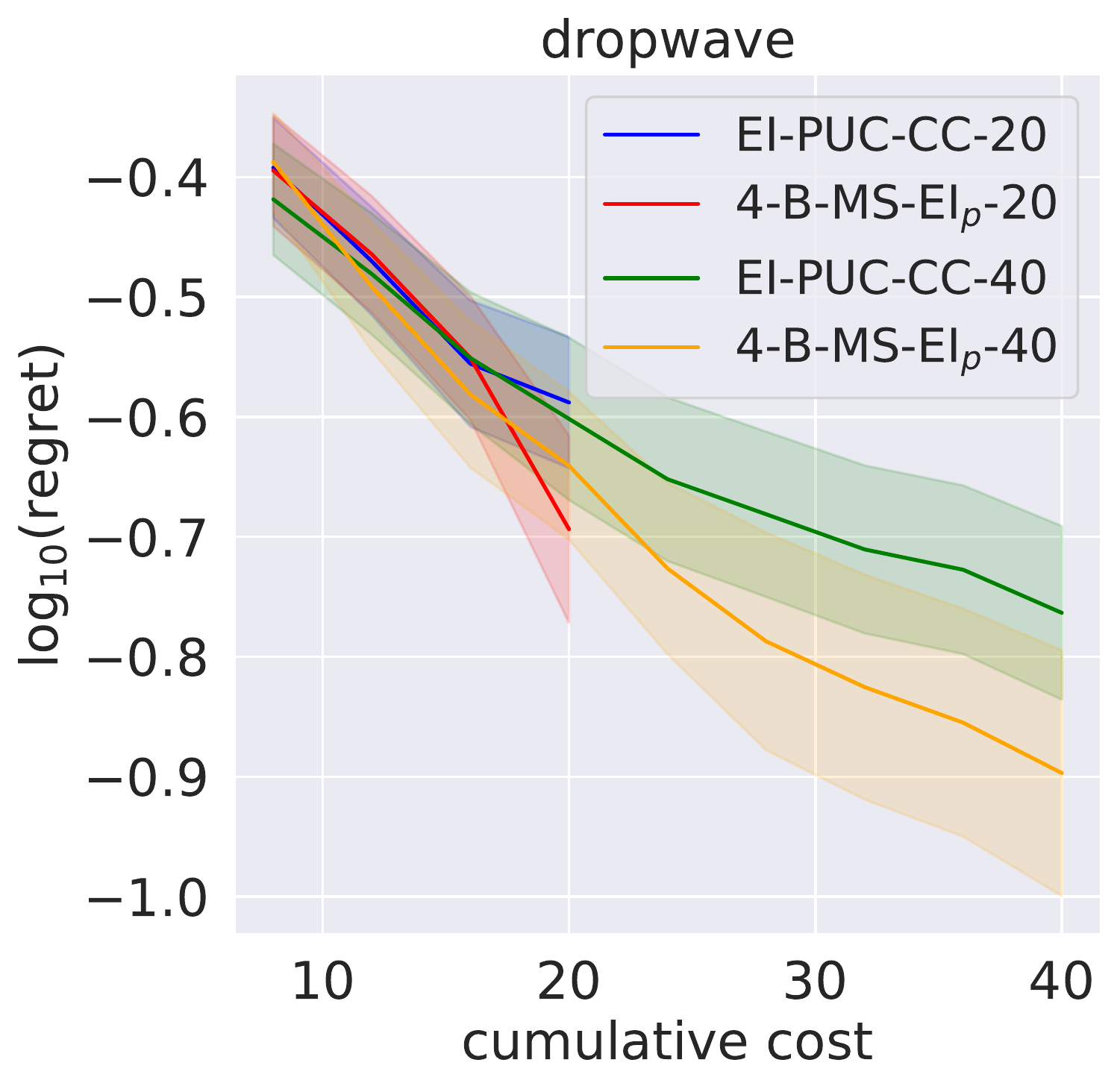}
 \includegraphics[width=0.235\textwidth]{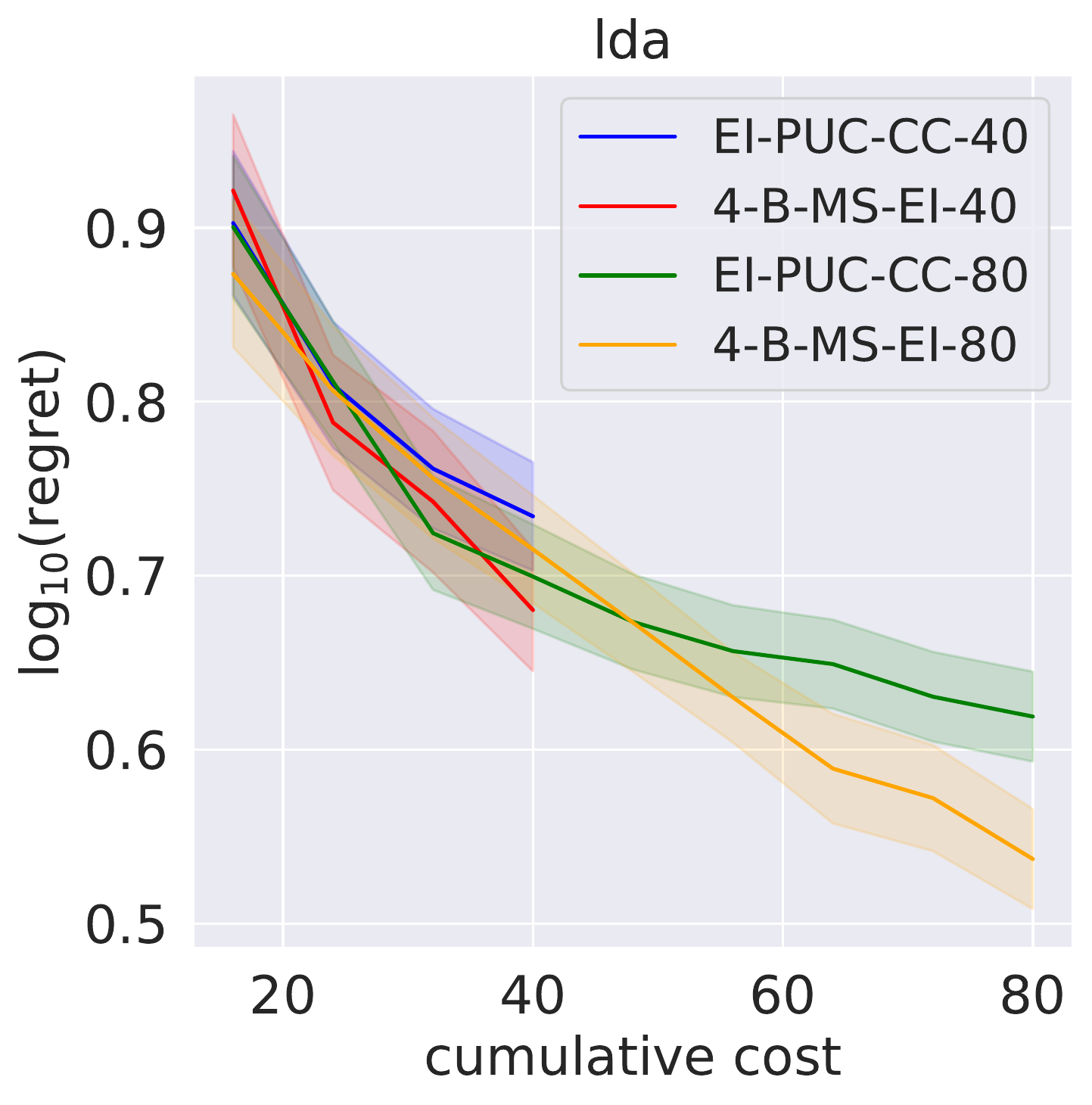}
 \includegraphics[width=0.252\textwidth]{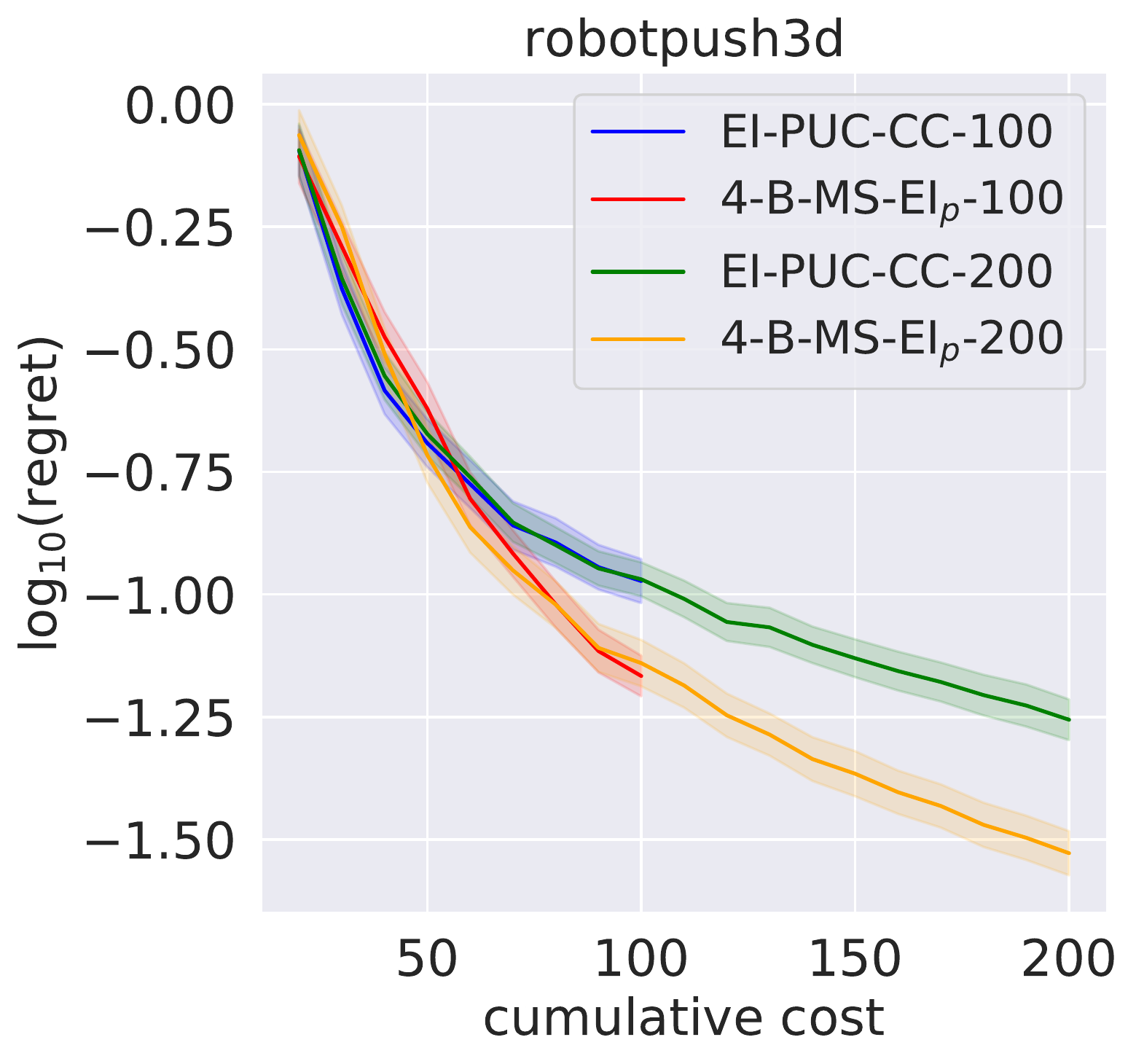}
  \caption{{Performance of B-MS-EI (B-MS-EI$_p$ for Dropwave and Robot Pushing) and EI-PUC-CC in three of our test problems under two different budgets. In contrast with EI-PUC-CC, B-MS-EI seems to benefit from knowing the budget constraint in advance.} \label{fig:experiments_budgets}}
\end{figure}

\end{document}